\documentclass{article}
\usepackage{times}
\usepackage{graphicx} 
\usepackage{subfigure} 
\usepackage[numbers]{natbib}
\usepackage{xfrac}
\usepackage{algorithm}
\usepackage{algorithmic}
\usepackage{times}
\usepackage{amssymb}
\usepackage{amsmath}
\usepackage{amsthm}
\usepackage{bm}
\usepackage{tikz}
\usepackage{tkz-euclide}
\usetkzobj{all}
\usepackage{chngcntr}
\usepackage{verbatim}
\usepackage{hyperref}
\usepackage{wrapfig}
\usepackage[font={small}]{caption}

\usepackage[T1]{fontenc}    
\usepackage{url}            
\usepackage{booktabs}       
\usepackage{amsfonts}       
\usepackage{nicefrac}       
\usepackage{microtype}      

\newcommand{\quadratic}{{\cal Q}} 
\newcommand{\param}{{\boldsymbol \beta}} 
\newcommand{\dparam}{{\boldsymbol \alpha}} 
\newcommand{\para}{{\beta}} 
\renewcommand{\b}{\mathbf{b}} 
\renewcommand{\c}{\mathbf{c}} 
\newcommand{\vv}{\mathbf{v}} 
\newcommand{\uv}{\mathbf{u}} 
\newcommand{\labels}{\mathbf{y}} 
\newcommand{\x}{\mathbf{x}} 
\newcommand{\y}{\mathbf{y}} 
\newcommand{\EE}{\mathbf{E}} 
\newcommand{\glmp}{\mathbf{z}} 
\newcommand{\Z}{\mathbf{Z}} 
\renewcommand{\H}{{\mathbf H}} 
\newcommand{\G}{{\mathbf G}} 

\renewcommand{\S}{{\cal S}} 
\renewcommand{\L}{{\cal L}} 
\newcommand{\I}{\mathbf{I}} 
\newcommand{\um}{\mathbf{U}} 
\newcommand{\vm}{\mathbf{V}} 
\newcommand{\mSigma}{\mathbf{\Sigma}} 
\newcommand{\mLambda}{\mathbf{\Lambda}} 
\newcommand{\mGamma}{\mathbf{\Gamma}}
\newcommand{\xm}{\mathbf{X}} 
\newcommand{\bigo}{\mathcal{O}}

\newcommand{\reg}{\mu} 

\newcommand{\R}{\mathbb{R}} 
\newcommand{\rmat}{\Psi} 
\newcommand{\risk}{{\cal R}} 
\newcommand{\E}{\mathbf{E}} 
\renewcommand{\S}{{\cal S}} 

\newtheorem{theorem}{Theorem}
\newtheorem*{lemma*}{Lemma}
\newtheorem*{corollary*}{Corollary}
\newtheorem*{theorem*}{Theorem}
\newtheorem{lemma}[theorem]{Lemma}

\newtheorem{corollary}[theorem]{Corollary}

\newcolumntype{S}{>{\centering\arraybackslash} m{.10\linewidth} }
\newcolumntype{T}{>{\centering\arraybackslash} m{.3\linewidth} }
\newcolumntype{D}{>{\centering\arraybackslash} m{.15\linewidth} }
\newcolumntype{K}{>{\centering\arraybackslash} m{.4\linewidth} }
\usepackage[final]{nips_2017}

\title{Accelerated Dual Learning by \\ Homotopic Initialization}
\author{
  Hadi Daneshmand, Hamed Hassani, Thomas Hofmann \\
  Department of Computer Science, ETH Zurich, Switzerland \\
   \footnotesize{\texttt{\{hadi.daneshmand,hamed,thomas.hofmann\}@inf.ethz.ch}}
}
\begin{document}

\maketitle

\begin{abstract}
Gradient descent and coordinate descent are well understood in terms of their asymptotic behavior, but less so in a transient regime often used for approximations in machine learning. We investigate how proper initialization can have a profound effect on finding near-optimal solutions quickly. We show that a certain property of a data set, namely the boundedness of the correlations between eigenfeatures and the response variable, can lead to faster initial progress than expected by commonplace analysis. Convex optimization problems can tacitly benefit from that, but this automatism does not apply to their dual formulation. We analyze this phenomenon and devise provably good  initialization strategies for dual optimization as well as heuristics for the non-convex case, relevant for deep learning. We find our predictions and methods to be experimentally well-supported.
\end{abstract}

\section{Introduction}

The standard approach to supervised machine learning is to cast it as an optimization problem with a suitable loss function  and a regularizer. Learning then amounts to minimizing the regularized training risk over a chosen parametric model family. However, this view obstructs the fact that the ultimate goal is in minimizing the expected risk on unseen data and that the regularized empirical risk serves merely as a proxy for the former. Optimization for machine learning is thus by design a quest for approximate solutions. This offers learning-specific tradeoffs between statistical accuracy and computational complexity, such as early stopping \cite{yao2007early}, convex relaxation \cite{chandrasekaran2013computational}, data sub-sampling \cite{daneshmand2016starting}, or accepting locally optimal solutions as in deep learning \cite{choromanska2015loss,soudry2016no}. Whereas optimization algorithms are often analyzed in terms of their convergence to the optimum, in learning the interest should primarily be on how quickly one can find suboptimal solutions of sufficient quality (relative to the size of the data sample). 

One key question in this context is what initialization strategy to use for the weights of a model, as the initial parameter choice has a huge impact on the transient phase of iterative learning algorithms. There are many examples that clearly demonstrate that initialization matters, for instance in deep learning via weight transfer \cite{yosinski2014transferable},  data-dependent initialization \cite{krahenbuhl2015data}, or in avoidance of saddle points \cite{lee2016gradient}. The same is true for unsupervised learning problems like matrix completion \cite{keshavan2010matrix} or $K$-means \cite{arthur2007k}. Notably, for convex learning, the role of initialization has been somewhat neglected, mainly because of the guaranteed convergence to the global optimum. However, for massive data sets, initialization can have a huge practical impact on the scalability of an algorithm. Moreover, the (easier) convex setting also allows for a more rigorous analysis of the effect of initialization and the reasons for the slow-down caused by poor starting points. In this vein, the current paper provides a detailed analysis of convex learning, specifically of ridge regression and generalized linear models, that suggests to pre-train models with artificially increased regularization and to use this as an initialization in the spirit of homotopy or continuation methods \cite{allgower2012numerical,mokhtari2016adaptive}. The focus of our work is on the dual problem, because it offers more flexibility in the data representation (i.e.~through the use of kernels), allows for fast algorithms like stochastic coordinate descent \cite{shalev2013stochastic} that exhibit linear convergence and is also amenable to data sharding and communication-efficient distributed implementations \cite{jaggi2014communication}. In particular, our method offers an yet unexplored direction to get speed-ups for sequential or parallel dual algorithms. 

\section{Ridge Regression}

\subsection{Primal Formulation} 

For concreteness, we perform an in-depth analysis of ridge regression \cite{hoerl1970ridge}.  Given a training set of $n$ observations  $(\x_i,y_i)$, with inputs $\x_i \in \mathbb{R}^d$ and reponses $y_i \in \R$, we denote by $\xm \in \R^{n\times d}$ the data matrix and by $\y \in \R^n$ the response vector. Let us assume w.l.o.g that $\y$ is mean zero and unit variance and that the data is centered.The ridge regression objective with regularization $\mu>0$, can be expressed as
\begin{align} \label{eq:quadratic_fun}
	\quadratic(\param) = \frac{1}{2} \param^\top \H \param  - \param^\top \b,
	\quad \H := \frac 1n \xm^\top \xm + \mu \I \in \R^{d \times d}, 
	\quad \b := \frac 1n \xm^\top \y \in \R^d
\end{align}
The optimal solution is explicitly given by the normal equations
\begin{align}
\nabla_{\param} \quadratic \stackrel !=0 \iff \param^* = \H^{-1} \b
= \left( \xm^\top \xm + n \mu \I\right) ^{-1} \xm^\top \y \,.
\end{align}
Note that we can interpret key quantities as expectations under the training distribution, which will subsequently become relevant. In particular note that $b_j = \E[X_j Y]$ is the covariance between the $j$-th feature dimension and the response. Obviously, the elements of $\frac 1n \xm^\top \xm$ can be expressed as $\E[X_iX_j]$, i.e.~they encode the empirical variances and covariances of the features. We also denote by $\E^\mu[X_i^2] = \E[X^2_i]+\mu$ the biased varince estimated that we get through regularization. 

For the purpose of analysis and to ease the exposition, we orthogonally transform the data into an eigenfeature representation. To that extend we use the scaled SVD of $\xm= \sqrt n \um \mSigma \vm^\top$, such that $\xm^\top \xm = n \vm \mSigma^\top \mSigma \vm^\top$ and define $\Z := \mathbf X \mathbf V = \sqrt n \mathbf U \mSigma$. Correspondingly the parameters are transformed via $\param \leftarrow \vm^\top \param$, we arrive at the diagonalized objective 
 \begin{align}
Q(\param) 
& = \frac 12 \param^\top \!\! \underbrace{\left( \mSigma^\top \mSigma + \mu \I \right)}_{=:\mLambda}  \! \param - \param^\top \underbrace{\frac{1}{n}\Z^\top  \y}_{=:\c} 
\;=\; \frac 12 \sum_{j=1}^d  \beta_j^2 \underbrace{\E^\mu[Z_j^2]}_{=\lambda_j} - \beta_j \underbrace{\E[YZ_j]}_{=c_j}
\label{eq:diagonal} 
\end{align}
\subsection{Gradient Descent Analysis}

Gradient descent (GD) optimizes an objective through iterative gradient updates with step size $\gamma>0$
\begin{align}
	\param^{t+1} = \param^{t} - \gamma \nabla_\param \quadratic(\param^{t}).
\end{align}
For the diagonalized quadratic objective, the iterate sequence is explicitly given as follows. 
\begin{lemma} \label{lemma:gd_iterates} GD initialized at $\param^0$ yields the iterate sequence
	\begin{align}
		\param^{t}  & = \param^* + \left(\I - \gamma\mLambda \right)^t (\param^0 - \param^*)
	\end{align} 
\end{lemma}
\begin{corollary} \label{corrolary:gd_fixpoint} For any $\gamma < 2 \lambda_1^{-1}$ and initial $\param^0$, the iterate sequence generated by GD is guaranteed to converge to $\param^*$ at a linear rate. 
%
\end{corollary}
\begin{corollary} 
\label{corrolary:slow}
Define the condition number as $\kappa := \lambda_{1}/\lambda_{d}$. The rate of parameter convergence of $\beta_d$ is lower bounded by $1-\gamma \lambda_d
 > 1 - 2/\kappa$.
\end{corollary}

Our interest is in the convergence of the objective value. We can easily (and exactly) relate distance in parameter space to suboptimality through $\mLambda$. 
\begin{lemma} \label{lemma:subopt_quad}
Let  $\param^*$ be the minimizer of \eqref{eq:quadratic_fun} and $\quadratic^*
\triangleq Q(\param^*)$. Then for any $\param$
\begin{align} \label{eq:subopt_quad}
	\quadratic(\param) - \quadratic^* =  \frac{1}{2}	(\param - \param^*) \mLambda (\param - \param^* ) = \frac 12 \sum_{j=1}^d \E^\mu[Z_j^2]  (\beta_j- \beta^*_j) ^2 \,.
\end{align}
\end{lemma}
Let us compare the worst case rate we would expect based on $\kappa$ in Corollary \ref{corrolary:slow} with the empirical suboptimality of the GD iterate sequence on some (randomly selected) data sets. The plots in Figure \ref{fig:worst-better} show that initially (and for the relevant transient phase) the observed reduction of suboptimality is typically much better than what may be expected based on the convergence rates in parameter space. This is a striking behavior that our work aims to explain and to better exploit. 

\begin{figure}
\begin{footnotesize}
\begin{sc}
\begin{tabular}{@{\hspace{0.1cm}}c@{\hspace{0.1cm}}c@{\hspace{0.1cm}}c@{\hspace{0.1cm}}c}
\includegraphics[width=0.2\linewidth]{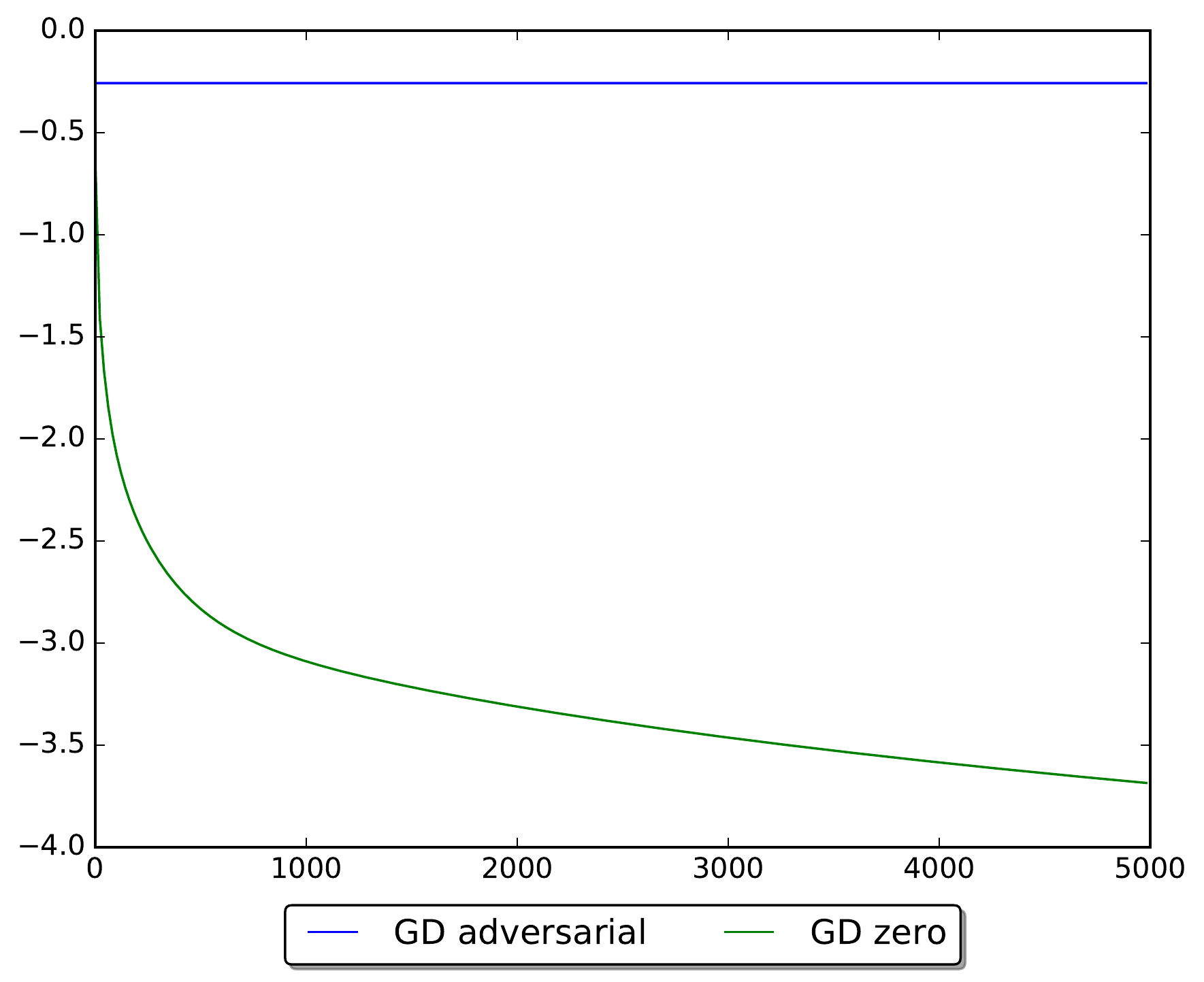}  & 
  \includegraphics[width=0.2\linewidth]{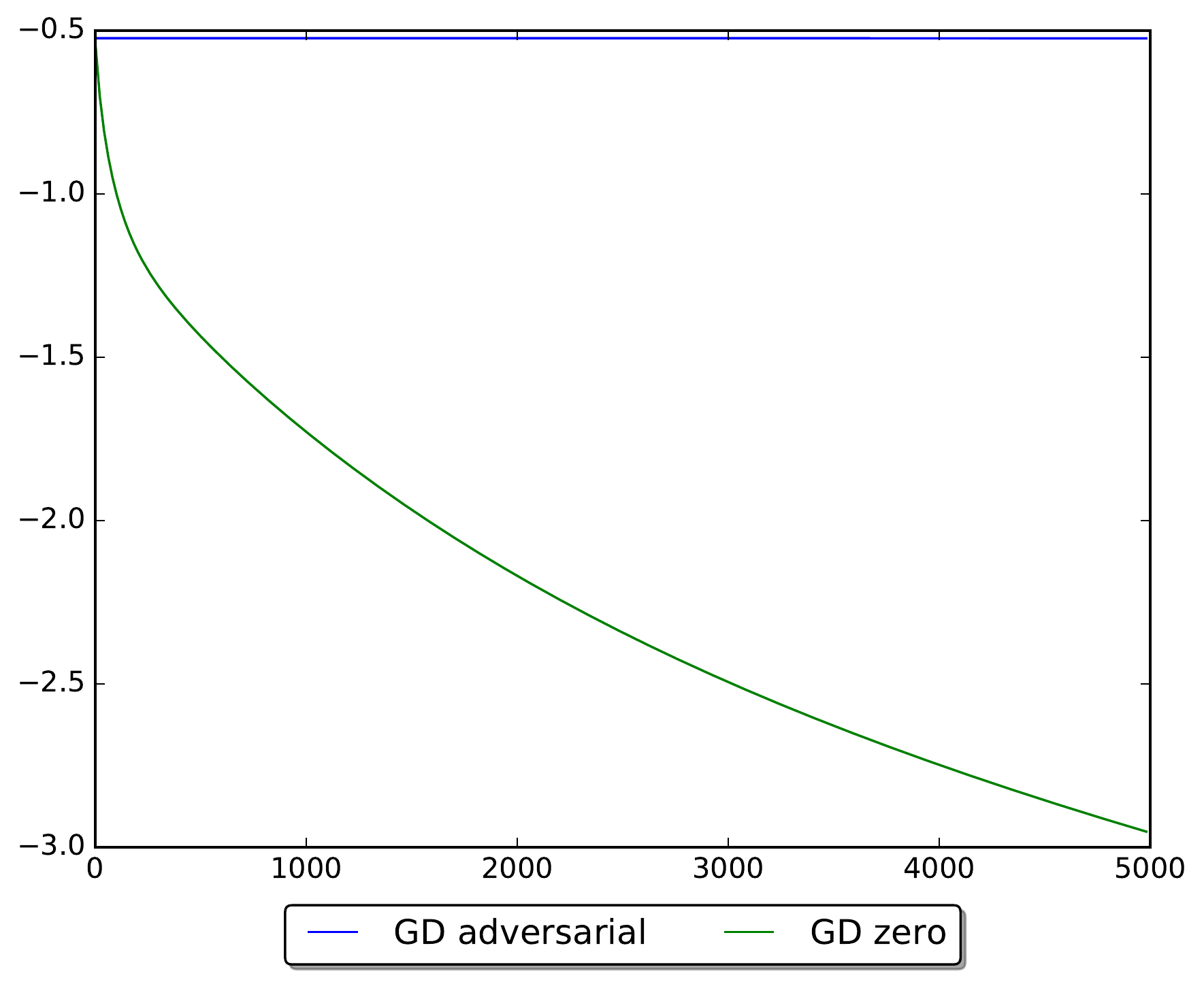} &
  \includegraphics[width=0.2\linewidth]{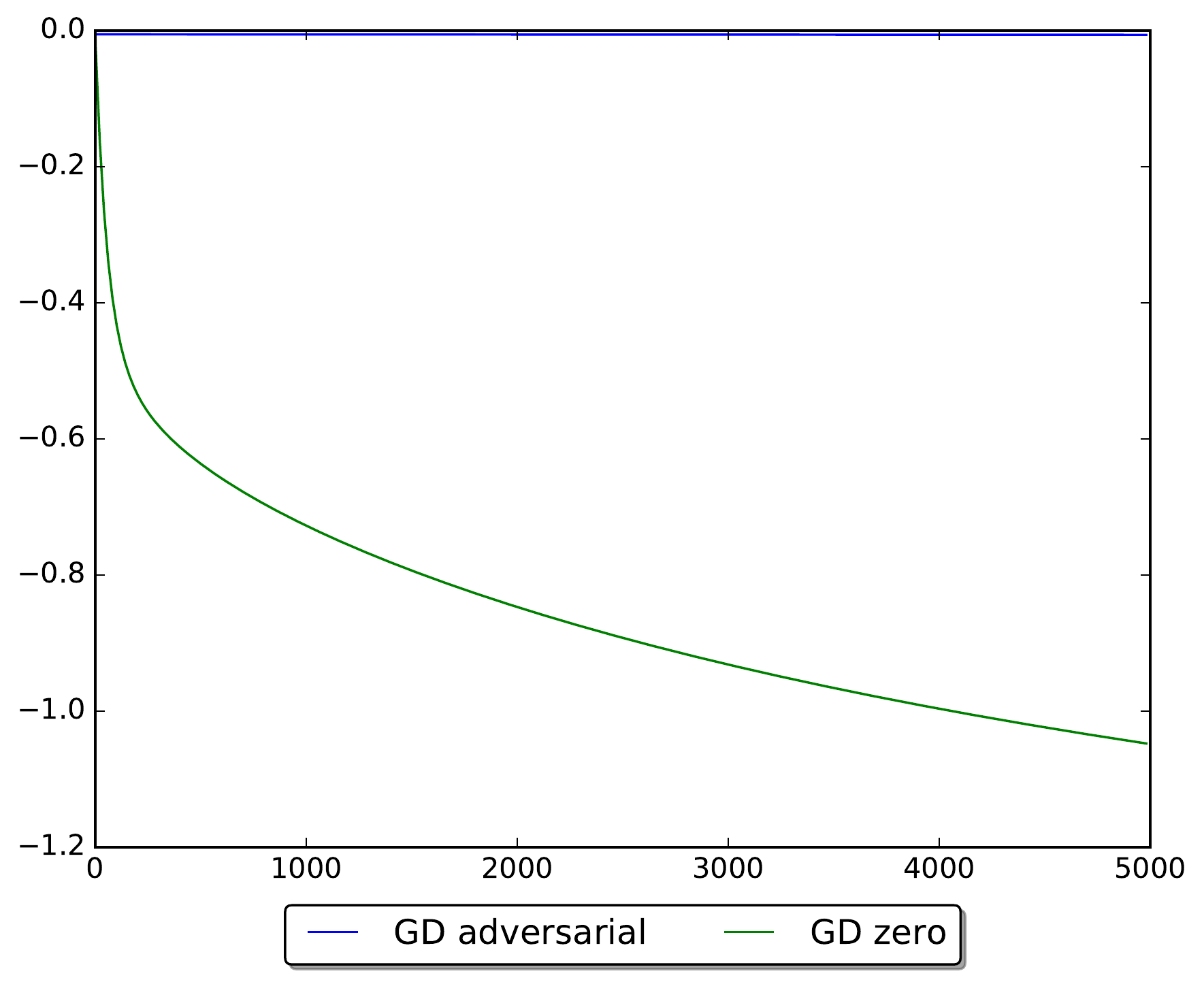}  & 
  \includegraphics[width=0.2\linewidth]{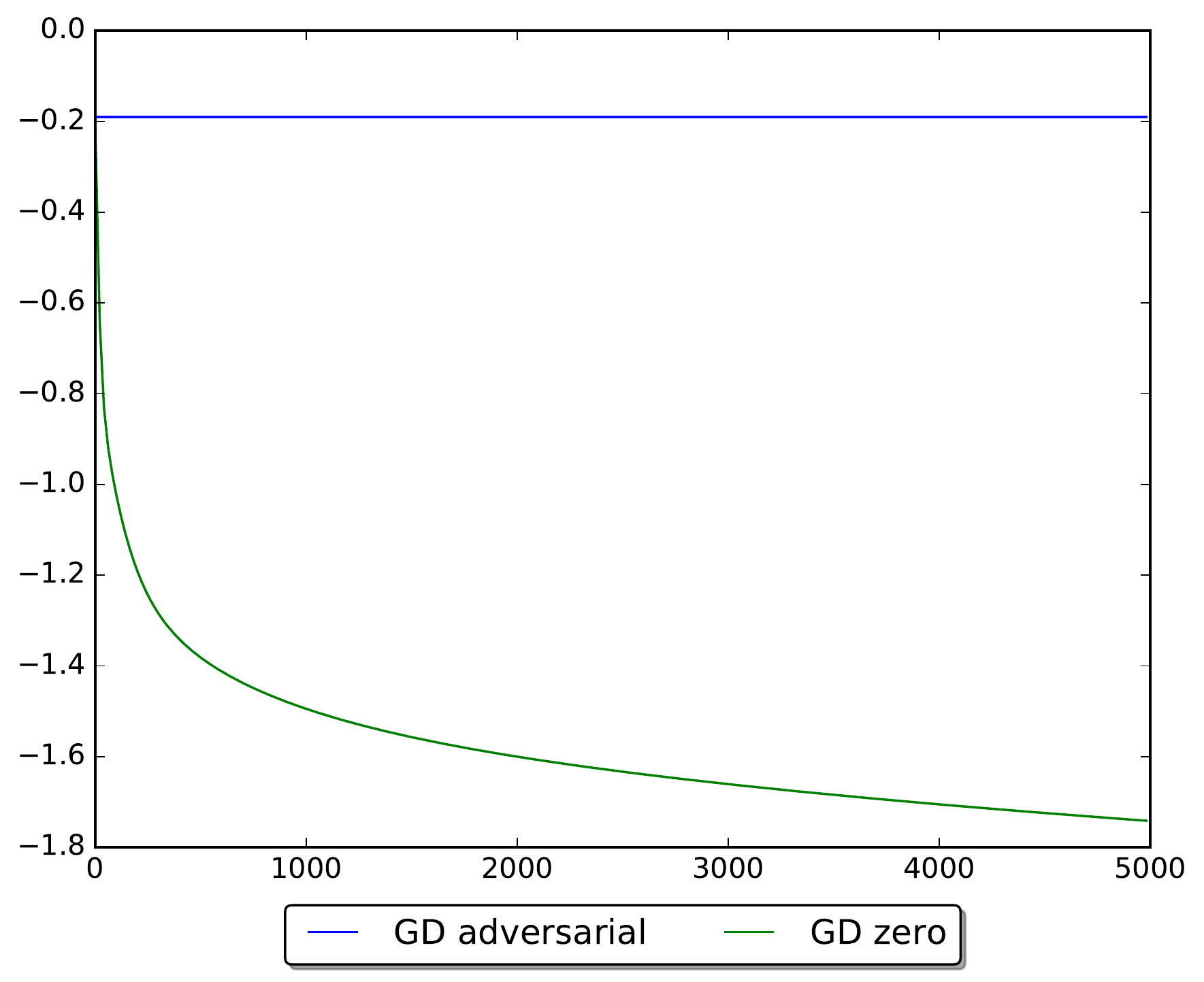} \\
  a9a & covtype &
  gisette & w8a  \vspace{-0.2cm}\\
\end{tabular}
\end{sc}
\end{footnotesize}
\caption{\label{fig:worst-better} Suboptimality of gradient descent iterates for different data sets with initialization $\param^0 = 0$.}
\end{figure}

\subsection{Analysis of Initialization}

We study a single coordinate in the diagonalized problem. 
\begin{align}
\quadratic(\para_j^{t}) - \quadratic^* =  (1-\gamma \lambda_j)^{2t} \; \cdot \; \frac {\lambda_j} 2 \left(\para^0_j - \para_j^* \right)^2
\end{align}
There are two factors here, one that decreases exponentially with $t$ and a constant that depends on the initialization. In an asymptotic setting, only the first would matter, leading to the well-known slow-down in convergence of GD for ill-conditioned problems, in directions where $\gamma \lambda_j$ becomes very small. But what if we can tolerate some suboptimality $\epsilon$?  Then we could try to find an initialization such that the second term would be less than $\epsilon$. There  could be two cases: (i) we devise a smart strategy to find a good $\param^0$ that is sufficiently close to $\param^*$, (ii) we set $\param^0=0$, but can reasonably assume that  $\param^*$ is small. We come back to (i) and first investigate (ii) by noting that 
\begin{align}
\quadratic(0) - \quadratic^* = \frac{1}{2} \frac{\E[YZ_j]^2}{\E^\mu[Z_i^2]}  \le 
 \frac 12 \frac{\EE[YZ_j]^2}{\E[Z_j^2]} = \frac 12 \rho_j^2.
\end{align}
Here $\rho_j$ is the correlation  coefficient between $Y$ and the eigenfeature $Z_j$. We now formulate the following hypothesis: the observed empirical sub-optimality of GD iterates is due to the fact that eigenfeatures $Z_j$ with small variance also have a small correlation with the response $Y$. This seems to be related to a similar assumption made in early stopping \cite{yao2007early} and in general when using norm-based regularization: we do not want to trust features that need to be amplified a lot (low variance, but high output covariance). Hence, a reasonable feature representation should avoid encoding relevant information in such a manner. What we show here though is that such considerations not only avoid overfitting, but also accelerate GD  training.  

In order to subject this idea to a more formal treatment, let us introduce a suitable regularity concept. A dataset exhibits   \textit{$\tau$-bounded response correlation}, if for any eigen-feature $Z_j$ with non-zero variance
\begin{align}
\label{eq:eigen-reg}
\rho_j^2 = \frac{\EE[Y Z_j]^2}{\E[Z_j^2]} \le \tau \E[Z_j^2] = \tau \sigma_j^2
\end{align} 
With this definition, we can immediately see that we can upper bound the suboptimality of GD iterates by making use of the regularity parameter for all eigenfeatures with $\E[Z_j^2] \le \zeta$, where $\zeta$ is an arbitrary threshold. Effectively, $\tau$-boundedness allows us to pay a constant approximation cost for eigenfeatures with small variance and to trade that off with the faster rates obtained for the  remaining eigenfeatures. We capture the gist of this effect in the following lemma. 
\begin{lemma}
\label{lemma:primal-sub}
Assume that the data set has $\tau$-bounded response correlation. Then for any $\zeta \ge 0$, the suboptimality of the GD iterate sequence can be upper bounded as
\begin{align}
\quadratic(\param^t) - \quadratic^* \le 
\frac 12 \left[ r(\zeta) (1-\gamma \zeta)^{2t} + (d-r(\zeta)) \right] \tau \zeta, \quad r(\zeta) := \{j: \E[Z_j^2] > \zeta\}
\end{align}
\begin{proof} Follows directly from the boundedness assumptions. 
\end{proof}
\end{lemma}
Figure \ref{fig:better-better} shows that our notion of regularity seems to agree well with the empirical observations that we pointed out before. Note that regularization increases the variance of eigenfeatures and thus leads to more favorable $\tau$-bounds. However, for clarity, we wanted  to capture this notion for the worst case of $\mu=0$.

\begin{figure}
\begin{center}
\begin{footnotesize}
\begin{sc}
\begin{tabular}{@{\hspace{0.1cm}}c@{\hspace{-0.1cm}}c@{\hspace{-0.1cm}}c@{\hspace{-0.1cm}}c}
\includegraphics[width=0.25\linewidth]{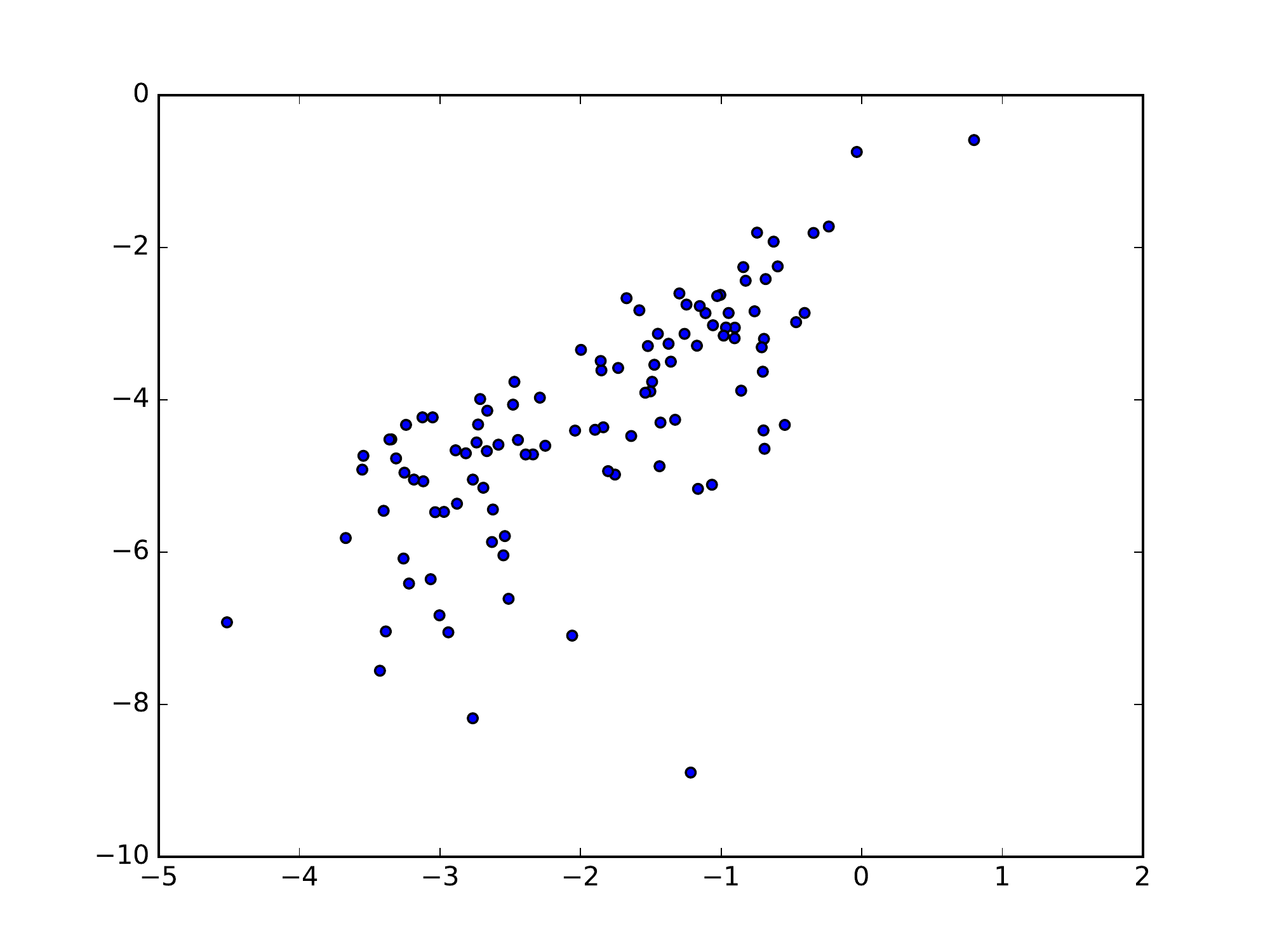}  & 
  \includegraphics[width=0.25\linewidth]{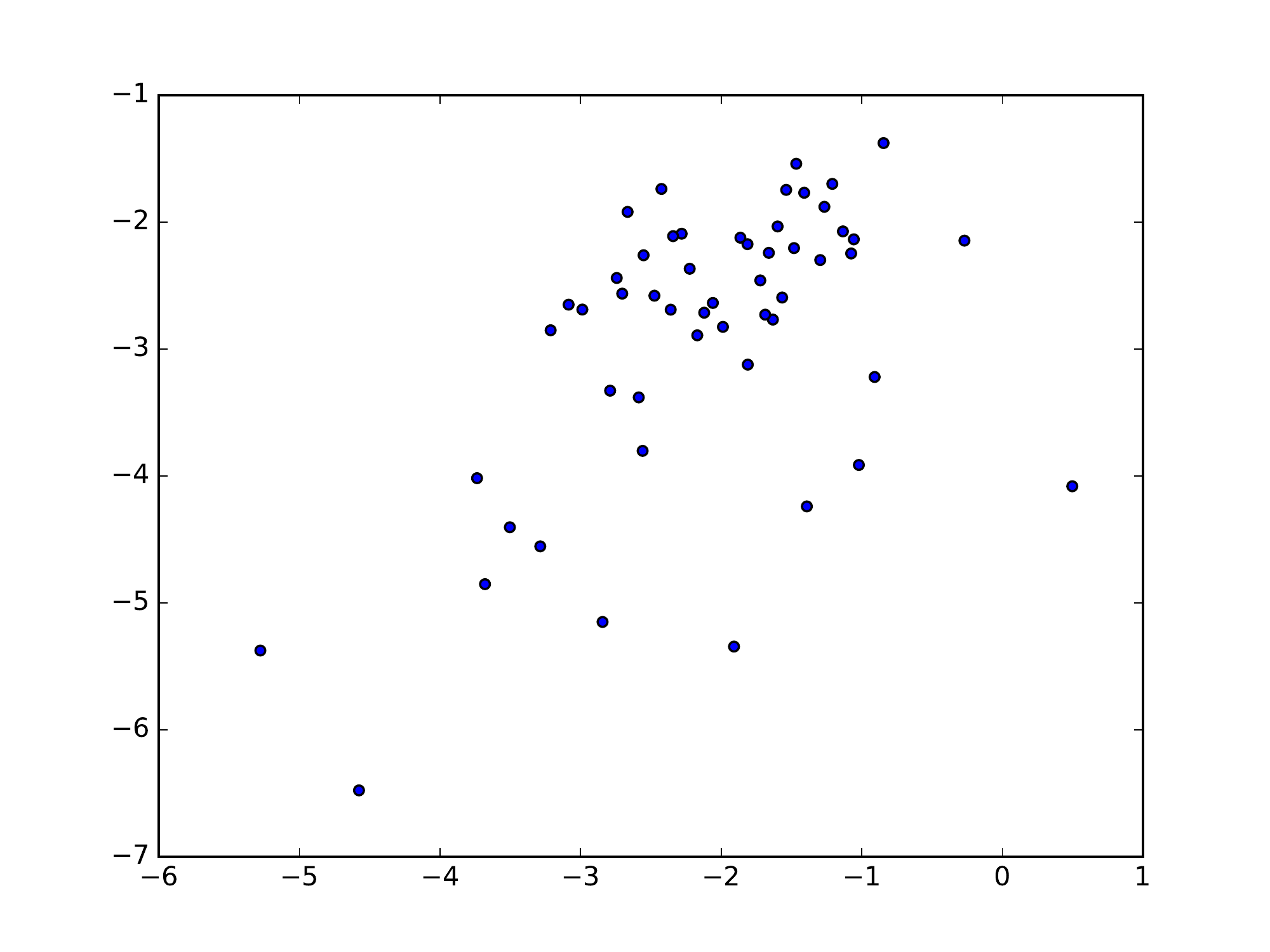} &
  \includegraphics[width=0.25\linewidth]{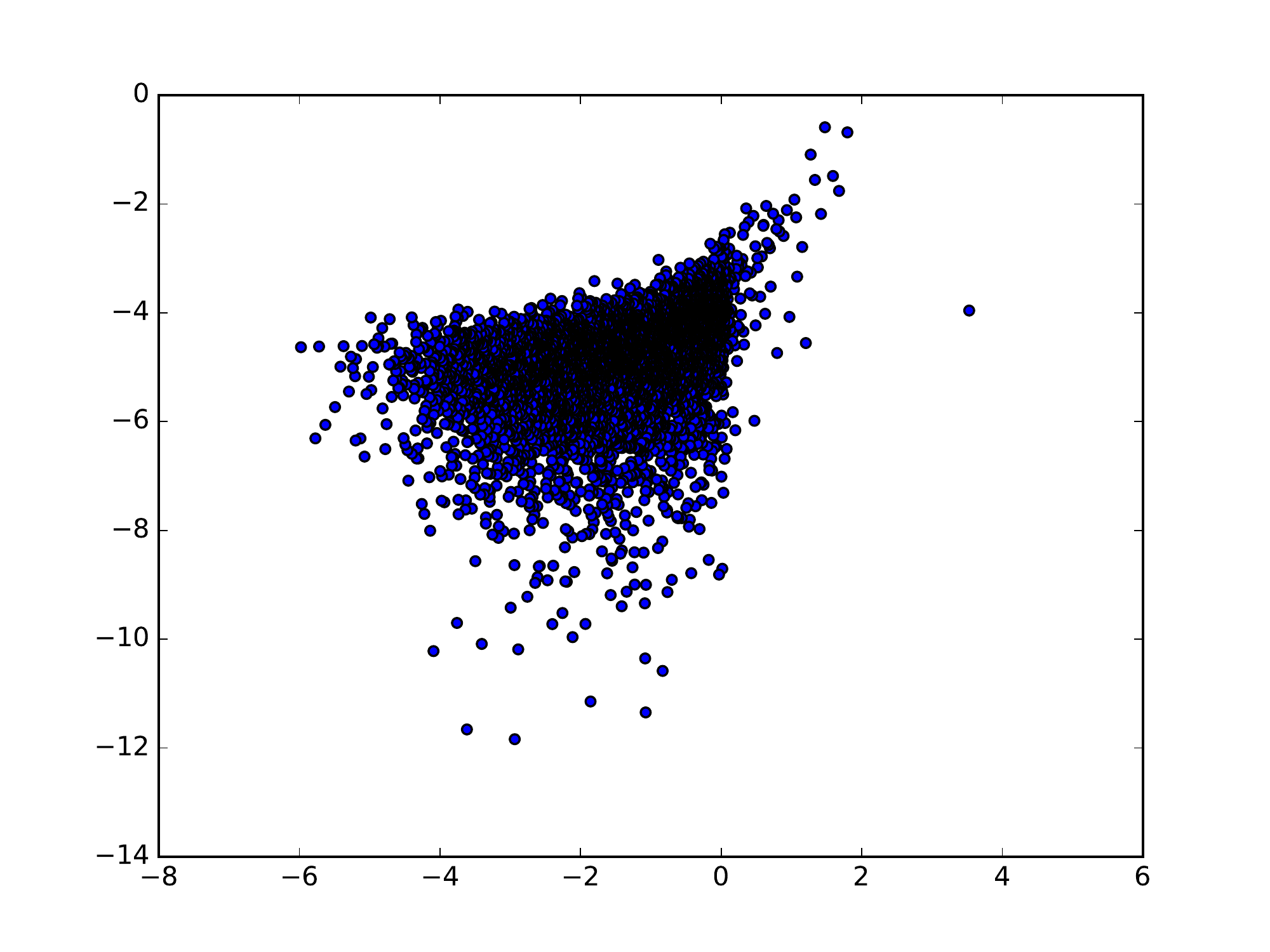}  & 
  \includegraphics[width=0.25\linewidth]{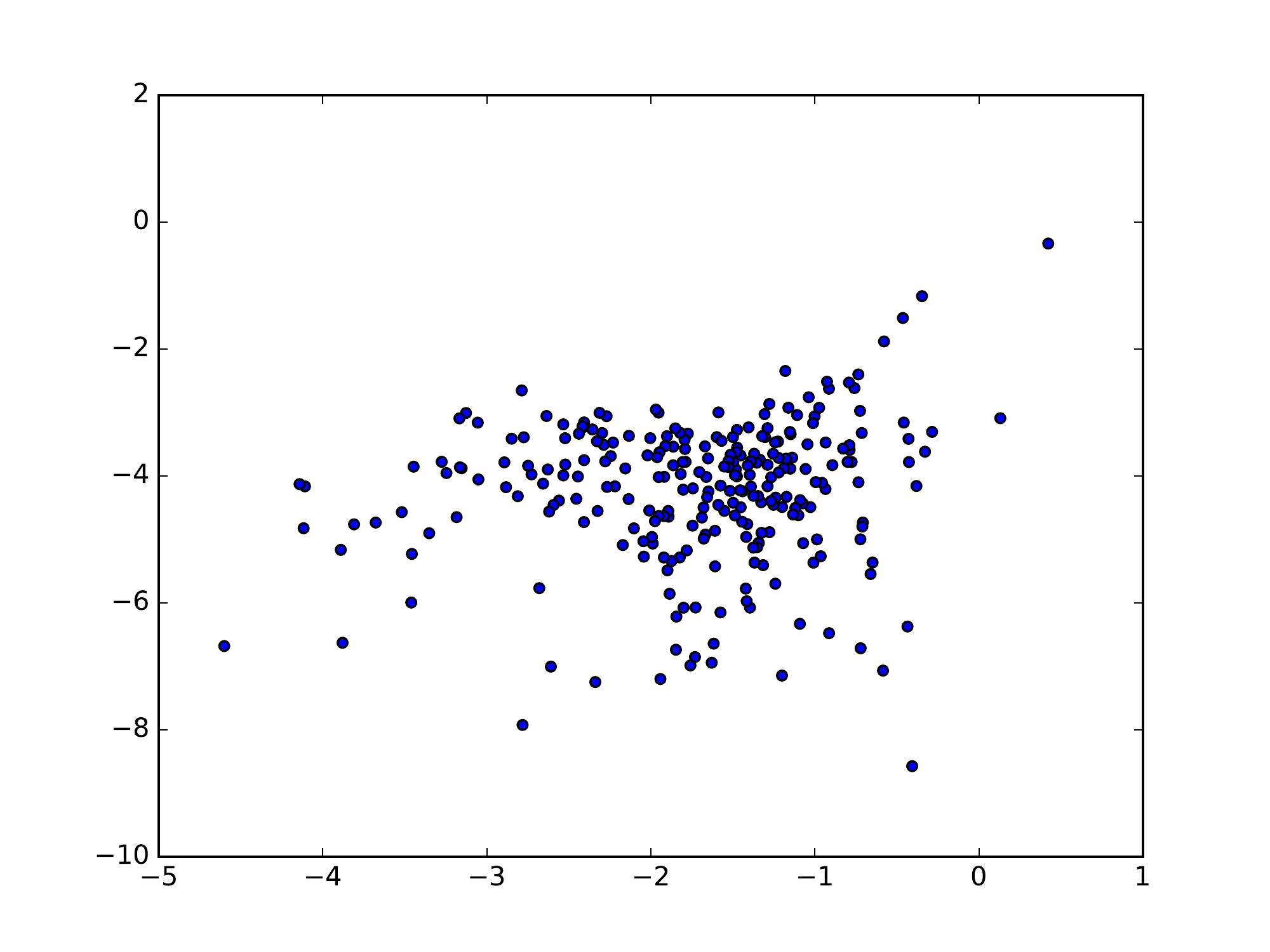} \\
  a9a & covtype &
  gisette & w8a  \vspace{-0.2cm}\\
\end{tabular}
\end{sc}
\end{footnotesize}
\end{center} 
\caption{\label{fig:better-better}  Histogram of $\tau$-boundeness. Each point the scatter plot corresponds to an eigenfeature for which the vertical coordinate is $\log_{10}\left(\E[Y Z_j]^2/(\E[Z_j^2]^2)\right)$ and the horizontal coordinate is the $\log_{10}(\E[Z_j^2]^2)$.}
\end{figure}

\section{Dual Ridge Regression} 
We will now consider the dual problem for ridge regression. It turns out that the  initialization  $\param^0 = 0$, which we analyzed for the primal in the previous section, is not advantageous for the dual. However, we can use the insight about data $\tau$-boundedness to derive an effective initialization for the dual problem.

\subsection{Dual Objective} 
It is straightforward to derive dual objectives for convex optimization problems using Fenchel duality, e.g.~see \cite{dunner2016primal}. For ridge regression, we obtain the following dual objective  with dual parameters $\dparam$
\begin{align}
\quadratic_\mu(\dparam) =  \frac 12 \dparam^\top \G \dparam - \b^\top \dparam, \quad 
n\G = \frac{1}{\mu n} \xm \xm^\top + \I, \quad n \b =  \y
\end{align}
In the dual space, where typically $n \gg d$,  there is an $n-d$ dimensional subspace  $\text{ker}(\mathbf X^\top)$, such that for the corresponding dual eigen-vectors $\vv$, $\vv^\top \G \vv = \frac 1n$. This means that we have at least $(n-d)$ orthogonal directions with a variance of $1/n$ accounting for a total variance of $(n-d)/n$. While in the primal case, there are at most $d$ terms to bound and typically $\mu \propto 1/n$, in the dual we may get a suboptimality contribution that is $\mathbf O(1)$, even under a $\tau$-boundedness assumption. 

Similar to the primal objective, we use (to carry out the analysis, not the optimization!) the following change of variables $\dparam \leftarrow \um^\top \dparam$. This change of variables leads to the following diagonalization of the dual objective: 
\begin{align} 
	\quadratic_\mu(\dparam) 
	&= \frac{1}{2} \dparam^\top \mGamma \dparam  -  \c^\top \dparam, \; \; n\mGamma := \frac{1}{ \mu} \mSigma^\top \mSigma + \I, \quad \c := \frac{1}{n} \um^\top \y, \quad  
	\dparam^{*\mu} = \frac 1n \mGamma^{-1} \c
\label{eq:dual-eigen}
\end{align}
and suboptimality can be written as 
\begin{align} \label{eq:suboptimality_dual}
  \quadratic_\mu(\dparam) - \quadratic_\mu(\dparam^*_\mu) = 
  \frac{1}{2n}\sum_{i=1}^d \frac{\lambda_i}{\mu} \left( \dparam_i - \dparam^*_i\right)^2 
  + \frac{1}{2n}\sum_{i=d+1}^n \left( \dparam_i - \dparam^*_i\right)^2  
\end{align}
Indeed, suboptimality can be decomposed to two terms: the suboptimality in image of the data matrix (the first sum), and the suboptimality in the kernel of the data matrix (the second sum). 
\begin{lemma} \label{lemma:dual_init}
	Suppose that $\dparam^{*\nu}$ is the minimizer of $\quadratic_\nu$, then for all $\nu, \mu>0$:
	\begin{align}
		 \dparam^{*\nu}_j = \dparam^{*\mu}_{j} \;   (\forall j>d) \quad \text{and} \quad
		 (\dparam^{*\nu}_i - \dparam^{*\mu}_i)^2 = n \left( \frac{(\mu-\nu) \sigma_i^2}{(\sigma_i^2 +\mu)(\sigma_i^2 +\nu)}\right)^2\EE[Y Z_i]^2
	\end{align} 
\end{lemma}
\begin{proof} 
	Follows directly  from the closed form solution of minimizers.	
\end{proof}
The above lemma shows that $(n-d)$ coefficients of two dual solutions with respect to different regularizer are exactly the same and the squared difference between other $d$ coefficients is propotional to  $(\mu-\nu)^2$. Computing the minimizer $\dparam^{*\nu}$ is relatively cheaper than $\dparam^{*\mu}$ for $\nu \ll \mu$ as rates often depend on the strong convexity parameter. Hence we suggest initializing gradient descent for $\quadratic_\mu$  with $\dparam^{*\nu}$ for a suitably chosen $\nu$. 
\begin{lemma} 
	Suppose that the data set has $\tau$-bounded response correlations. Then for any $\zeta \ge \frac 1n$, the suboptimality of the GD iterates starting from $\dparam^{*\nu}$ with step size $(n \mu)\gamma$ can be upper bounded as
\begin{align}
\quadratic(\dparam^t) - \quadratic^* \le  
  \frac{\tau \zeta (\nu-\mu)^2}{2 \nu \mu }
  \left[ 
     r(\zeta) (1- \gamma \zeta)^{2t} + (d-r(\zeta)) 
   \right], \; r(\zeta) := | \{ j: \E[Z_j^2] > \zeta\}|
\end{align}
\end{lemma} 
\begin{proof} 
Plugging the result of lemma~\ref{lemma:dual_init} into suboptimality of the dual objective in Eq.~\eqref{eq:suboptimality_dual} concludes the proof.
\end{proof}
The above result closely mirrors Lemma \ref{lemma:primal-sub} with the addition of a factor that depends on the chosen $\nu$ (in relation to $\mu$). A smaller homotopic parameter $\nu$ -- closer to $\reg$ -- enjoys
the convergence with a smaller scaling factor $(\nu-\mu)^2/(\nu\mu)$. However, computing the initial vector $\param_{\nu}^*$ is more expensive for a small $\nu$. Setting that aside, the significance of the homotopic initialization is the fact that the $(n-d)$ directions of eigenvalue $n^{-1}$ do no longer have en effect on the suboptimality. Note that in the eigensystem of Eq.~\eqref{eq:dual-eigen}, we can easily solve (irrespective of $\mu$)
\begin{align}
\alpha^*_j = \uv_j^\top \y  \stackrel{\min}\longleftarrow \frac 1{2n} \alpha_j^2 - c_j \alpha_j, \quad (\forall j>d)
\end{align}
However, this is not a practical computation as the basis vectors $\uv_j$ of the kernel of $\xm$ are not known and the eigen-parameterization is not the one accessible to the algorithm. The trick of the homotopic initialization is that effectively we set  $\alpha_j^*$, without having to perform  the diagonalization of $\xm^\top  \xm$. 

\subsection{Accelerated RCDM by Homotopic Initialization}

In the large scale setting (when the sample size $n$ is large), the gradient step is computationally expensive. Random coordinate descent method (RCDM) \cite{nesterov2012efficiency,shalev2013stochastic}
is computationally more attractive than GD and offers a competitive convergence.  A RCD-step is obtained by a coordinate-wise approximation of the gradient. More precisely, it randomly picks a random coordinate $r \in \{1,\ldots,n\}$ and updates it using the corresponding coordinate of the gradient (denoted by $\quadratic'_{r}(\cdot)$) as
\begin{align}
 \dparam^{+}_{r}  = \dparam_{r} - \gamma_{r}
 \quadratic'_{r}(\dparam),
\end{align}
where $\gamma_r$ is the coordinate-wise step size of RCDM. We ask if homotopic initialization  also accelerates the convergence -- up to a suboptimal solution. The convergence dependency on initialization is more subtle in a stochastic setting, where each optimization step is perturbed by the noise of the gradient approximation. Here, we theoretically prove that  homotopic initialization accelerates RCDM. To this end, we provide a different convergence analysis for RCDM on the dual objective of ridge regression. Later, we will show how this result provides a better convergence rate in the objective using  homotopic initialization.  
\begin{theorem} \label{the:RCDM_convergence}
 Let $\G_{i,j}$ denote the $(i,j)$-th element of the Hessian matrix $\G$. Let the
 parameter vector $\dparam^{(t)}$ be obtained by $t$ RCDM-steps on
 $\quadratic$, starting from $\dparam_0$ with coordinate-wise step sizes
	\begin{align} \label{eq:step-size-lossbound-main}
			\gamma_r^{-1} = \G_{r,r} +\sum_{j} |\G_{r,j}|, \quad \gamma_{\min} := \min_r \gamma_r, \; \; \gamma_{\max} := \max_{r} \gamma_r.
	\end{align}
	For the above  parameters, either the norm of gradient is bounded as
	\begin{align}
		\E \|\quadratic'_\mu(\dparam^{(t)})\|^2 \leq 2 \rho^2
		\left(\frac{\gamma_{\max}}{\gamma_{\min}}\right) \| \dparam_0 -
		\dparam^*\|^2
	\end{align} 
	or suboptimality is bounded as  
	\begin{equation}
		\E\left[\quadratic_\mu(\dparam^{(t)}) - \quadratic^*_\mu\right] \leq
		\frac{1}{2}\left(1 - \frac{\rho \gamma_{\min}
		}{n}\right)^t\left(\quadratic_\mu(\dparam_0) - \quadratic^*_\mu\right) +
		\frac{1}{2} \rho \left(\frac{\gamma_{\max}}{\gamma_{\min}}\right)\| \dparam_0 - \dparam^*\|^2
	\end{equation}
	for every $\frac{1}{n} \leq \rho \leq \| \G \| $ (expectation is
	over the random choice of coordinates).
\end{theorem}
Suppose that $\sum_{j} |\x_i^\top \x_j | \leq B$ for all
$i$. Using coordinate-wise step sizes of Eq.~\eqref{eq:step-size-lossbound-main},
the classical analysis of RCDM \cite{nesterov2012efficiency} suggests
the rate $1 - (n + 2B/\mu)^{-1}$, while our analysis improves the rate by a factor of $\rho/n>1$ to
$1- (\rho/n)(n + 2 B/\mu)^{-1}$. This improvement is not only with
respect to the strong-convexity factor $\mu$ (e.g., as in the  catalyst method \cite{lin2015universal}), but also with respect to
the sample size~$n$.
\\
The acceleration of Theorem
\ref{the:RCDM_convergence} is up to a suboptimal solution: either the norm of the 
gradient is bounded by $\bigo(\rho^2 \| \dparam_0 - \dparam^*\|^2)$ or the
suboptimality is at most $\bigo(\rho \| \dparam_0 - \dparam^*\|^2)$. Both the 
suboptimality bounds highly depend on the initial Euclidean distance to the
minimizer. In fact, one can take advantage of the accelerated rate only if the initial distance $\| \dparam_0 - \dparam^*\|^2$ is
small. We observed that this
distance is significantly large for initialization with the all-zero vector.
Nonetheless, we theoretically bound this distance for homotopic
initialization. 
\begin{lemma} \label{lem:homotopic_initial_path_bound}
Assume that the data set has $\tau$-bounded response correlation.  Then for any $\zeta$ the distance between two minimizers of the dual objective with different regularizers is bounded as
		\begin{align}
			\| \dparam^*_{\nu} - \dparam^*_{\mu} \|^2  \leq  \left((\mu-\nu)^2/\nu\right) (d-r(\zeta)) 
			  n  \tau \zeta.
		\end{align}
\end{lemma}
Lemma \ref{lem:homotopic_initial_path_bound} implies that for a smaller $\nu$, closer to $\mu$, the
initial distance is smaller and hence the acceleration of homotopic
initialization is up to a better suboptimality. However,  convergence to
$\dparam^*_{\nu}$ is slower as it directly relates to
the $\nu$. In our experiments, we observed that setting $\nu = \frac{1}{4}
\sqrt{\mu}$ -- which is considerably cheaper than computing the minimizer
$\dparam_{\mu}^*$-- provides a significant acceleration.
The coordinate-wise step sizes of Eq.~\eqref{eq:step-size-lossbound-main} are quite
pessimistic. In our experiments, we used larger step sizes $\gamma_r =
\G_{r,r}^{-1}$, which is equal to the coordinate-wise Lipschitz constants
\cite{nesterov2012efficiency} of the dual ridge objective. 
\section{Generalized Linear Model }
The accelerated convergence for an approximate solution also applies to generalised linear models. Given a convex differentiable smooth function $\varphi$, a generalized linear model (GLM) aims at minimizing the non-quadratic objective \cite{erdogdu2016scaled}
\begin{align}
 \risk(\glmp) = \E \left[ \varphi(\x_i^\top \glmp) - \labels_i
 \left( \x_i^\top \glmp \right)\right],
\end{align} 
where the
expectation is over the population distribution of the data. Let $\varphi^{(k)}(\cdot)$ denote the $k$-th derivative of $\varphi(\cdot)$. The above formulation obtains logistic regression, with   
${\varphi(a) = \log(1+\exp(a))}$, and ridge regression, with $\varphi(a) = a^2$.  To take advantage of the boundedness property of a dataset, we modify the gradient descent step as 
\begin{align} 
\glmp^{+} = \glmp - \gamma_t \risk'(\glmp) -\eta_t \E\left[ y \x \right]
\end{align}
 where $\gamma_t$ and $\eta_t$ are two step-sizes of the modified gradient step. Indeed, we suggest to use a biased gradient step to accelerate the initial convergence. If the input vectors $\x_i$ are drawn i.i.d from a
gaussian distribution, then we can obtain an accelerated convergence of the modified gradient descent on the generalized linear model.
 This is made precise in the next lemma.
 \begin{lemma} \label{lemma:glm_zero_init}
 Suppose that inputs are drawn i.i.d from a multivariate normal
	distribution with mean $0$ and covariance matrix $\Sigma$, i.e. $\x_i \sim
	N(\vec{0},\Sigma)$ where $\| \Sigma \| = L$.   Assume that $|\varphi^{(2)}(a)|\leq \phi$. If the dataset is $\tau$-bounded, then there is a step size schedule for modified gradient descent such that iterates of GD starting from zero obtains the following suboptimality bound for all $0<\zeta<L$:
	\begin{align}
		& \risk(\glmp^{(t)	}) - \risk^* \leq c_{\glmp^*}\tau \phi \left(  (1-\zeta/L)^{2t} \left(
		1 - r(\zeta)\right) + r(\zeta) \zeta \right), \\ 
		& r(\zeta) := | \{ j: \E[Z_j^2] > \zeta\}|,  \; c_{\glmp^*} := ( \E_{\x} \left[ \varphi^{(2)}(\x^\top \glmp^*)\right])^{-1}.
	\end{align}
\end{lemma}
 Although the data is assumed to be generated from a normal distribution in the last lemma,  we believe that
this result can be extended to an arbitrary distribution using  zero-bias
transformations (see \cite{erdogdu2016scaled} for more details). Furthermore, we believe that homotopic initialization can also obtain a similar acceleration, up to a suboptimal solution, on the dual objective of GLM. 
\section{Initialization for Deep Neural Networks}
 So far we have seen that if a dataset has the boundedness property, then 
 initialization with the all-zero vector will accelerate optimization, on the primal objective, up to a
 sub-optimal solution. We proved this for ridge regression as well as for generalized linear models in a simplified setting. 
Our experiments show that non-linear features of a trained neural network provide such a representation of the data (see figure \ref{fig:neural_net}).
 This observation justifies the effectiveness of one of the most common
 initialization schemes of deep neural networks: Use resulting weights of a trained
 shallow network to initialize a deep neural network
 \cite{Goodfellow2016deepbook}. We attribute the gain of such
 initialization to the boundedness of features obtained by a smaller network.
 Suppose that we used a trained network with $N-1$  layer to initialize the first $N-1$
 layers of a larger network (assume that the number of hidden units is the same as the number of output unites).
  If we freeze weights of the first $N-1$ layers, then optimization with respect to 
 weights of the last layer is a convex programming task using a convex loss function, which can be modelled by a generalized linear model. Since features obtained by a trained network are $\tau$-bounded, we expect an initial acceleration of GD
 by initialising the last layer to all-zero. In
 practice, we do not need to freeze weights of early layers because back propagation
 naturally causes small changes in early layers due to the vanishing gradient phenomenon
 \cite{hochreiter2001gradient}. Furthermore, we do not need to increase the number of layers by one.
 In our experiments, we have observed that even a network with $N/2$ layers yields a good
 initialization for a network with $N$ layers (see figure~\ref{fig:neural_net}). Since
 optimization of shallow networks is relatively cheaper than a deep neural network, this initialization is relatively cheap to
 compute. Our argument can also be extended to the layer-wise training for deep
 neural networks \cite{bengio2007greedy}. Indeed, the boundedness property is a
 statistical property of a representation that plays an important role in
 optimization for machine learning.

\section{Experimental Results}  \label{sec:experiments}
\subsection{Datasets and Protocol}
\begin{wraptable}{r}{8cm}
\label{table:datasets}
\caption{{\it Details of the datasets.}}
\label{table:datasets}
\vskip 0.15in
\begin{center}
\begin{footnotesize}
\begin{sc}
\begin{tabular}{lccc} 
\hline
Dataset & Size & Dimension & Boundedness ($\tau$) 
\\
\hline
a9a & 32561 & 123 & 0.2\\
covtype & 581012 & 54 & 5.5 \\
gisette & 6000 & 5000 & 100\\
ijcnn1 & 49990 & 22 & 6 \\
w8a & 49749 & 300 & 1.4 \\
SUSY & 5000000 & 18 & 1.2 \\ 
\hline
\end{tabular}
\end{sc}
\end{footnotesize}
\end{center}
\vskip -0.1in
\end{wraptable}
In this section, we present our empirical results 
 on real datasets, selected
from LIBSVM library \cite{libsvm} (see Table~\ref{table:datasets} for more
details). We measured the boundedness constant $\tau$ for these datasets, which
is used in our analysis. The regularizer is $\mu = 10^{-6}$ for all
 datasets except for {\sc Gisette}, which has relatively less
 samples, and hence we used regularizer $\mu = 10^{-3}$.

\subsection{Initialization for Dual Objective}
We ran an experiment to assess the advantage of the homotopic initialization on dual
programming. Our experiment is on initialization of Random Coordinate
Descent Method (RCDM) for optimizing the dual objective of ridge regression.
Throughout all the experiments, we used the homotopic parameter $\nu = 0.25\sqrt{\mu}$, which is computationally a favourable choice.  We used coordinate-wise step sizes
$\gamma_r = \G_{r,r}^{-1}$, which are equal coordinate-wise Lipchitz constants and its a common choice for RCDM
~\cite{nesterov2012efficiency}.
Our sampling scheme of coordinates is random permutation in each epoch.
Figure~\ref{fig:homotopic_rcdm_ridge} shows the dual suboptimality, primal
suboptimality, and test error through optimization. To compute the primal suboptimality, we mapped the dual parameter to the primal one using the
mapping $\param^{(t)}= (n \mu)^{-1} \xm\dparam^{(t)}$ (used in
\cite{shalev2013stochastic}). This mapping obtains the primal minimizer given
the dual minimizer, i.e. $\param^{*} = (n \mu)^{-1}
\xm\dparam^{*}$ holds. For the test error, we computed the average of
the squared loss on the test data.
Overall, we observe that the homotopic initialization causes a worse
initial primal and dual suboptimality compared to the initialization with
all-zero vector. Nonetheless, the suboptimality decays quickly and reachs a better suboptimality compared to the initialization with zero. Indeed, the homotopic initial vector lies in the space of coarser eigen-features (features associated with larger eigenvalues)  that accelerates optimization.  Although the accelerated convergence is up to a suboptimal solution,  this suboptimal solution achieves a test error that is comparable to 
the test error of the empirical minimizer on most of datasets. The acceleration, obtained by
the homotopic initialization, is related to the boundedness factor $\tau$ of
the dataset reported in table~\ref{table:datasets}. For example, the
homotopic initialization on {\sc gisette}, which has a large 
$\tau$, obtains relatively less gain. 

Although our analysis of homotopic initialization was limited to the ridge regression problem,
our experiments on dual SVM show the same behaviour for this initialization.  Our results on dual SVM are included in the appendix. 
\subsection{Initialization for neural networks}
We ran an experiment to highlight the role of boundedness of a representation in the convergence speed of gradient descent on neural networks. In this experiment, we train a multilayer perceptron (MLP) with 10 hidden layers and 100 hidden units in each layer. We used two datasets {\sc MNIST} and {\sc CIFAR} with 50000, and 20000 samples, respectively. Here, we compared the boundedness of features obtained from the last layer of the network before and after training. Our observations show that a trained neural network provides boundedness (see figures \ref{fig:neural_net}.a and \ref{fig:neural_net}.b).  Based on this observation, we use a layer-wise initialization strategy. We trained a MLP with 5 layers and the same number of hidden units. Then we initialized the first 5 layers of the main network (with 10 layers) by the trained weights of the smaller network and we set the rest of weights to zero.We compared the convergence of GD with step size 0.1 using these two different initialisation schemes. The layer-wise initialization, which yields a representation with the boundedness property, significantly accelerates the initial convergence of GD. 
\subsubsection*{Acknowledgments} 
We would like to thank Aurelien Lucchi, Octavian Ganea, and Dünner Celestine for helpful discussions. 
\begin{figure}
\begin{footnotesize}
\begin{sc}
\begin{tabular}{@{}S@{\hspace{0.075cm}}D@{\hspace{0.075cm}}D@{\hspace{0.075cm}}D@{\hspace{0.075cm}}D@{\hspace{0.075cm}}D@{\hspace{0.075cm}}D}
& a9a & covtype & gisette & ijcnn1& w8a & SUSY
\\ 
 Dual \ & \includegraphics[width=0.9\linewidth]{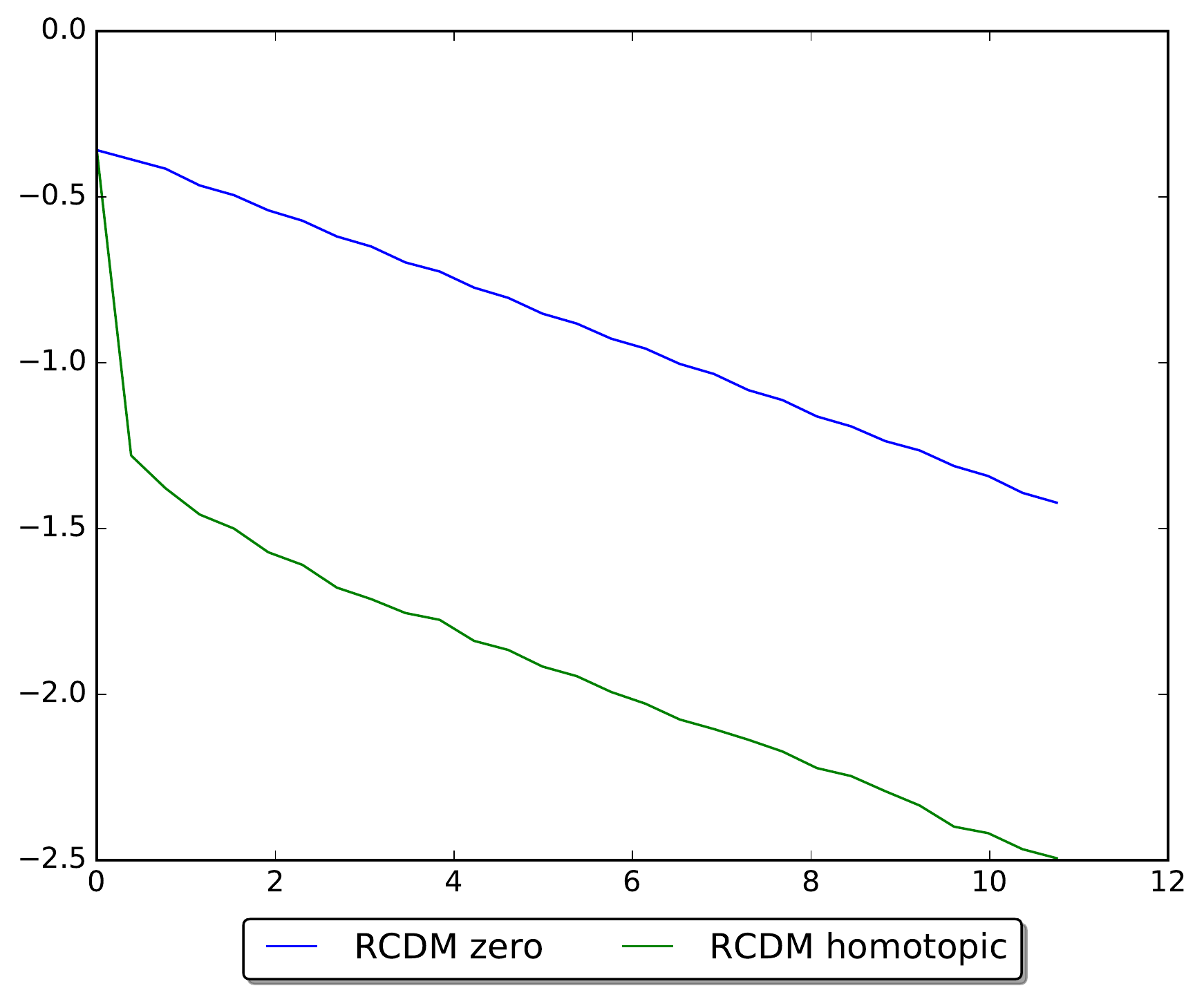}& 
 \includegraphics[width=0.9\linewidth]{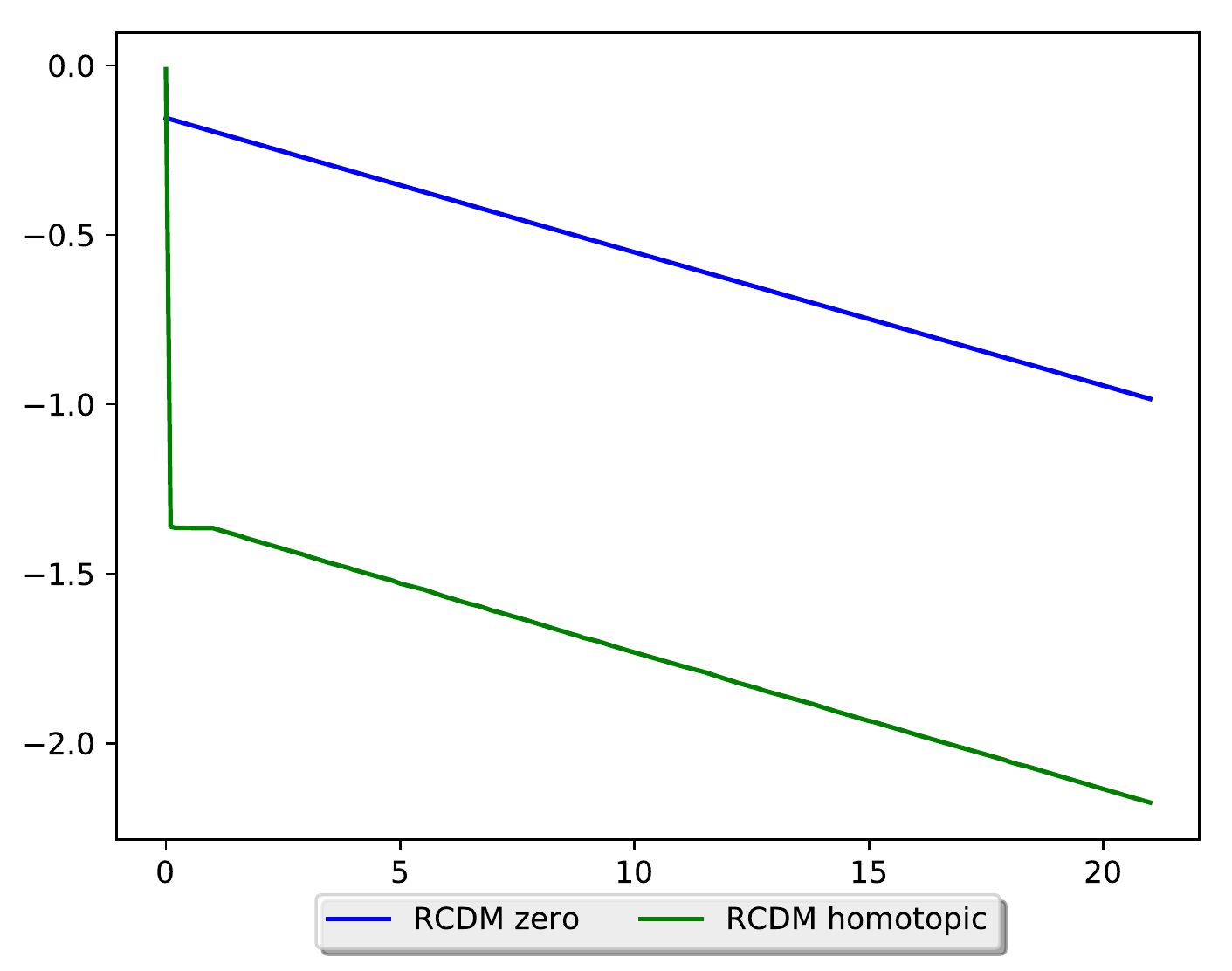} & 
  \includegraphics[width=0.9\linewidth]{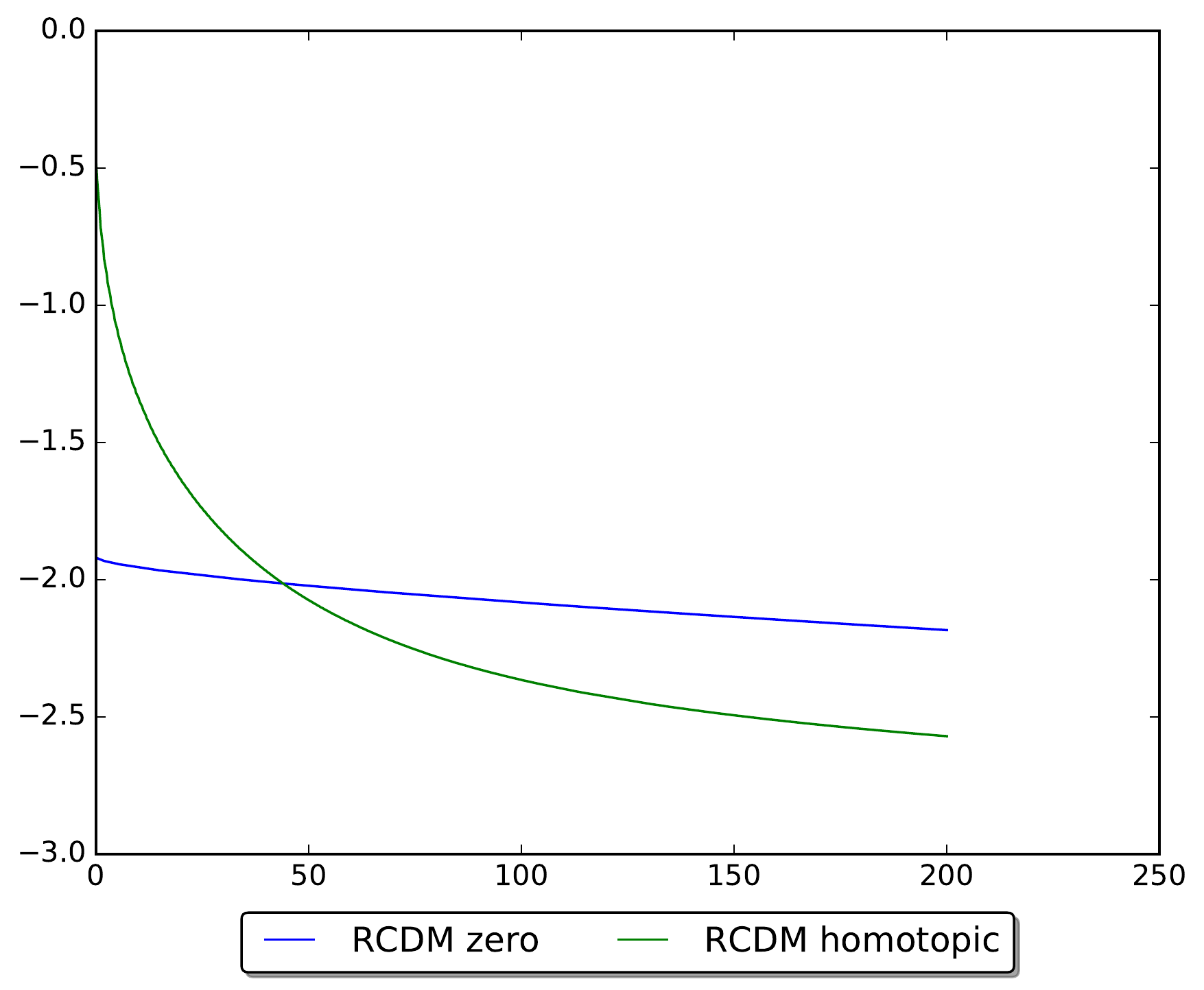}&
   \includegraphics[width=0.9\linewidth]{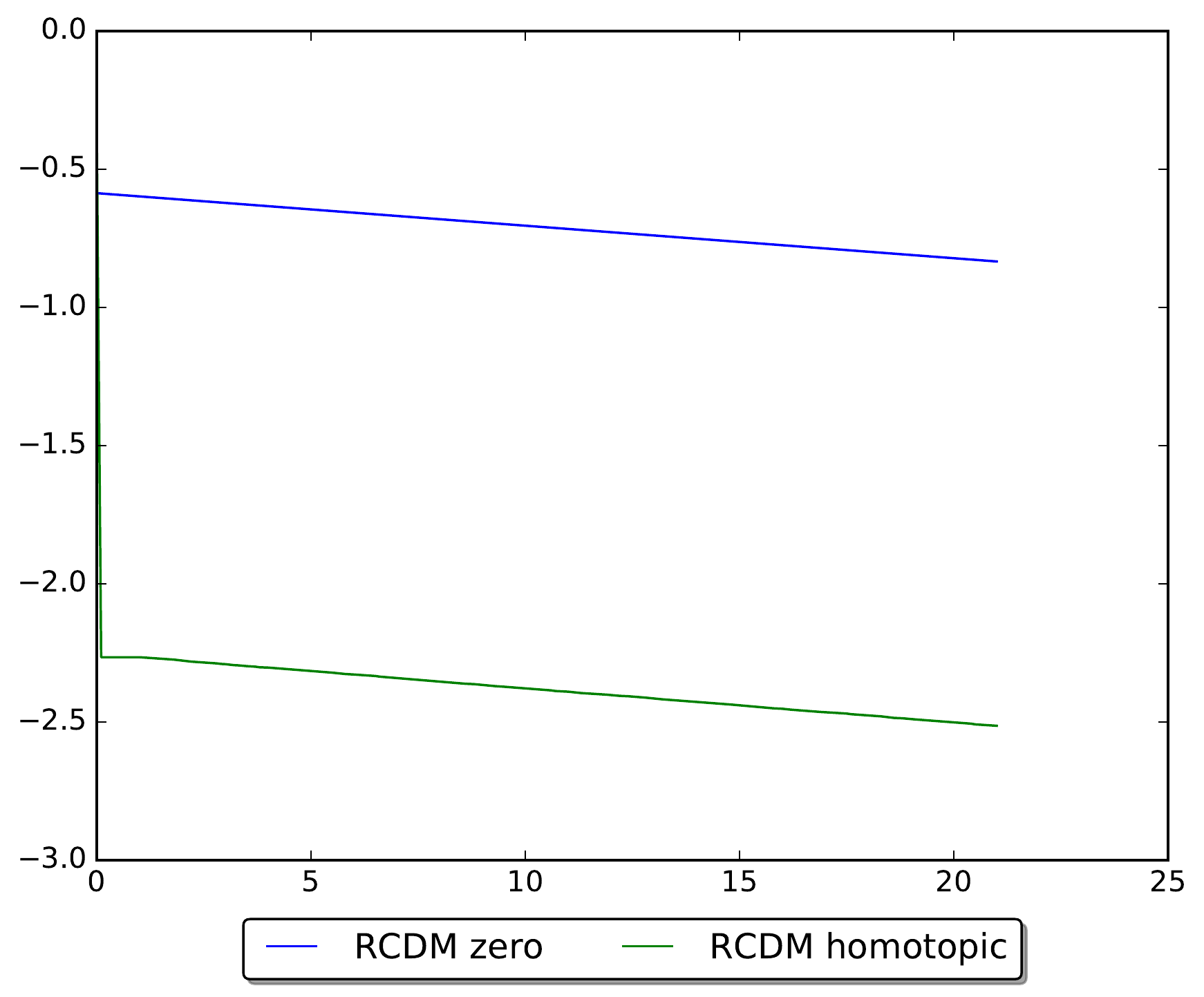}&
   \includegraphics[width=0.9\linewidth]{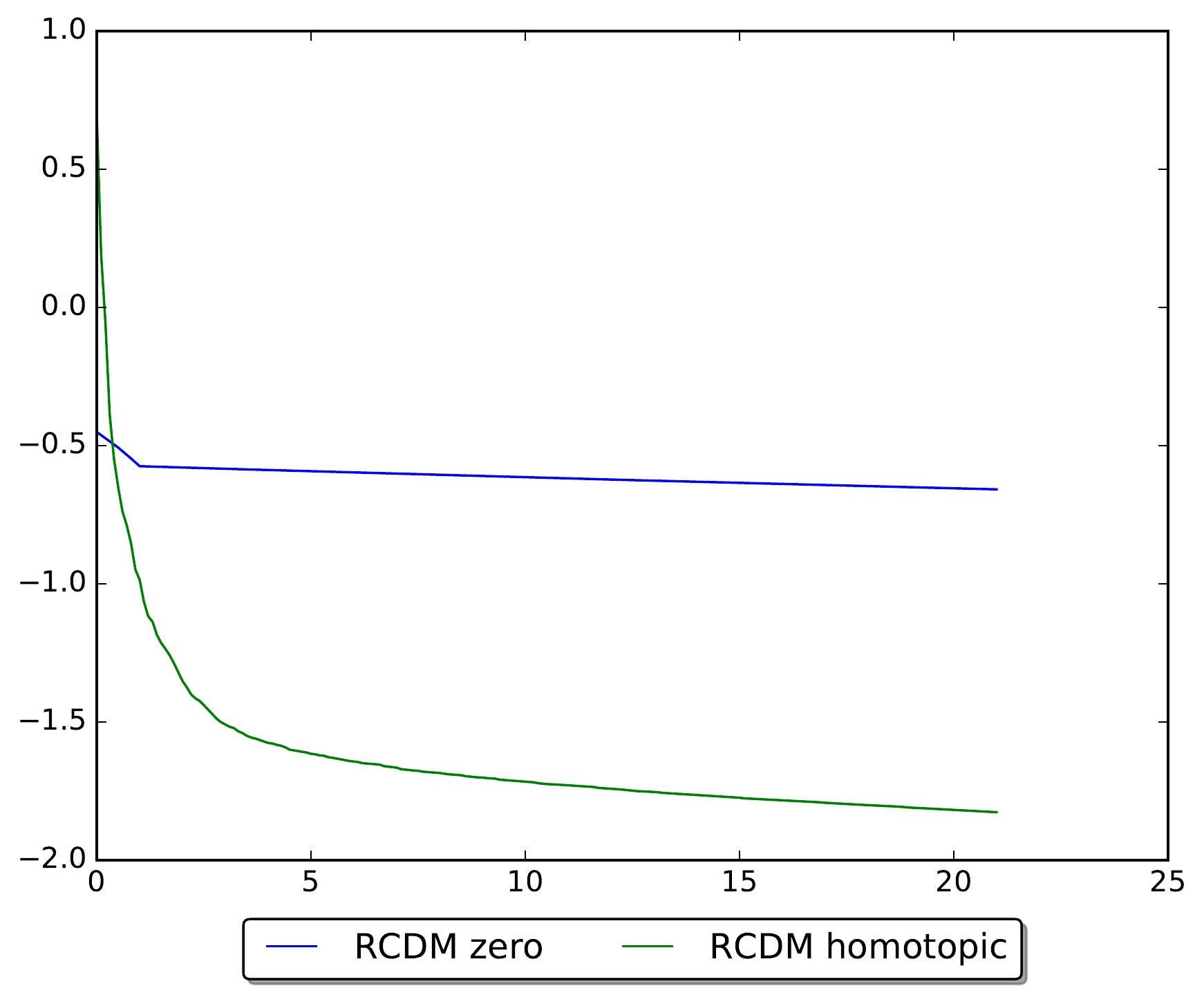}& 
   \includegraphics[width=0.9\linewidth]{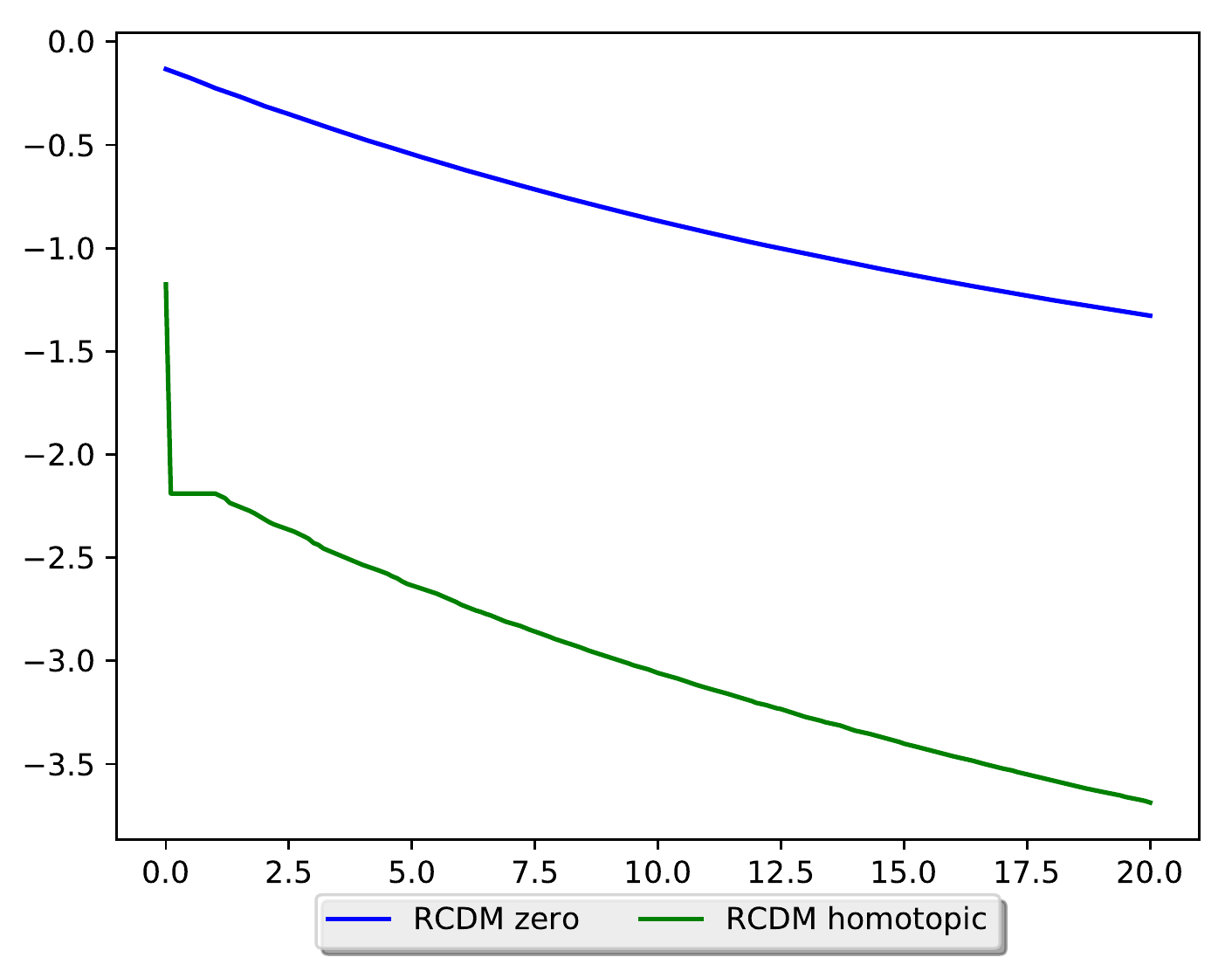}
   \\
   Primal \ & \includegraphics[width=0.9\linewidth]{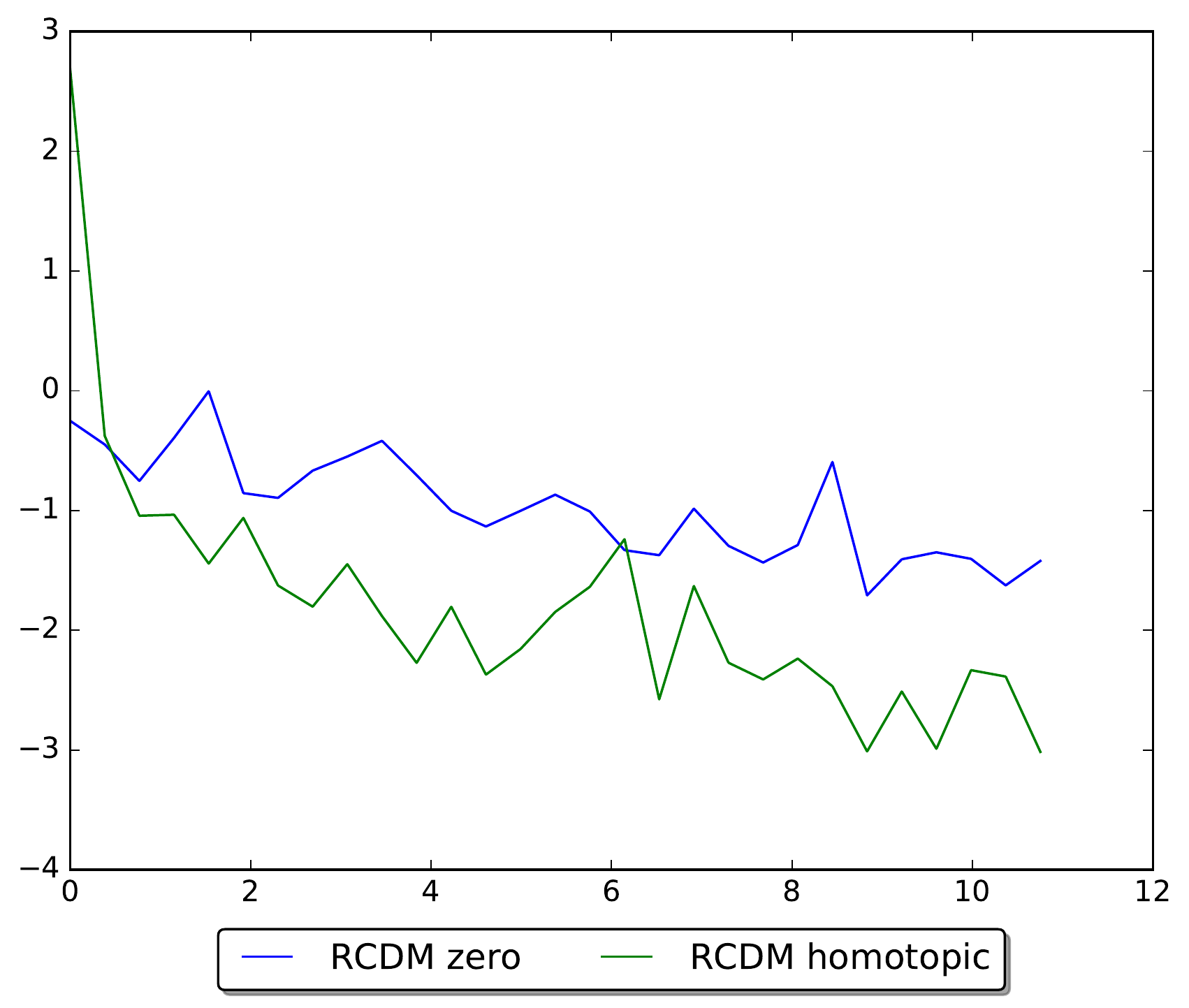}& 
 \includegraphics[width=0.9\linewidth]{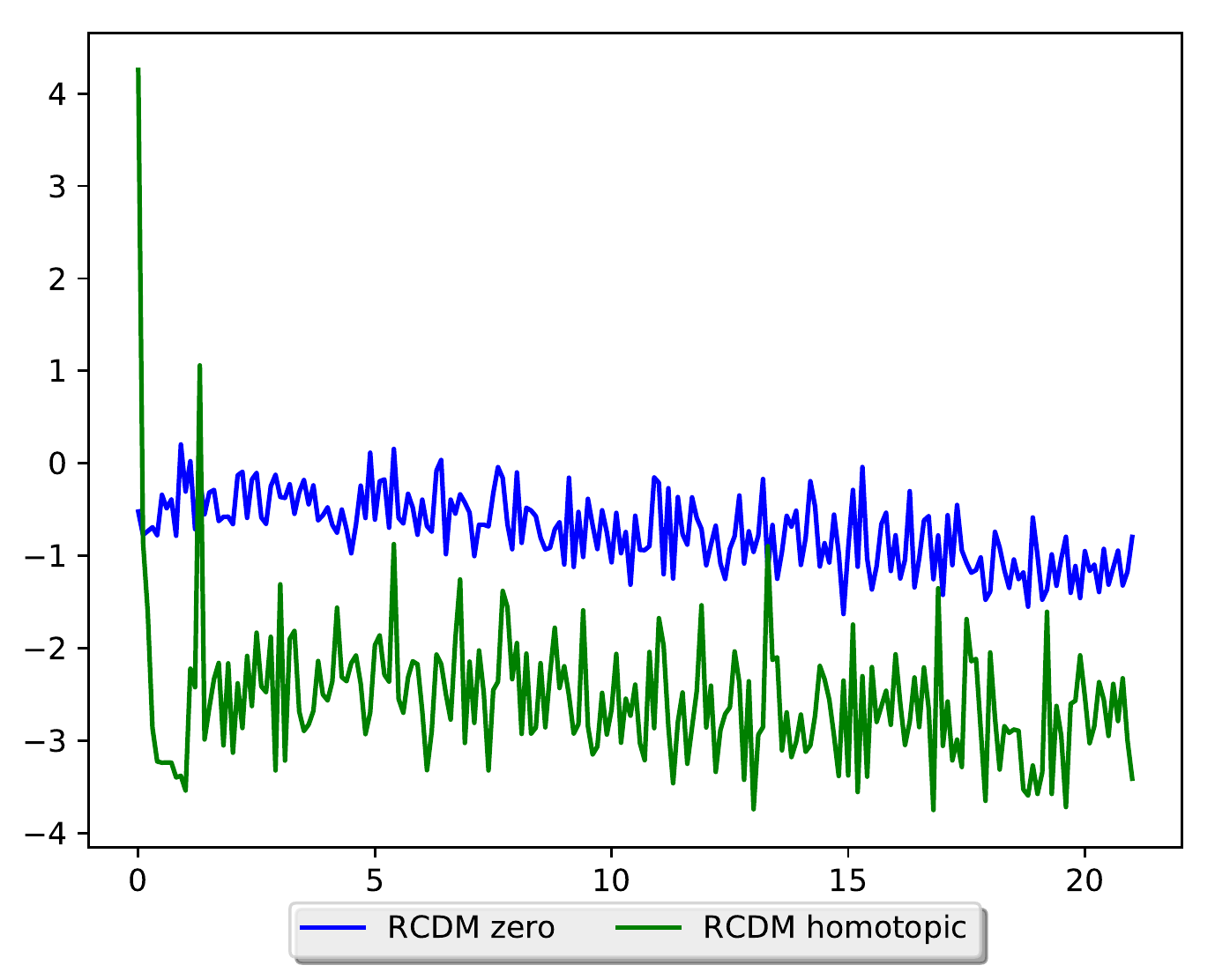} & 
  \includegraphics[width=0.9\linewidth]{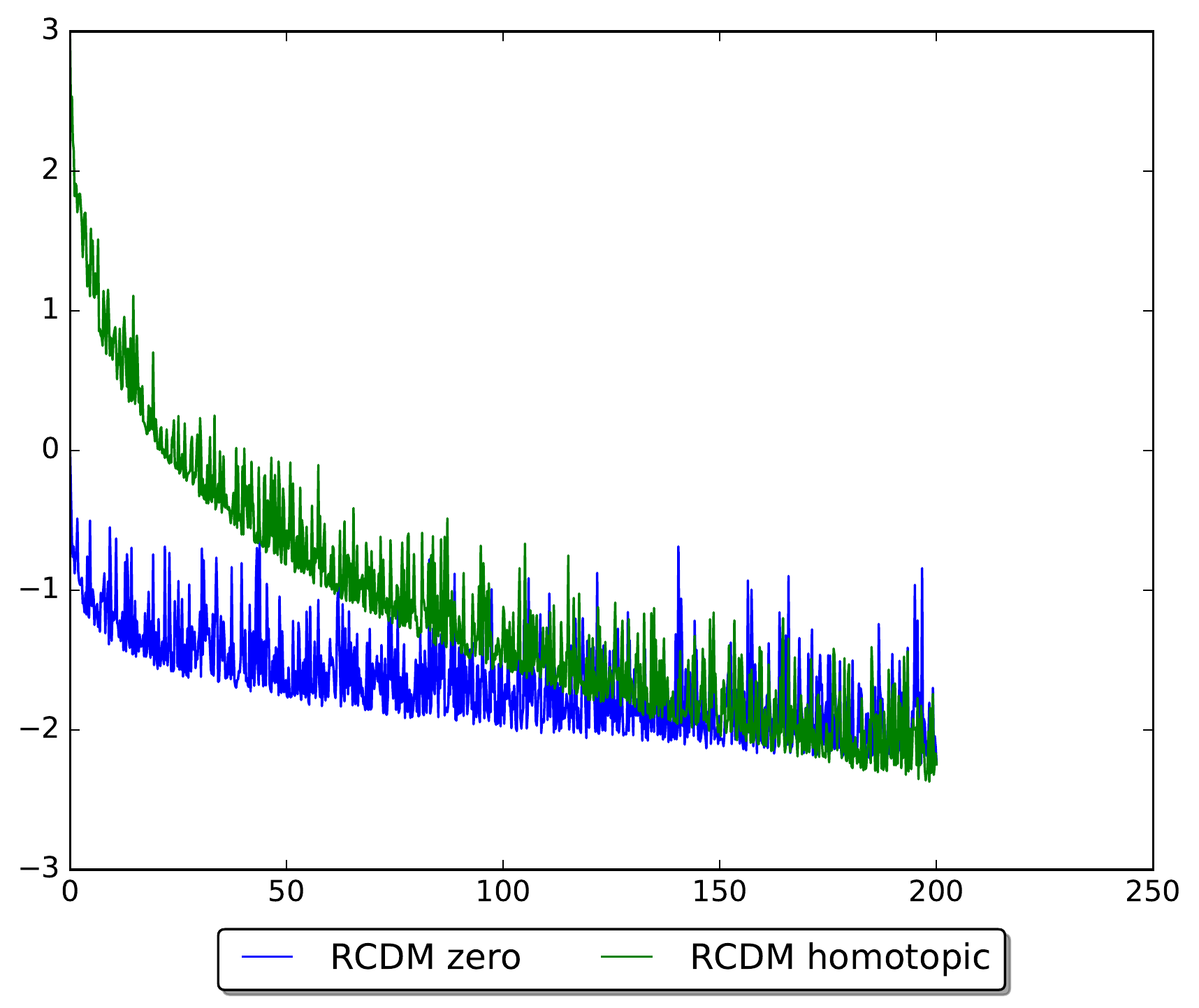}&
   \includegraphics[width=0.9\linewidth]{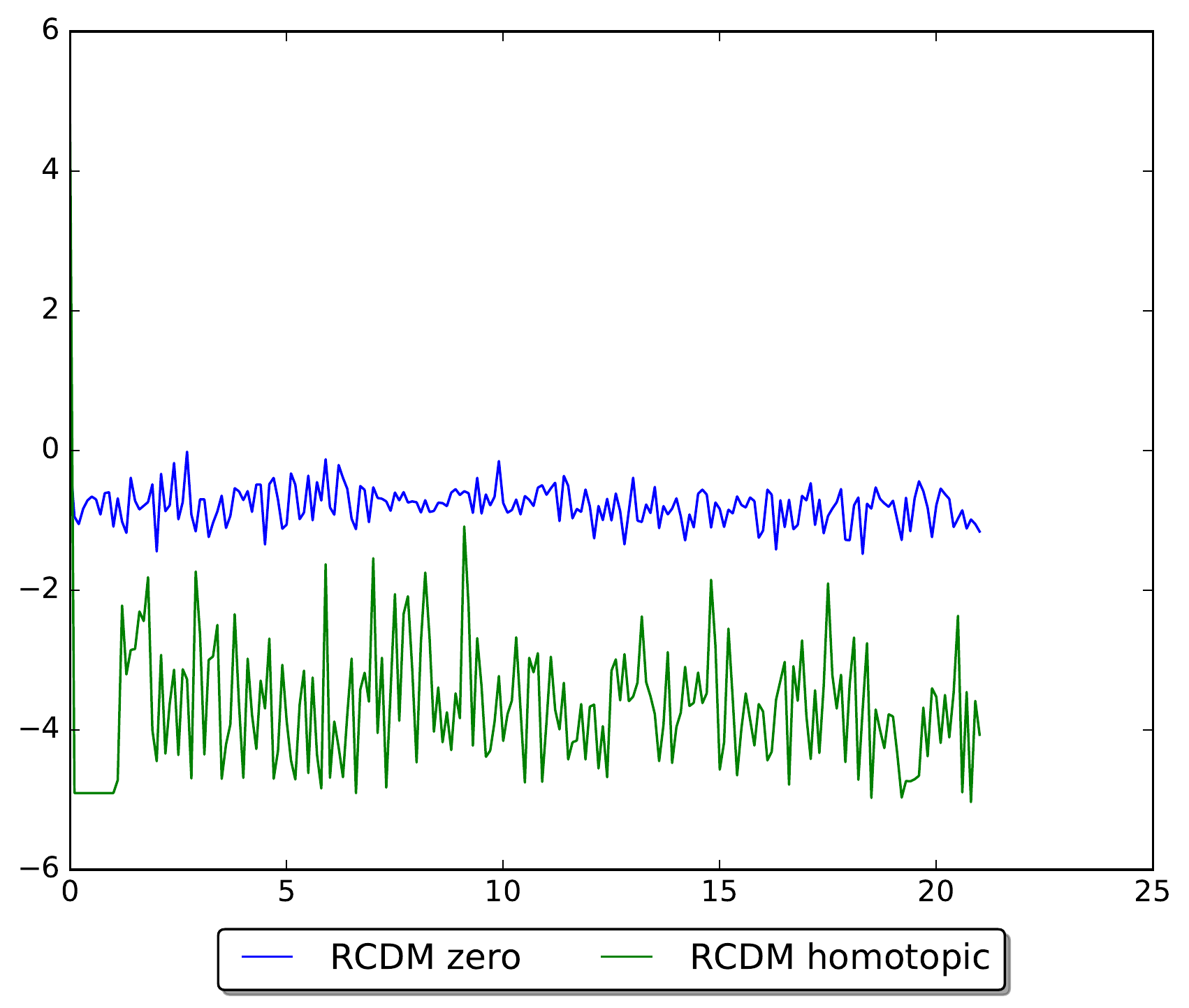}&
   \includegraphics[width=0.9\linewidth]{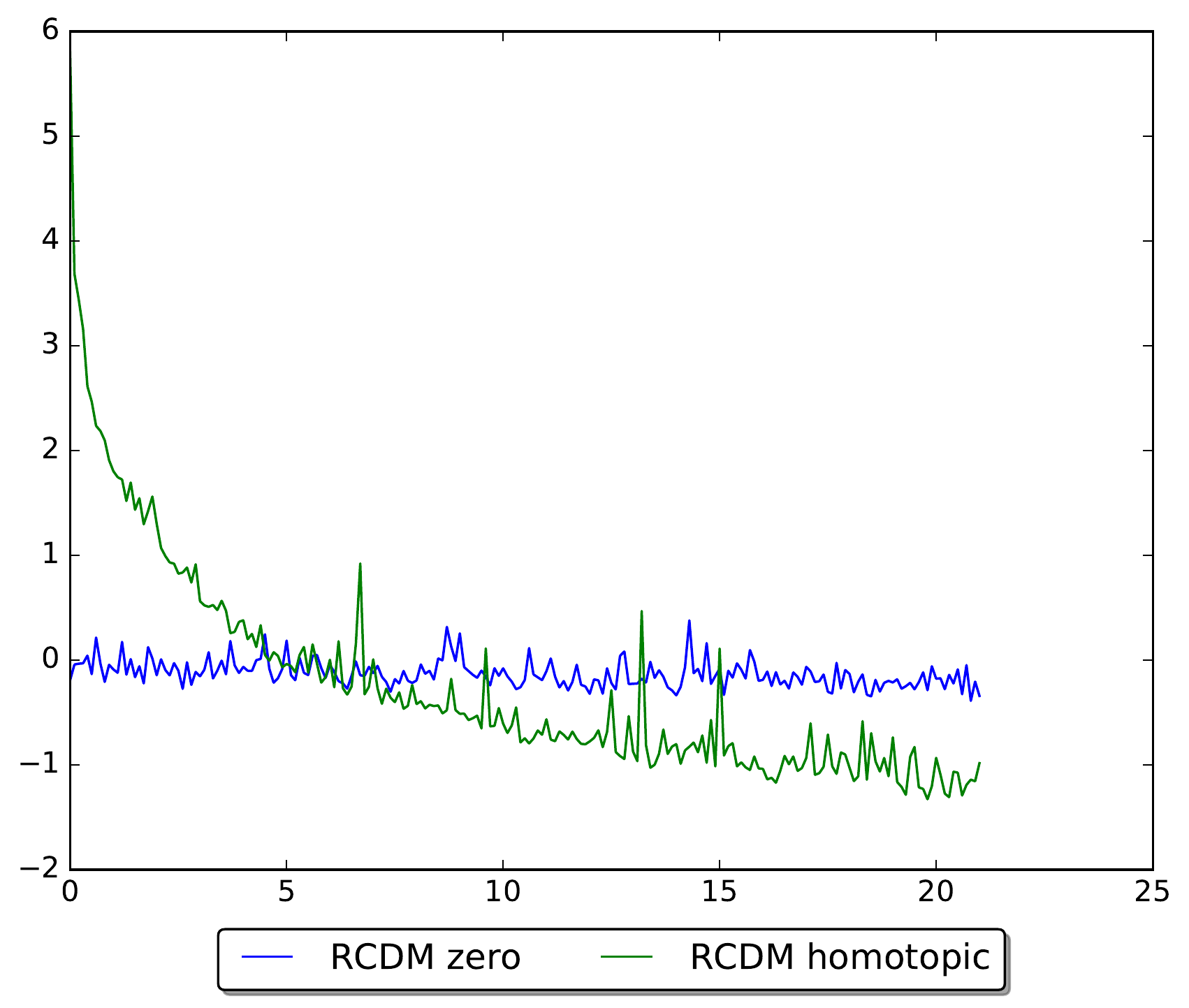}& 
   \includegraphics[width=0.9\linewidth]{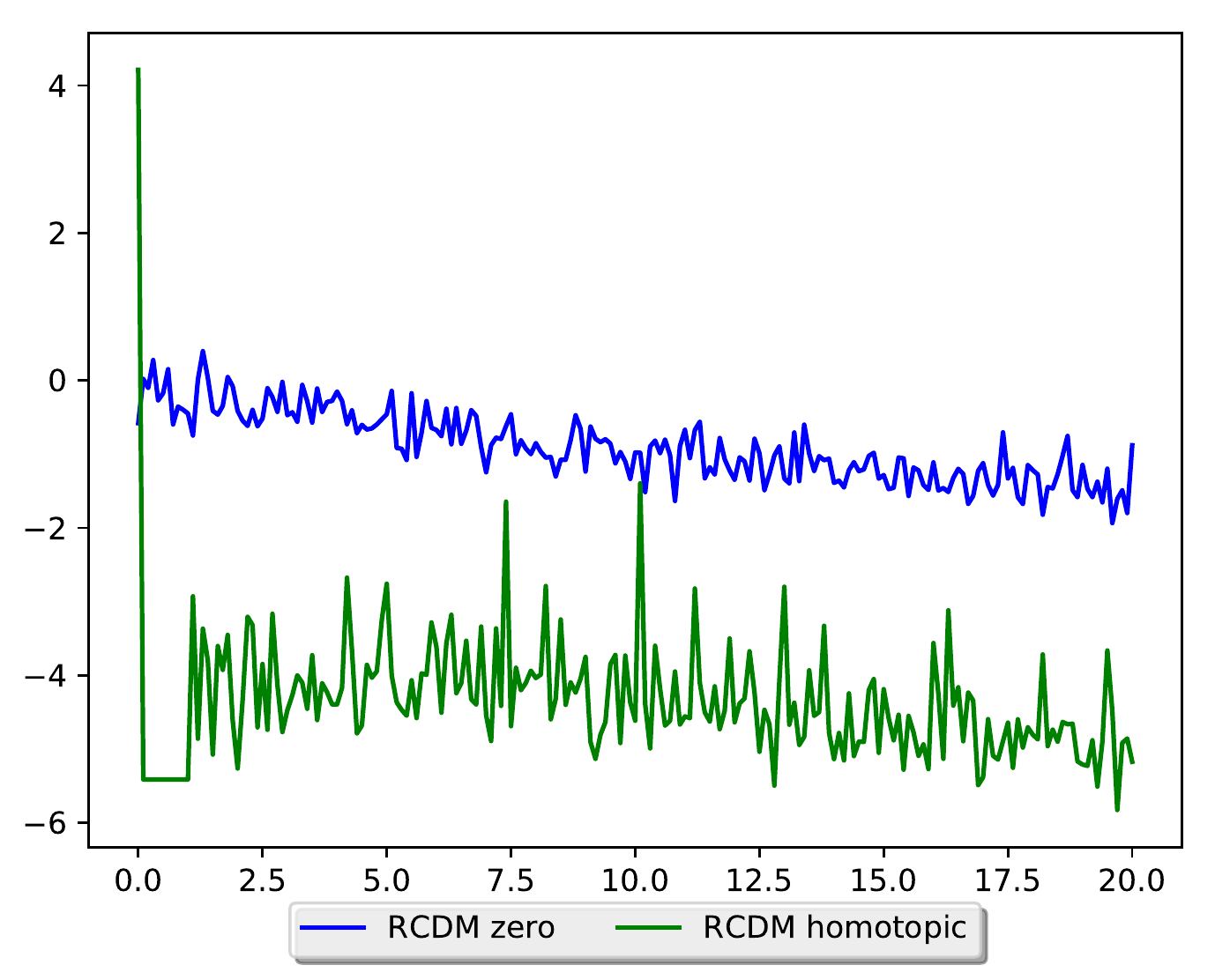}
   \\
   Test Error \ & \includegraphics[width=0.9\linewidth]{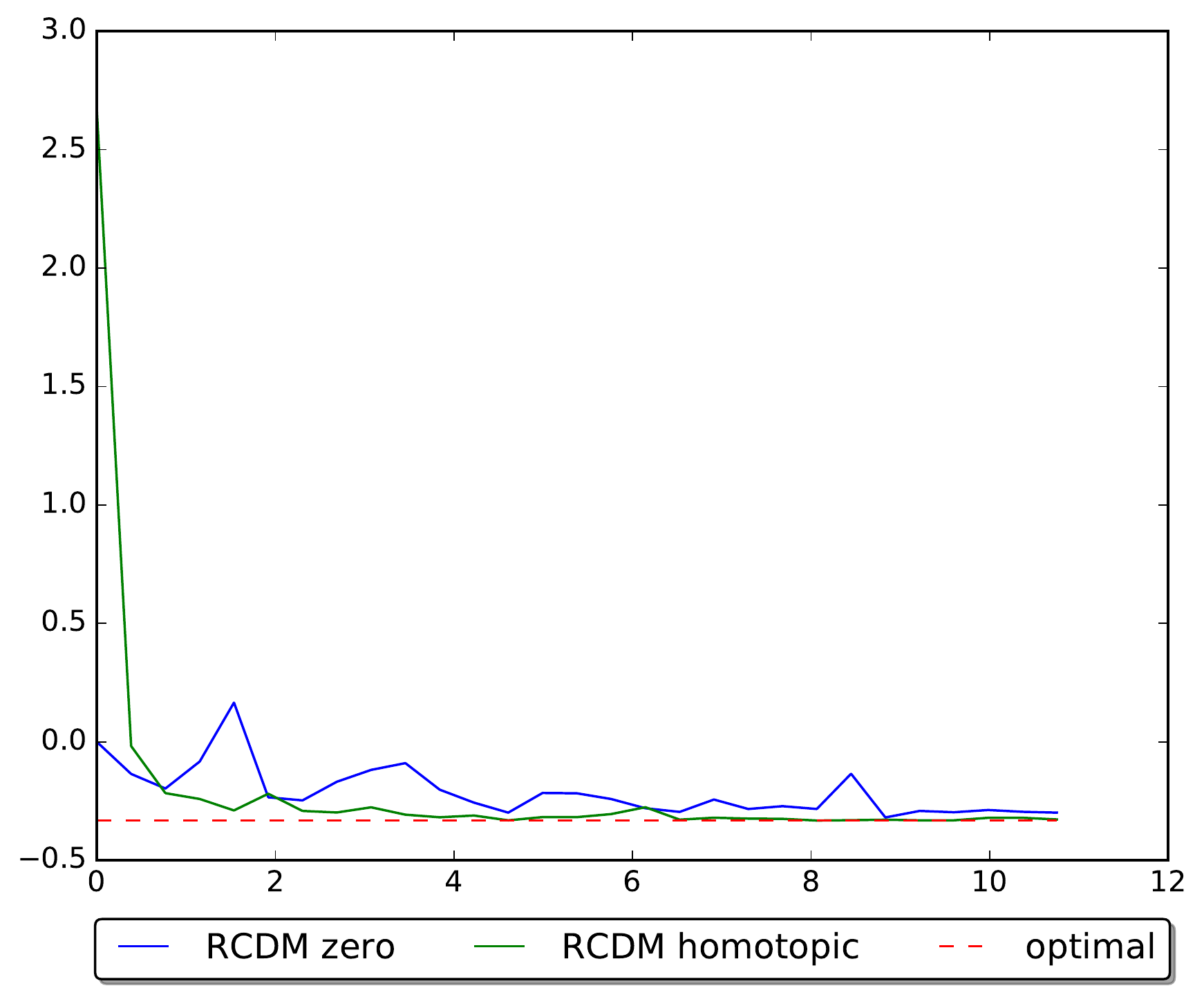}& 
 \includegraphics[width=0.9\linewidth]{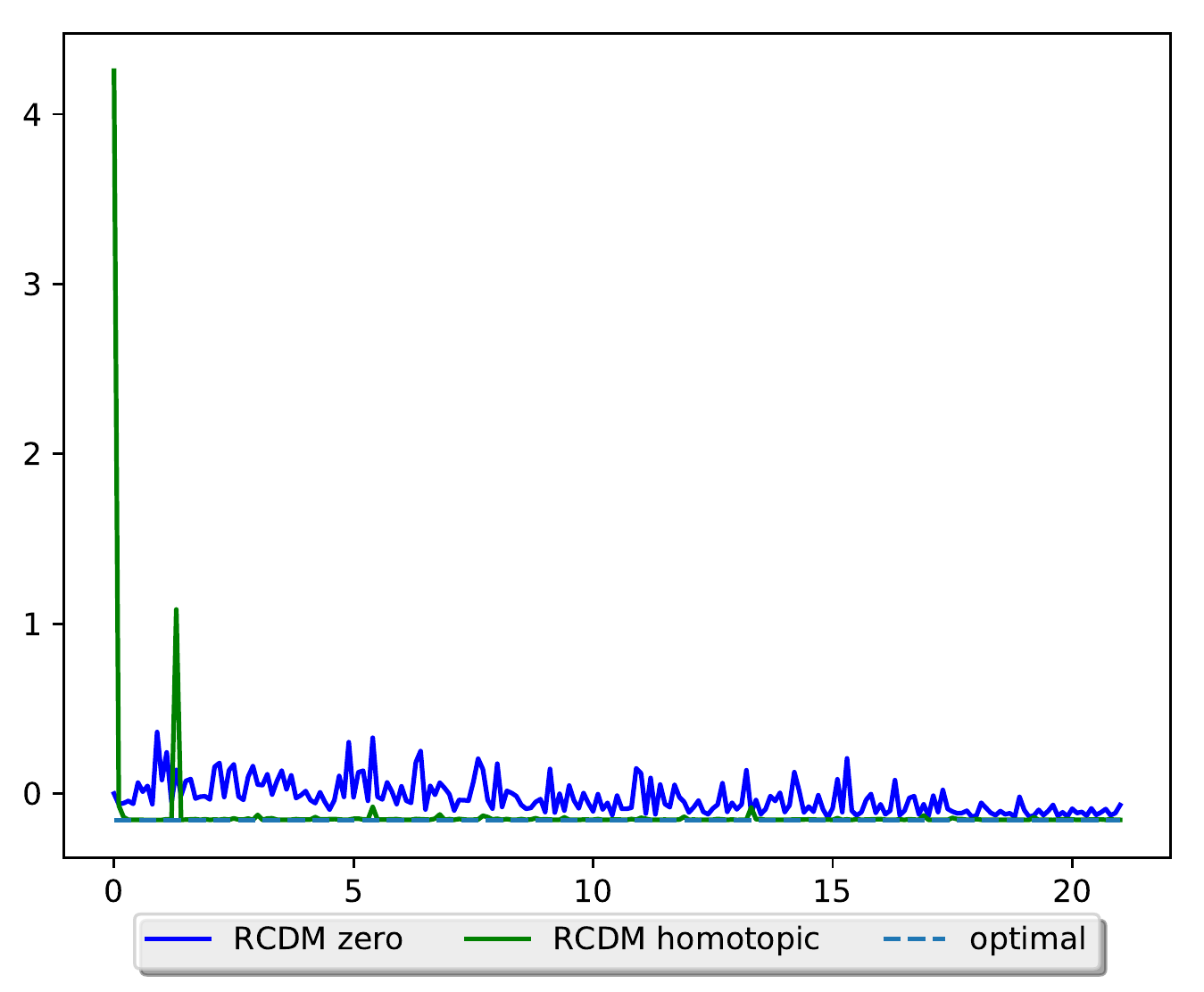} & 
  \includegraphics[width=0.9\linewidth]{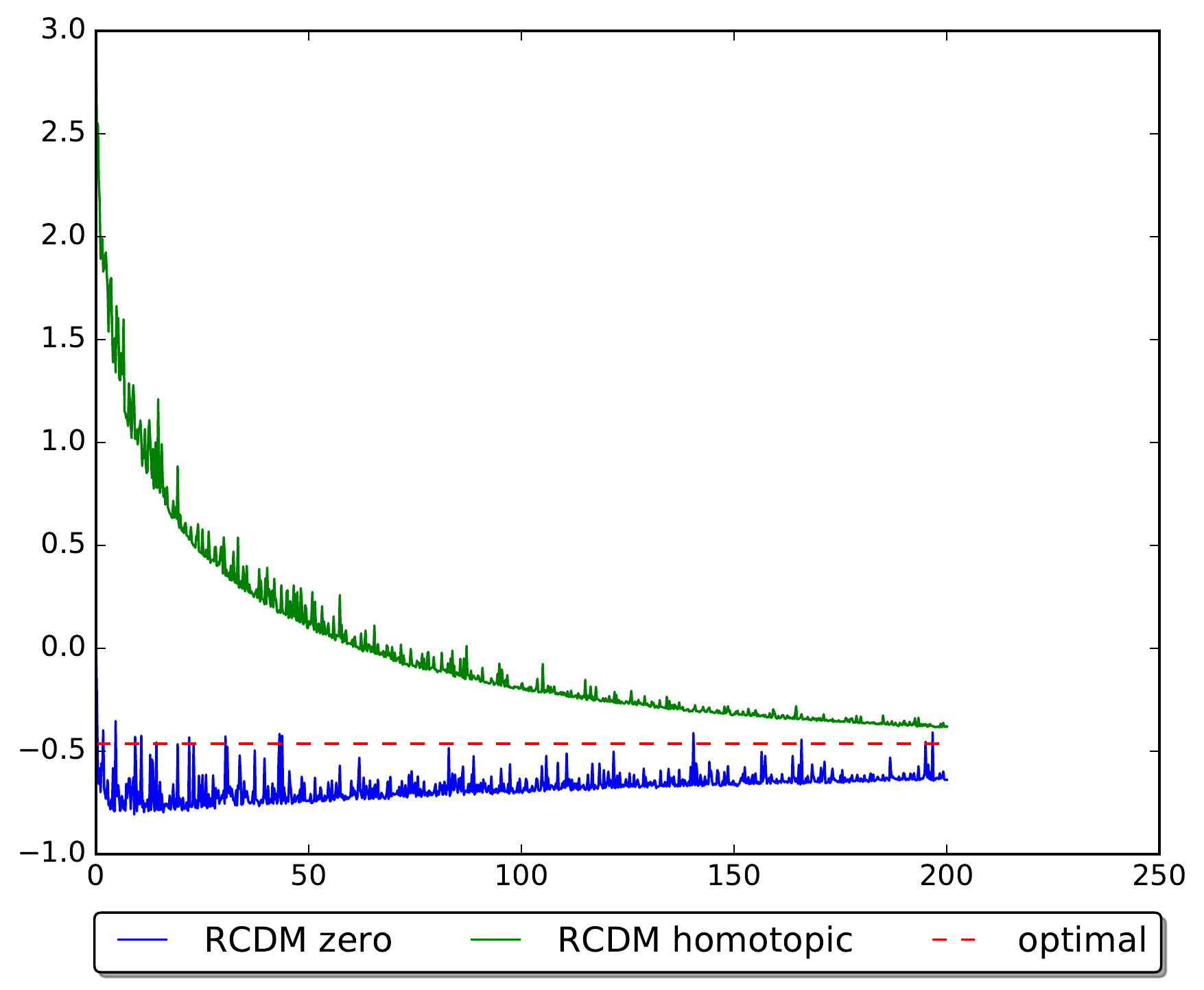}&
   \includegraphics[width=0.9\linewidth]{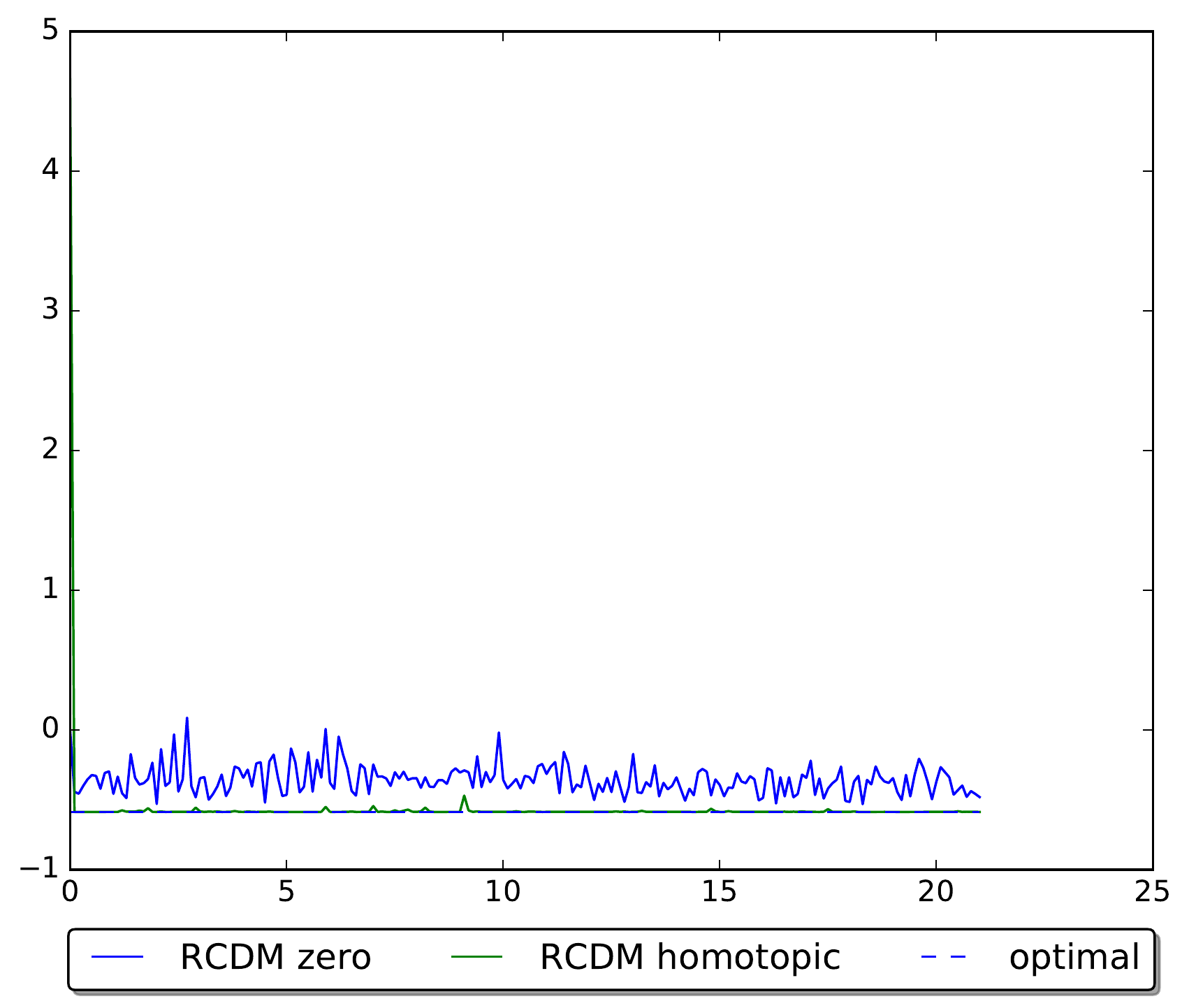}&
   \includegraphics[width=0.9\linewidth]{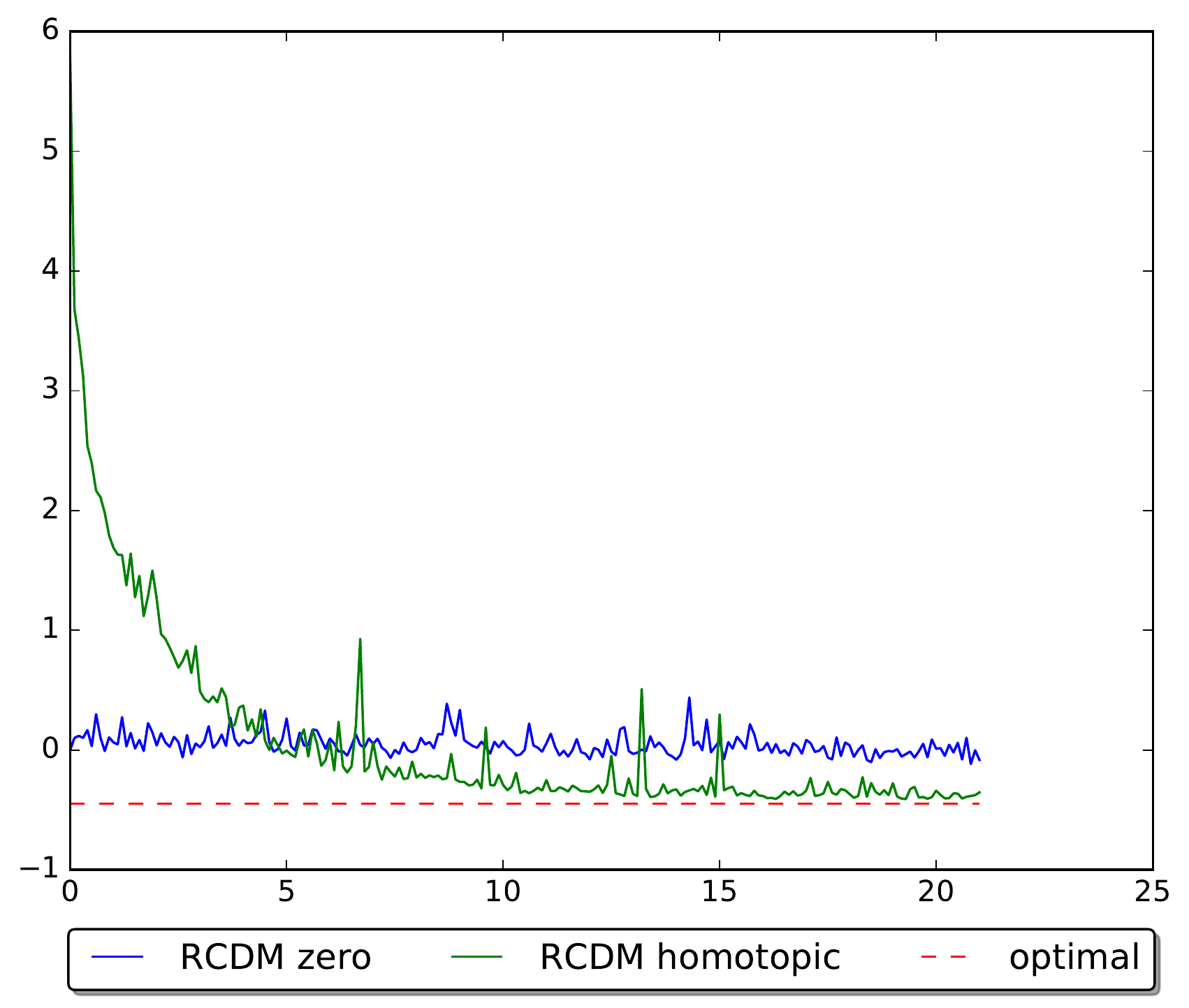}& 
   \includegraphics[width=0.9\linewidth]{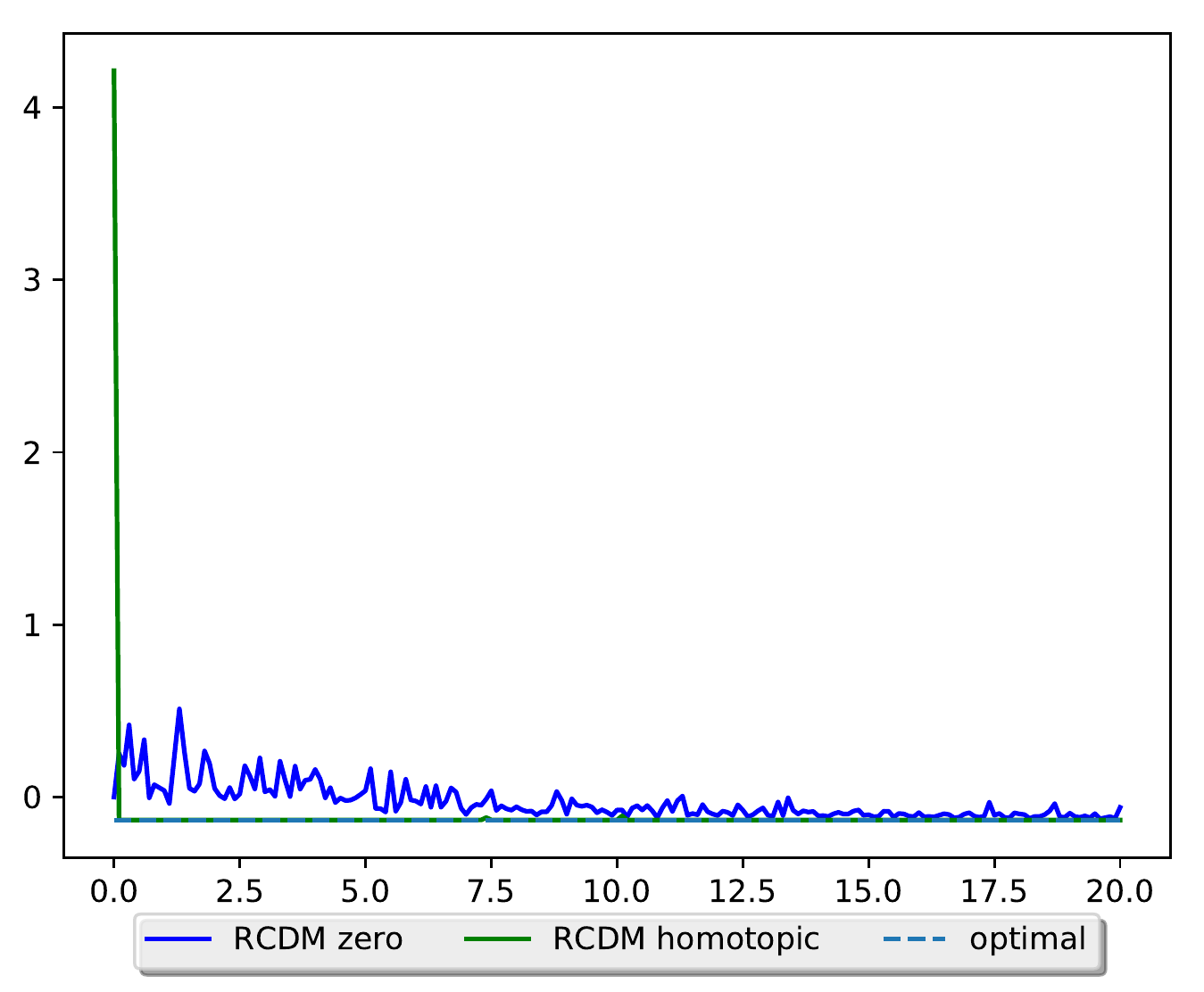}
   \\
\end{tabular}
\end{sc}
\end{footnotesize}

\caption{\label{fig:homotopic_rcdm_ridge}Homotopic initialization for the dual ridge regression.
The vertical axis represents the logarithm of dual
suboptimality ($\log_{10}( \quadratic(\dparam^{(t)})- \quadratic(\dparam^{*})$),
primal suboptimality ($\log_{10}(\quadratic(\param^{(t)})-
\quadratic(\param^*)$), and logarithm of the average of squared loss on the
test set. The horizontal axis shows the number of epochs. The horizontal dashed lined in the test error plot shows the test
error of the minimizer of ridge regression on the training set. The
training set including 80\% of the data.}

\end{figure}

 \begin{figure}
 \label{fig:deep_convergence}
 \begin{sc}
 
 \begin{tabular}{@{\hspace{0.1cm}}c@{\hspace{0.1cm}}c@{\hspace{0.1cm}}c@{\hspace{0.1cm}}c}
      \includegraphics[width=0.25\linewidth]{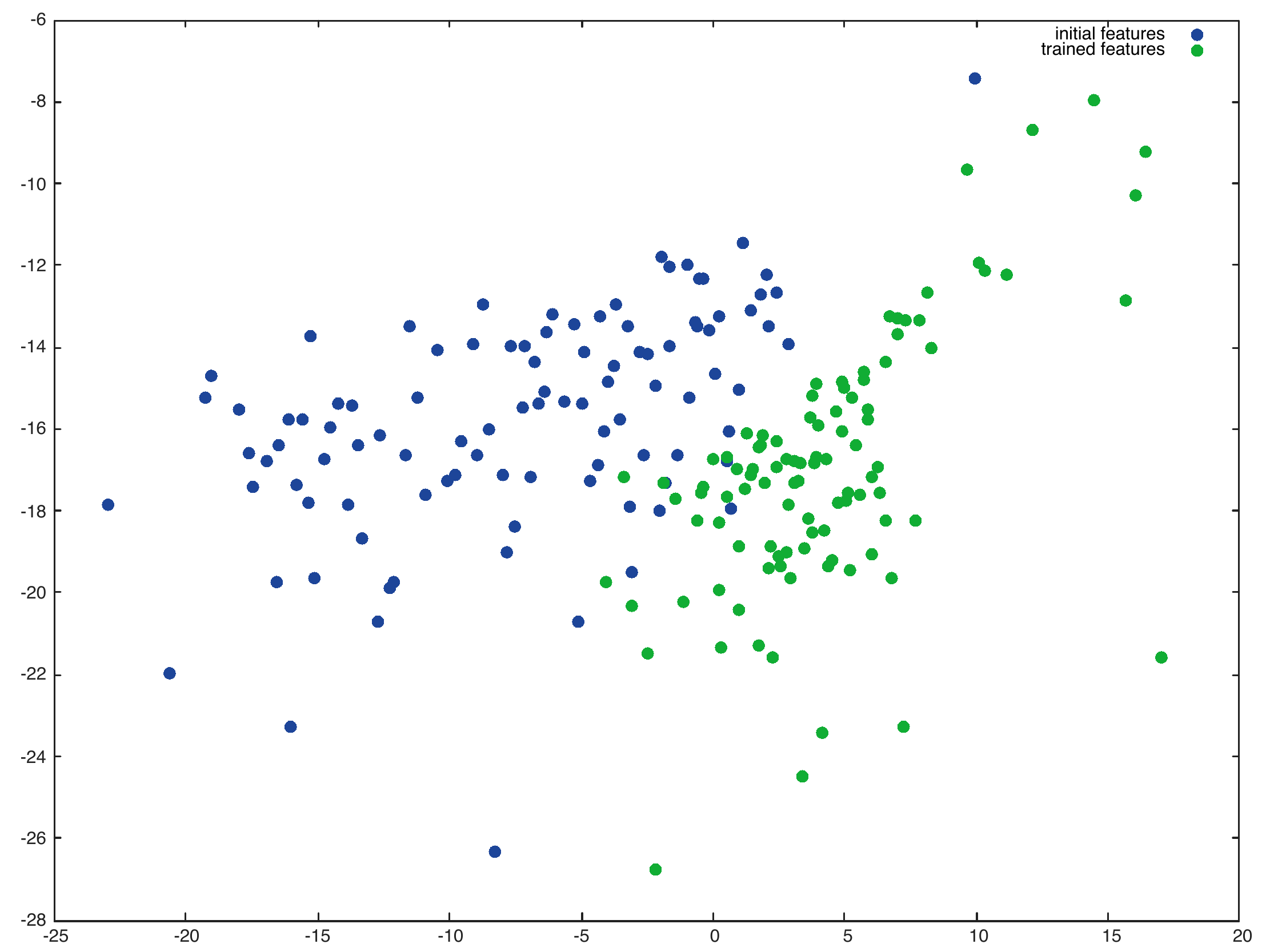} &
   \includegraphics[width=0.25\linewidth]{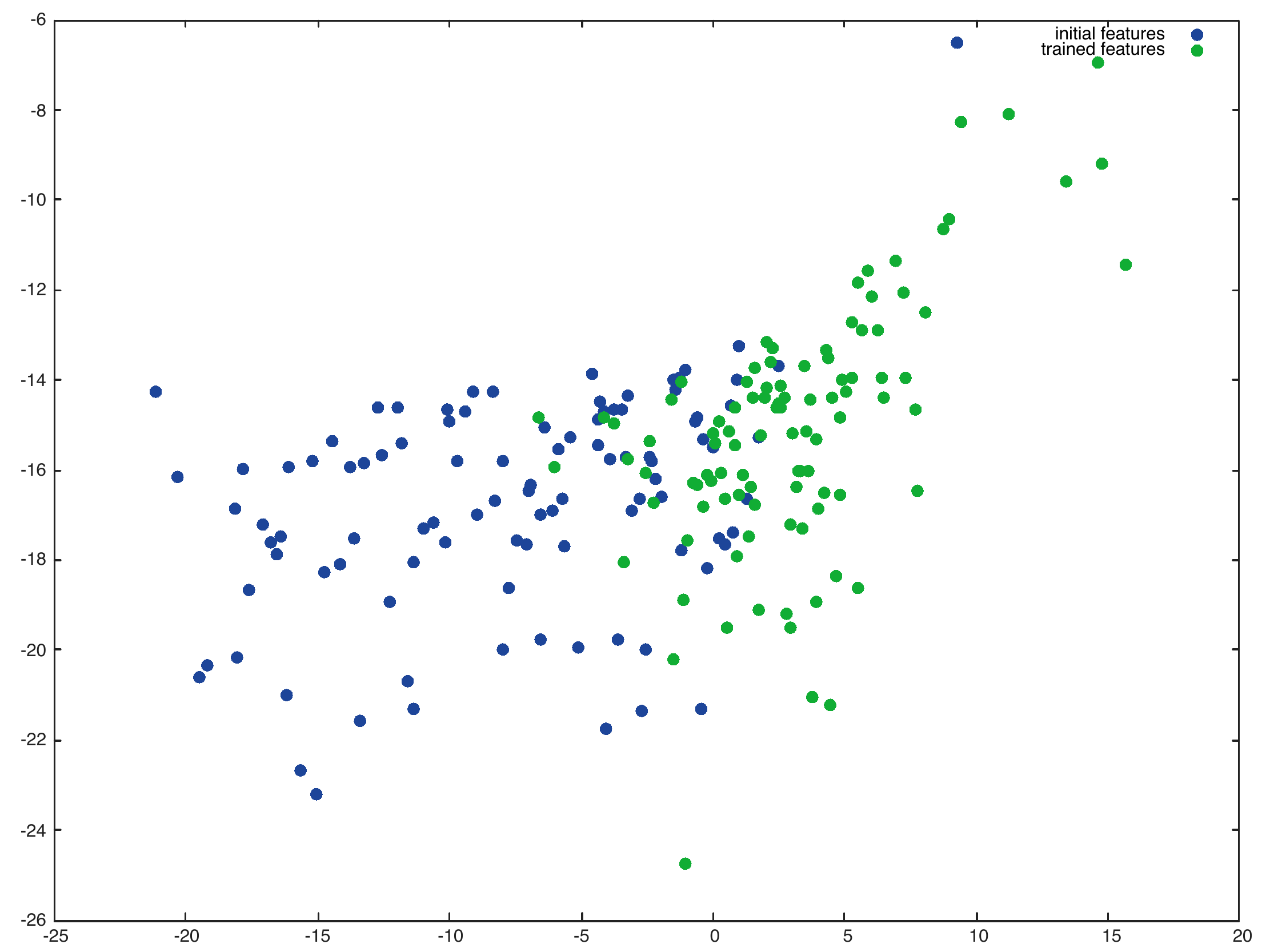} &
    \includegraphics[width=0.25\linewidth]{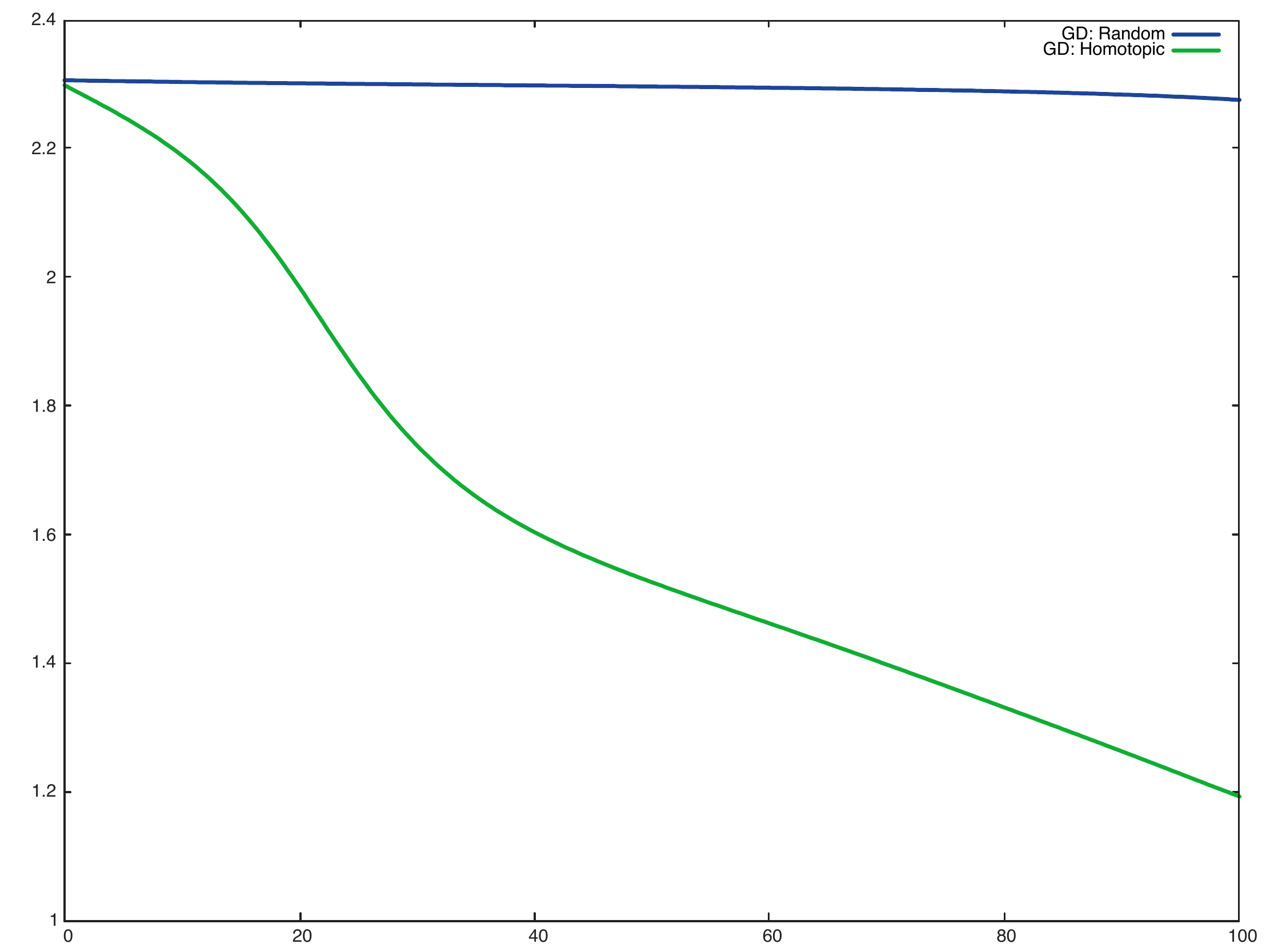}&
    \includegraphics[width=0.25\linewidth]{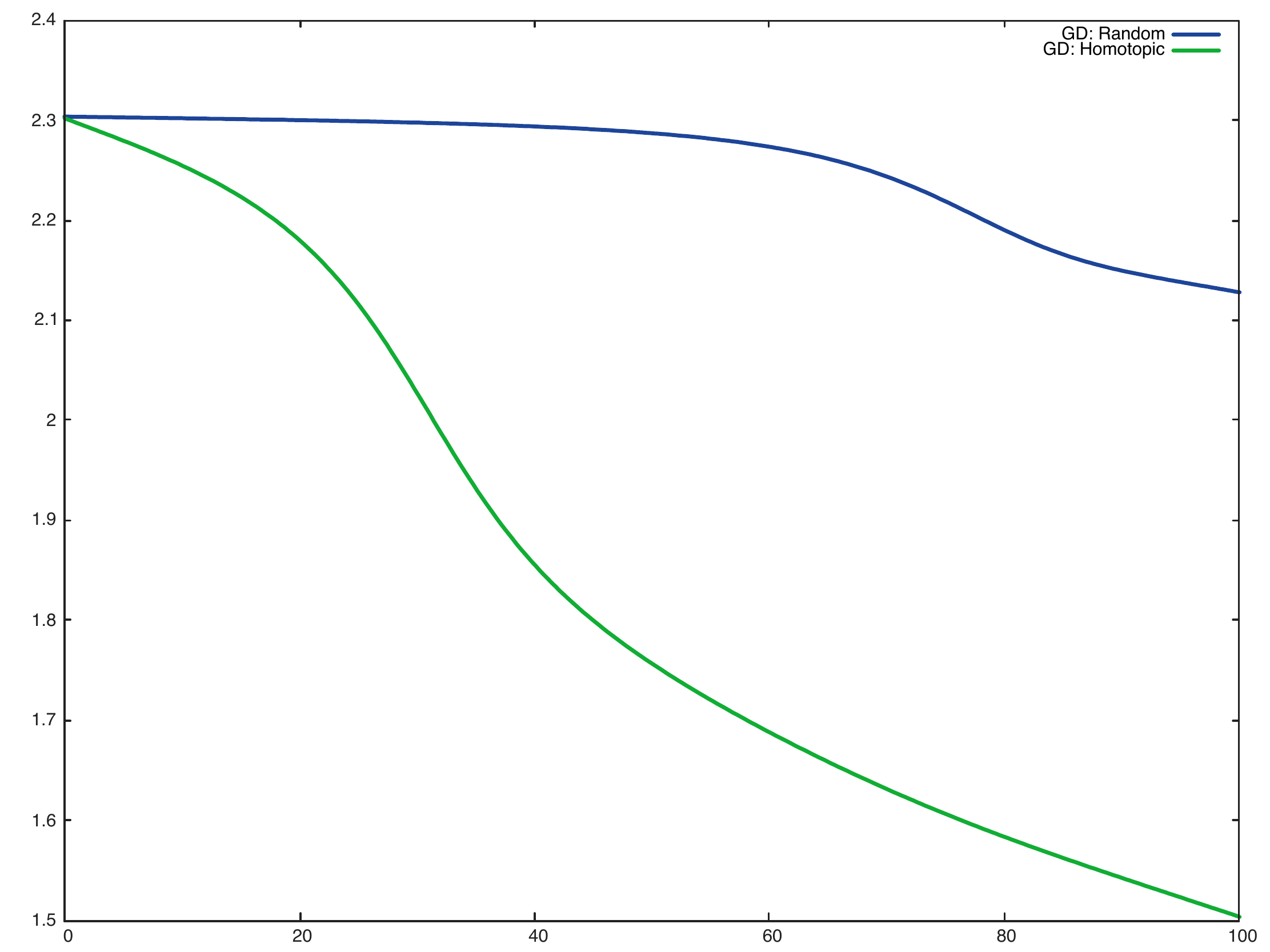}  
  \\ 
 a. mnist & b. cifar & c. mnist & d. cifar
\end{tabular}
 \end{sc} 
\caption{Initialization for neural networks. Figures a. and b. illustrate the dependency of
labels to eigenfeatures of non-linear features obtained by the last hidden layer
of MLP. The vertical axis represents $\log_{10}\left(\E[Y Z_j]^2/(\E[Z_j^2]^2)\right)$ and the
horizontal axis represents $\log_{10}(\E[Z_j^2]^2)$. Figures c. and b. show convergence of GD under two
different initialization schemes: random
initialization (blue graph), layer-wise initialization (green graph).
The vertical axis represents the training error. The horizontal axis is the
number of gradient steps. 
}
\label{fig:neural_net}
\end{figure}

\newpage
\begin{footnotesize}
\bibliography{refs}
\bibliographystyle{plain}
\end{footnotesize}

\newpage
\section{Appendix}  
Here, we present the proof of lemmas and theorems of the manuscript with our experiments on the SVM problem. We start with the analysis of random coordinate descent method, then we provide the proof of the claim on generalized linear model.  Finally, we represent the experimental result.
\subsection{Analysis of RCDM}
We proved that homotopic initialization accelerates convergence of gradient descent, up to a suboptimal solution, on the dual objective. In this section, we extend the result to the random coordinate descent method. For gradient descent, iterates could be tracked in a closed form, and the role of the initialisation was closely reflected in the convergence. The analysis of random coordinate descent, however, is more involved due to the noisy estimation of the gradient.  
In the next lemma, we decompose the convergence bound of RCDM to two terms: a term that depends on the initial sub-optimality on the dual objective, the second term is determined by the initial Euclidean distance to the dual minimizer.  Later, we will use this result to prove that homotopic initialization provides an acceleration up to a suboptimal. 
\begin{theorem*} [\ref{the:RCDM_convergence}]
 Let $\G_{i,j}$ denote the $(i,j)$-element of the Hessian matrix $\G$. Let the
 parameter vector $\dparam^{(t)}$ is obtained by $t$ RCDM-steps on
 $\quadratic$ starting from $\dparam_0$ with coordinate-wise step sizes
	\begin{align} \label{eq:step-size-lossbound-app}
			\gamma_r = \left(\G_{r,r} +
		\sum_{j} |\G_{r,j}|\right)^{-1}, \; \gamma_{\min} = \min_r \gamma_r, \; \gamma_{\max} = \max_r \gamma_r
	\end{align}
   For the obtained parameter, either the norm of gradient is bounded as
	\begin{align}
		\E \|\quadratic'_\mu(\dparam^{(t)})\|^2 \leq 2 \rho^2
		\left(\frac{\gamma_{\max}}{\gamma_{\min}}\right) \| \dparam_0 -
		\dparam^*\|^2
	\end{align} 
	or suboptimality is bounded as  
	\begin{equation}
		\E\left[\quadratic_\mu(\dparam^{(t)}) - \quadratic^*_\mu\right] \leq
		\frac{1}{2}\left(1 - \frac{\gamma_{\min}\rho
		}{n}\right)^t\left(\quadratic_\mu(\dparam_0) - \quadratic^*_\mu\right) +
		\frac{1}{2} \rho \left(\frac{\gamma_{\max}}{\gamma_{\min}}\right)\| \dparam_0 - \dparam^*\|^2
	\end{equation}
	for every $\frac{1}{n} \leq \rho \leq \| \G \| $ (expectation is
	over the random choice of coordinates).
\end{theorem*}
\begin{proof} 
We rewrite the RCDM update in a form that facilitates our future analysis. Consider matrix $\rmat_r$ whose $r$-th diagonal element is one and remaining elements are zero; using this matrix,  RCDM step can be written as 
\begin{align}
		\dparam^{+}  & = \dparam - \gamma_{r}
 \rmat_r \quadratic'(\dparam) \nonumber\\ 
  		& = \dparam - \gamma_{r}
 \rmat_r \left( \quadratic'(\dparam) - \quadratic'(\dparam^*) \right) \nonumber\\ 
        & = (\I - \gamma_r \rmat_r \G) (\dparam -\dparam^*) + \dparam^*.
        \label{eq:rcdm_udpate_matrix}
\end{align}
Note that we used the optimality condition of $\dparam^*$ in the second step. We skipped the subscript $\mu$ for $\quadratic_\mu$ throughout the proof.  To obtain the desired convergence guarantee, we decompose the Hessian matrix, of the dual objective, as $\G = \L_\rho +
	\S_\rho$ where matrices $\L_\rho$ and $\S_\rho$ are obtained from SVD
	decomposition of the data matrix: 
	\begin{align}
	\L_\rho  & = (n \mu)^{-1} \sum_{i:\lambda_i  \geq \rho } \lambda_i  \uv_i \uv_i^\top
	\\
	\S_\rho & = (n \mu)^{-1}\sum_{j:\lambda_j< \rho}  \lambda_j 
	\uv_j \uv_j^\top.
\end{align}
The second smallest eigenvalue of the $\L_\rho$ is $\rho$,
the spectral norm of $\S_\rho$ is bounded by $\rho$, and $\L_\rho \S_\rho = 0
$. Using the above decomposition, the expected suboptimality after one RCD-step can be written as
\begin{align}
	\E_r \left[ \quadratic(\dparam^{+}) - \quadratic^*\right]
	& \stackrel{\eqref{eq:subopt_quad}}{=} \frac{1}{2} \E \left[ \|
	\dparam^{+} - \dparam^* \|_\G \right] \nonumber\\ 
	& = \frac{1}{2}\E \left[ \|\dparam^{+}
	- \dparam^* \|_{\L_\rho} \right]
	+ \frac{1}{2}\E \left[ \| \dparam^{+}
	- \dparam^* \|_{\S_\rho} \right] \label{eq:convergence_decomposition_rcdm}
\end{align}
We prove that the first term, which depends on directions with large eigenvalues, decays in a favourable rate
(in lemma~\ref{lem:rcdm-fast-convergence}) and the second term is bounded (by a
factor of the threshold $\rho$). 
\begin{lemma} \label{lem:rcdm-fast-convergence}
		RCMD with coordinate-wise step sizes 
		\begin{align} \label{eq:stepsize_bound}
			\gamma_r \leq  (2 \G_{r,r})^{-1}
		\end{align}
		guarantees the decrement 
		\begin{align} \label{eq:L_convergence}
			\E \left[ \| \dparam^{+}
	- \dparam^*\|_{\L_\rho} \right]  \leq \left(1-\frac{\gamma_{\min}\rho}{n}
	\right) \E \left[ \| \dparam
	- \dparam^*\|_{\L_\rho} \right]
		\end{align}
		as long as 
		\begin{align} \label{eq:gradient_lower_bound_rcdm}
			 \| \quadratic'(\dparam) \|^2 \geq 2  \| \S_\rho\left(\dparam
	- \dparam^* \right)\|^2.
		\end{align}
	\end{lemma}
	\begin{proof} 
	 Consider the compact notations $\Delta =\dparam - \dparam^*$ and $\Delta_+ =
	 \dparam^{+} - \dparam^*$. We rewrite the expected suboptimality after one setp of
	 RCDM as 
		\begin{align} 
			\E \| \Delta_{+}\|_{\L_\rho} &
			\stackrel{\eqref{eq:rcdm_udpate_matrix}}{=} \E \|  (\I
			- \gamma_{r} \rmat_r \G) \Delta \|^2_{\L_\rho} \nonumber\\
			& =  \E \left[ \Delta^\top \left( \I - \gamma_r \G \rmat_r \right) \L_\rho
			\left( \I - \gamma_r \rmat_r \G \right) \Delta \right] \nonumber \\
			& = T_1 + T_2
		\end{align}
		where terms $T_1$ and $T_2$ are formulated as 
		\begin{align}
			T_1 & := \E \left[ \Delta^\top \L_\rho \left(\I - \gamma_r \rmat_r \G
			\right) \Delta \right]
			\\
			T_2 & := \E \left[ \Delta^\top \left( \gamma_r^2 \G \rmat_r \L_\rho
			\rmat_r \G  - \gamma_r \L_\rho \rmat_r \G \right) \Delta \right] 
		\end{align}
		The first term is bounded as:  
		\begin{align}
			T_1 & = \Delta^\top \L_\rho \left( \I - \E_r \left[ \gamma_r \rmat_r \right]
		\G \right) \Delta \nonumber\\
		& \stackrel{[1]}{\leq} \Delta^\top \L_\rho \left( \I - \frac{\gamma_{\min}}{n}
		\G\right)\Delta
		\nonumber\\
		& \stackrel{[2]}{\leq} \Delta^\top \left(\L_\rho - \frac{\gamma_{\min}}{n}
		\left( \L_\rho^2 + \L_\rho \S_\rho \right)\right) \Delta 
		\nonumber\\ 
		& \leq \Delta^\top \L_\rho \left( \I - (\gamma_{\min}/n)
		\L_\rho\right)
		\Delta
		\nonumber\\
		& \leq \Delta^\top \L_\rho^{\sfrac{1}{2}} \left( \I - (\gamma_{\min}/n) 
		\L_\rho\right) \L_\rho^{\sfrac{1}{2}}
		\Delta
		\nonumber\\
		& \stackrel{[3]}{\leq} \left( 1- (\gamma_{\min}\rho/n)\right)
		\Delta^\top \L_\rho
		\Delta.
		\end{align}
		In steps \footnotesize{[1]--[3]}, we used the fact that $\E\left[
		\rmat_r\right] = \I/n$, $\G = \L_\rho + \S_\rho$, and the second smallest eigenvalue of the matrix $\L_\rho$ is
		greater than $\rho$, respectively. It remains to show that $T_2$ is
		negative:
		\begin{align}
			T_2 & =  \E \left[ \Delta^\top \left( \gamma_r^2 \G \rmat_r \L_\rho
			\rmat_r \G  - \gamma_r \L_\rho \rmat_r \G \right) \Delta \right] \nonumber\\ 
			    & \stackrel{[1]}{=} \E \left[ \Delta^\top \left( \L_{\rho,r}
			    \gamma_r^2 \G \rmat_r \G  - \gamma_r \L_\rho \rmat_r \G \right) \Delta
			    \right] \nonumber\\
			    & \stackrel{[2]}{=}  \E \left[ \gamma_r \Delta^\top \left( \G_{r}
			    \G \rmat_r \G  - \L_\rho \rmat_r \G \right) \Delta \right] \nonumber\\ 
			    & \stackrel{[3]}{\leq} \gamma_{\min} \E \left[ \Delta^\top \left(
			    \frac{1}{2}\G \rmat_r \G  - \L_\rho \rmat_r \G \right) \Delta \right]
			    \nonumber\\
			    & \leq \gamma_{\min}  \Delta^\top \left( \frac{1	}{2} \G \E
			    \left[\rmat_r \right] \G  - \L_\rho \E\left[\rmat_r\right] \G \right)
			    \Delta \nonumber \\ 
			    & \stackrel{[4]}{\leq} \frac{\gamma_{\min}}{n}  \Delta^\top \left(
			    \frac{1 }{2} \G^2 - \L_\rho \G \right)
			    \Delta 
			    \nonumber \\ 
			    & \stackrel{[5]}{\leq} \frac{\gamma_{\min}}{2}  \Delta^\top \left( 
			    \S_\rho^2 - \L_\rho^2 \right)
			    \Delta \nonumber \\ 
			    & \stackrel{[6]}{\leq} 0
		\end{align}
		 where we used following facts in each step: 
		 \begin{itemize}
		   \item[\footnotesize{[1]:}] Note that  $\L_{\rho,r}$ denotes the diagonal element $r$ of matrix $\L_\rho$. Matrix $\rmat_r$ has only one non-zero
		   element: the $r$-th diagonal which is one, hence $\rmat_r \L_{\rho} \rmat_r = \L_{\rho,r} \rmat_r $. 
		   \item[\footnotesize{[2]:}] Recall $ (\L = \G - \S) \rightarrow (\L_{\rho,r} = \G_r - \S_{\rho,r})$.
		   Since $\S_{\rho}$ is positive definite, its diagonal elements are positive.
		   This concludes that $\L_{\rho,r} \leq \G_r $. 
		   \item[\footnotesize{[3]:}] The choice of step sizes leads to $\gamma_r \G_r \leq 1/2$.
		   \item[\footnotesize{[4]:}] We know that $\E_r \left[\rmat_r\right] = \I/n$
		  according to the definition of matrix $\rmat_r$.  
		  \item[\footnotesize{[5]:}] Matrix $\L_\rho$ and $\S_\rho$ are orthogonal.
		  \item[\footnotesize{[6]:}] As long as the norm of the gradient is sufficiently large (the condition in Eq.~\eqref{eq:gradient_lower_bound_rcdm}), this inequality holds. To prove this, we
		 write the norm of the gradient as: 
		 \begin{align}
		 	\| \quadratic'(\dparam) \|^2 &  = \| \quadratic'(\dparam) -
		 	\quadratic'(\dparam^*)\|^2 \nonumber \\ 
		 	& = \| \G \Delta \|^2 \nonumber \\ 
		 	& =  \| \L_\rho \Delta \|^2 + \| \S_\rho \Delta \|^2
		 \end{align}
		 Now the condition of Eq.~\eqref{eq:gradient_lower_bound_rcdm} implies 
		 \begin{align}
		  & \| \quadratic'(\dparam) \|^2 \geq 2 \| \S_\rho \Delta \|^ 2  \nonumber \\ 
		 \leftrightharpoons & \| \L_\rho \Delta \|^2 + \| \S_\rho \Delta \|^2 \geq 2
		 \| \S_\rho \Delta \|^ 2 \nonumber \\ 
		 \leftrightharpoons & \| \L_\rho \Delta \|^2 - \| \S_\rho \Delta \|^2 \geq 0
		 \end{align}
		 \end{itemize}  
	\end{proof}
	The last lemma implies that the convergence of $\E\left[ \| \dparam^+ -
	\dparam^*\|_{\L_\rho}\right]$ is dominated by $\rho$ up to a suboptimal solution in term of the norm of the gradient.  
	It remains to prove that the second term of suboptimality, i.e. $\E \left[  \| \dparam^+ -
	\dparam^*\|_{\S_\rho}\right] $, scales with a factor of $\rho$: 
	\begin{align}
		\E \left[  \| \dparam^+ -
	\dparam^*\|^2_{\S_\rho}\right] & \leq \| \S_\rho\| \E \left[ \| \dparam^+ - \dparam^* \|^2
	\right] \nonumber \\ 
	& \leq \rho \E \left[ \| \dparam^+ - \dparam^* \|^2
	\right]. \label{eq:upperbound_on_slow_convergence_rcmd}
	\end{align}
    Let parameter vector $\dparam^{(t)}$ denote the obtained parameter vector
    after $t$ RCDM steps. We know that RCMD with step size of $\gamma_r \leq
    (\G_{r,r})^{-1}$ monotonically convergences to the minimizer in terms of the
    objective value. Furthermore, $(1/n)$-strong convexity of the dual
    function implies that $\| \dparam^{(t)}-\dparam^*\|^2 \leq
    n(\quadratic^{(t)} - \quadratic^*)$. Therefore, $\| \dparam^{(t)} - \dparam^*
   \| $ asymptotically convergences to zero; however, this convergence is not
   monotone. In other words, the distance might diverge. In the
   next lemma, we bound this distance by a factor of the initial distance, i.e. $\| \dparam_0 - \dparam^* \|^2$.
	\begin{lemma} \label{lem:distance_bound_rcdm}
		The distance of RCD iterates $\dparam^{(t)}$ from the minimizer are bounded by a constant factor of the initial distance from the minimizer:
    \begin{align}
    \E \| \dparam^{(t)} - \dparam^*\|^2 \leq \frac{\gamma_{\min}}{\gamma_{\max}}
    \| \dparam^{(0)} - \dparam^* \|^2
    \end{align}
    for the choice of step sizes in
    Eq.~\eqref{eq:step-size-lossbound-app}. 
	\end{lemma}
	\begin{proof}
		We introduce the diagonal matrix $\Lambda$ whose $r$ diagonal element is the
		coordinate-wise step size $r$. The definition of $\Lambda$ implies $\gamma_{\min}\I \preceq \Lambda \preceq
		\gamma_{\max}\I$. Furthermore, the choice of step sizes of
		Eq.~\eqref{eq:step-size-lossbound-app} ensures matrix $\Lambda-\G$ is diagonal
		dominant, therefore it is positive definite. Using these facts, we prove that
		$\E \|\dparam^{(t)}-\dparam^*\|^2_{\Lambda}$ does not diverge.
		Some straightforward algebra leads to 
		\begin{align}
			\E \|\dparam^{(t)}-\dparam^*\|^2_{\Lambda}
			\leq \E \|\dparam^{(t-1)}-\dparam^*\|^2_{\Lambda} + (\dparam^{(t-1)}-\dparam^*)^\top T_1
			(\dparam^{(t-1)}-\dparam^*) \nonumber
		\end{align} 
		where $T_1$ is a negative definite matrix: 
		\begin{align}
			T_1 & = \E\left[\gamma_r^2 \G \rmat_r \Lambda \rmat_r \G - \gamma_r \G
			\rmat_r \Lambda\right] \nonumber \\ 
			& \stackrel{[1]}{\leq} \E\left[\gamma_r \G \rmat_r \G - \gamma_r \G
			\rmat_r \Lambda\right] \nonumber\\ 
			& \leq \gamma_{\min} \left(\G \E \left[ \rmat_r\right]\G - \G \E \left[
			\rmat_r\right] \Lambda \right) \nonumber\\ 
			& \stackrel{[2]}{\leq} \frac{\gamma_{\min}}{n}\G \left( \G -
			\Lambda\right)\nonumber \\
			& \stackrel{[3]}{\leq} 0, 
		\end{align}
		where steps [1]--[3] are obtained by the choice of step size, $\E
		\left[ \rmat_r\right] = \I/n$, and $0 \preceq \Lambda - \G$, respectively.
		Using this result, we bound the desired distance as 
		\begin{align}
			\E\| \dparam^{(t)}- \dparam^*\|^2 & = \E \| \Lambda^{\sfrac{-1}{2}}
			\Lambda^{\sfrac{1}{2}}(\dparam^{(t-1)}-\dparam^*)\|^2 \nonumber \\ 
			& \leq \| \Lambda^{-1}\| \E
			\|\Lambda^{\sfrac{1}{2}}(\dparam^{(t-1)}-\dparam^*)\|^2 \nonumber \\ 
			& \stackrel{[1]}{\leq} \| \Lambda^{-1}\|
			\|\Lambda^{\sfrac{1}{2}}(\dparam_0-\dparam^*)\|^2 \nonumber \\ 
			& \leq  \| \Lambda^{-1}\| \| \Lambda \| \| \dparam_0 - \dparam^*
			\|^2 \nonumber\\ 
			& \stackrel{[2]}{\leq} \left( \frac{\gamma_{\max}}{\gamma_{\min}}\right) \|
			\dparam_0 - \dparam^* \|^2. 
		\end{align}
		Step \footnotesize{[1]} is obtained by recursion on $t$. Step [2] is derived
		by the definition of the matrix $\Lambda$.
	\end{proof}
	Using last two lemmas, we prove the theorem: 
	\begin{align}
		\E\left[ \quadratic(\dparam^{(t)}) - \quadratic^*\right]
		& \stackrel{[1]}{=} \frac{1}{2}\E \left[ \|
		\dparam^{(t)} - \dparam^* \|^2_{\L_\rho} \right]
	+ \frac{1}{2}\E \left[ \| \dparam^{(t)}
	- \dparam^* \|^2_{\S_\rho} \right] \nonumber \\ 
	& \stackrel{[2]}{\leq} \frac{1}{2}  \left(1-\frac{\gamma_{\min}\rho}{n}
	\right)^t \left( \quadratic(\dparam_0) -\quadratic^*\right) 
	+ \frac{1}{2}\E \left[ \| \S_\rho^{\sfrac{1}{2}}\left(\dparam^{(t)}
	- \dparam^* \right)\|^2 \right] \nonumber\\ 
	& \stackrel{[3]}{\leq} \frac{1}{2}  \left(1-\frac{\gamma_{\min}\rho}{n}
	\right)^t \left( \quadratic(\dparam_0) -\quadratic^*\right) 
	+ \frac{1}{2} \left(\frac{\gamma_{\max}}{\gamma_{\min}}\right) \rho \| \dparam_0
	- \dparam^* \|^2 
	\end{align}
	where we used following facts: 
	\begin{itemize}
	  \item[\footnotesize{[1]:}] This step is obtained from spectral decomposition
	  of the hessian matrix (see Eq.~\eqref{eq:convergence_decomposition_rcdm}). 
	  \item[\footnotesize{[2]:}] Recursion on the result
	  of lemma~\ref{lem:rcdm-fast-convergence} yields the inequality.
	  \item[\footnotesize{[3]:}] This inequality is a direct result of
	  lemma~\ref{lem:distance_bound_rcdm}.
	\end{itemize}
\end{proof}
The suboptimality bound of the last theorem provides a fast convergence rate up to a suboptimal solution that is dominated by the initial distance $\| \dparam_0 - \dparam^*\|^2$. In the next lemma, we bound this distance using the homotopic parameter $\nu$.
\begin{lemma*}[\ref{lem:homotopic_initial_path_bound}]
		Assume that the data set is $\tau$-bounded.  Then the distance between two
		minimizers of the dual objective with different regularizers is bounded as
		\begin{align}
			\| \dparam^*_{\nu} - \dparam^*_{\mu} \|^2  \leq  \left((\mu-\nu)^2/\nu\right) | \{ j: \lambda_j <\rho\}|
			  n  \tau \rho.
		\end{align}
\end{lemma*}
\begin{proof} 
Lemma~\ref{lemma:dual_init} with some straight forward algebra conclude the result. 
\end{proof}
\subsection{Generalized Linear Model}
\paragraph{The gradient of GLM} Recall the objective of GLM: 
\begin{align}
 \risk(\glmp) = \E \left[ \varphi(\x^\top \glmp) - y
 \left( \x^\top \glmp \right)\right].
\end{align}

Suppose that the $\x$ are from a multivariate normal distribution, i.e. $\x \sim N(\vec{0},\Sigma)$. Then one can use integration by parts (Stein's lemma~\cite{erdogdu2016scaled}) to write the gradient of the above objective as 
\begin{align} \label{eq:gd_glm_stein}
 \risk'(\glmp) & = - \E_{\x,y} \left[ y \x\right] 
   			+ \E_{\x} \left[\x \varphi'(\x^\top \glmp)\right] \\ 
		 & \stackrel{\text{[Stein's lemma]}}{=} - \E_{\x,y} \left[ y \x\right] +
   			\E_{\x} \left[\varphi^{(2)}(\x^\top \glmp) \right]  \Sigma \glmp. 
\end{align}
where $\varphi^{(2)}$ denote the second derivate of $\varphi(.)$.  The Hessian matrix of $\risk$ highly depends on the covariance matrix of the distribution of input: 
\begin{align} \label{eq:hessian}
  \risk''(\glmp) & =  \E \left[\varphi^{(2)}(\x^\top \glmp)\right] \Sigma + 
   		 	\E \left[ \varphi^{(3)} (\x_i^\top \glmp) \x \right] \glmp^\top\Sigma \\ 
		& \stackrel{\text{[Stien's lemma]}}{=} \E \left[\varphi^{(2)}(\x^\top \glmp)\right] \Sigma + 
   		 	\E \left[ \varphi^{(4)} (\x^\top \glmp) \right] \Sigma \glmp \glmp^\top\Sigma.\end{align}  
Using the above the Hessian and the gradient, we provide a useful expression of the gradient that facilitates the our convergence analysis. 
\begin{lemma} \label{lem:gd_glm_alt}
	There exists constants $\xi_1$, and $\xi_2$, which depend on $\glmp$, and $\glmp^*$ such that 
	\begin{align} 
		\risk'(\glmp) = \xi_1 \Sigma ( \glmp - \glmp^*) + \xi_2 \E_{y\x} \left[ y \x \right] 
	\end{align}
	holds. 
\end{lemma}
\begin{proof} 
According to mean-value theorem, there is a $\bar{\glmp} = (1-\theta) \glmp+ \theta \glmp^*,\theta \in \left( 0, 1\right)$ such that
\begin{align} 
		 \risk'(\glmp) - \risk'(\glmp^*) 
		& ={\risk''(\bar{\glmp})} (\glmp- \glmp^*) \\ 
	& \stackrel{\eqref{eq:hessian}}{=} \E \left[ \varphi^{(2)}(\x^\top \bar{\glmp})\right]  \Sigma (\glmp - \glmp^*) + \E \left[ \varphi^{(4)} (\x^\top \bar{\glmp}) \right] \Sigma \bar{\glmp} \left( \bar{\glmp}^\top\Sigma (\glmp - \glmp^*) \right) \\ 
	& =  \E \left[ \varphi^{(2)}(\x^\top \bar{\glmp})\right]  \Sigma (\glmp - \glmp^*) + c_1 \Sigma \bar{\glmp} \\
	& = c_2  \Sigma (\glmp - \glmp^*) + c_1 \Sigma \glmp\\ 
	& \stackrel{\eqref{eq:gd_glm_stein}}{=} c_2 \Sigma (\glmp - \glmp^*) + c_3 \left( \risk'(\glmp) + \E_{\x,y} \left[ y \x\right]  \right), \label{eq:gd_glm_alt}
	\end{align}
	where constants are 
	\begin{align} 
		c_1 & = \E \left[ \varphi^{(4)} (\x^\top \bar{\glmp}) \right] \left( \bar{\glmp}^\top\Sigma (\glmp - \glmp^*) \right) \\
		c_2 &=  \E \left[ \varphi^{(2)}(\x^\top \bar{\glmp})\right] - \theta c_1 \\
		c_3 & = c_1\left(\E \left[ \varphi^{(2)}(\x^\top \bar{\glmp})\right]\right)^{-1}.
	\end{align}
	The Eq.~\ref{eq:gd_glm_alt} with the optimality condition $\risk'(\glmp^*) = 0 $ conclude the proof for $\xi_1 = c_2/(1-c_3)$, and $ \xi_2 = c_3/(1-c_3)$. 
\end{proof} 
\paragraph{A biased gradient step}
 We suggest our modified gradient step as 
\begin{align} 
 \glmp^{(t)} = \glmp^{(t-1)} - \gamma_t \risk'(\glmp^{(t-1)}) - \eta_t \E\left[ y \x\right]
\end{align}
where constant $\gamma_t$ and $\eta_t$ are two step sizes. In the next lemma, we prove that the convergence of such a modified gradient descent depends on the covariance matrix of the distribution. 
\begin{lemma} \label{lem:gd_glmp_iterates}
	Let iterate $\glmp^{(t)}$ is obtained by $t$ GD steps on GLM. There is a schedule for $\gamma_t$ and $\eta_t$ such that modified GD steps can be written as  
	\begin{align} 
		\glmp^{(t+1)} - \glmp^* & =  (\I - \Sigma/L) \left( \glmp^{(t)} - \glmp^* \right),
	\end{align}
	where $L = \| \Sigma\|$.
\end{lemma}
\begin{proof} 
We prove the above bound by induction on $t$.
	\begin{align} 
		\glmp^{(t+1)} - \glmp^* & = \glmp^{(t)} - \glmp^* -\gamma_t \risk'(\glmp^{(t)}) - \eta_t \E \left[ y \x \right] \\ 
		& \stackrel{\text{Lemma}~\ref{lem:gd_glm_alt}}{=}  \glmp^{(t)} - \glmp^* - \gamma_t \xi_1 \Sigma ( \glmp^{(t)} - \glmp^*) - \gamma_t \xi_2 \E \left[ y \x \right] - \eta_t \E \left[ y \x\right]
	\end{align}
The above equation with step size $\gamma_t = (L\xi_1)^{-1}$ and $\eta = -\gamma_t \xi_2$ completes the proof. 
	
holds. Replacing the above equation in the convergence bound with the step size $\gamma_t = ( L \E_{\x} \left[ \varphi^{(2)}(\x^\top \bar{\glmp}^{(t)})\right] )^{-1}$ concludes the proof. \end{proof}
Although the last lemma does not specify step sizes, one can find this schedule by line-search.  The result allows us to track iterates of GD in a closed form, which provides an accelerated convergence up to a suboptimal solution for GLM -- similar to ridge regression.  The next lemma proves this. 
 \begin{lemma*} [\ref{lemma:glm_zero_init}]
 Suppose that inputs are drawn i.i.d from a multivariate normal
	distribution with mean $0$ and covariance matrix $\Sigma$, i.e. $\x_i \sim
	N(\vec{0},\Sigma)$ where $\| \Sigma \| = L$.   Assume that $|\varphi^{(2)}(a)|\leq \phi$. If the dataset is $\tau$-bounded, then there is a step size schedule for modified gradient descent such that iterates of GD starting from zero obtains the following suboptimality bound for all $0<\zeta<L$:
	\begin{align}
		& \risk(\glmp^{(t)	}) - \risk^* \leq c_{\glmp^*}\tau \phi \left(  (1-\zeta/L)^{2t} \left(
		1 - r(\zeta)\right) + r(\zeta) \zeta \right), \\ 
		& r(\zeta) := | \{ j: \E[Z_j^2] > \zeta\}|,  \; c_{\glmp^*} := ( \E_{\x} \left[ \varphi^{(2)}(\x^\top \glmp^*)\right])^{-1}.
	\end{align}
\end{lemma*}
\begin{proof} 
	Using the mean-value theorem, there is $\bar{\glmp}^{(t)} = \theta \glmp^{(t)} + (1-\theta)\glmp^*, \theta \in \left[0, 1\right]$ such that  
	\begin{align}
		\risk(\glmp^{(t)}) - \risk(\glmp^*) & = \| \glmp^{(t)} - \glmp^* \|_{\risk''(\bar{\glmp}^{(t)})} \\ 
		& \stackrel{\eqref{eq:hessian}}{=} \E \left[ \varphi^{(2)}(\x^\top \bar{\glmp}^{(t)}) \right] \| \glmp^{(t)} - \glmp^* \|_{\Sigma} \\  
		& \leq \phi \| \glmp^{(t)} - \glmp^* \|_{\Sigma}  \\ 
		& \stackrel{\text{Lemma~\ref{lem:gd_glmp_iterates}}}{\leq} \phi \| (\I - \Sigma/L)^t \left( \glmp^{(0)} - \glmp^* \right)\|_{\Sigma}  \label{eq:convergence_glm}
	\end{align}
	It remains to relate the $\glmp^*$ to the boundedness assumption. To this end, we use optimality condition of $\glmp^*$ as 
	\begin{align}
		\risk'(\glmp^*) \stackrel{!}{=} 0 \leftrightharpoons
		 \glmp^* \stackrel{\eqref{eq:gd_glm_stein}}{=} c_{\glmp^*} \Sigma^{-1} \E[\x
		y], \; c_{\glmp^*} := \left( \E_{\x} \left[ \varphi^{(2)}(\x^\top \glmp^*)\right]\right)^{-1}
	\end{align}
The above result, which is also provided in \cite{erdogdu2016scaled}, shows that the minimizer of the GLM can be obtained by scaling the minimizer of ridge regression. The convergence bound of Eq.~\eqref{eq:convergence_glm} is also scaled convergence of the gradient descent on ridge regression. Plugging the minimizer into the Eq.~\eqref{eq:convergence_glm} with the boundedness assumption completes the proof. 
\end{proof}

\subsection{Experiments} 
\begin{figure}[h]
\begin{footnotesize}
\begin{sc}
\begin{tabular}{@{}S@{\hspace{0.075cm}}T@{\hspace{0.075cm}}T@{\hspace{0.075cm}}T}
\\ 
 & Dual Convergence & Primal Convergence & Test Error \\ 
a9a & \includegraphics[width=0.9\linewidth]{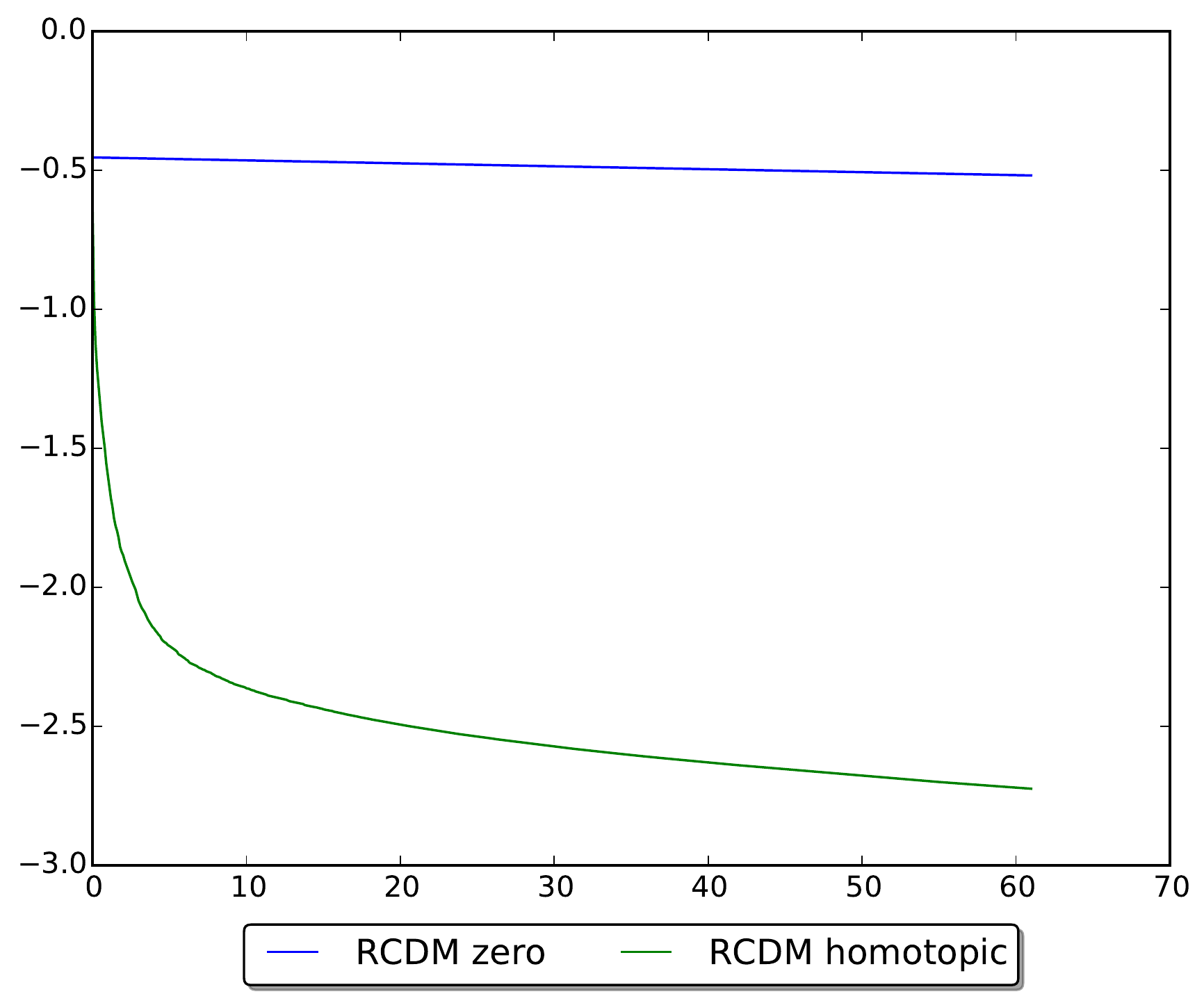}&
  \includegraphics[width=0.9\linewidth]{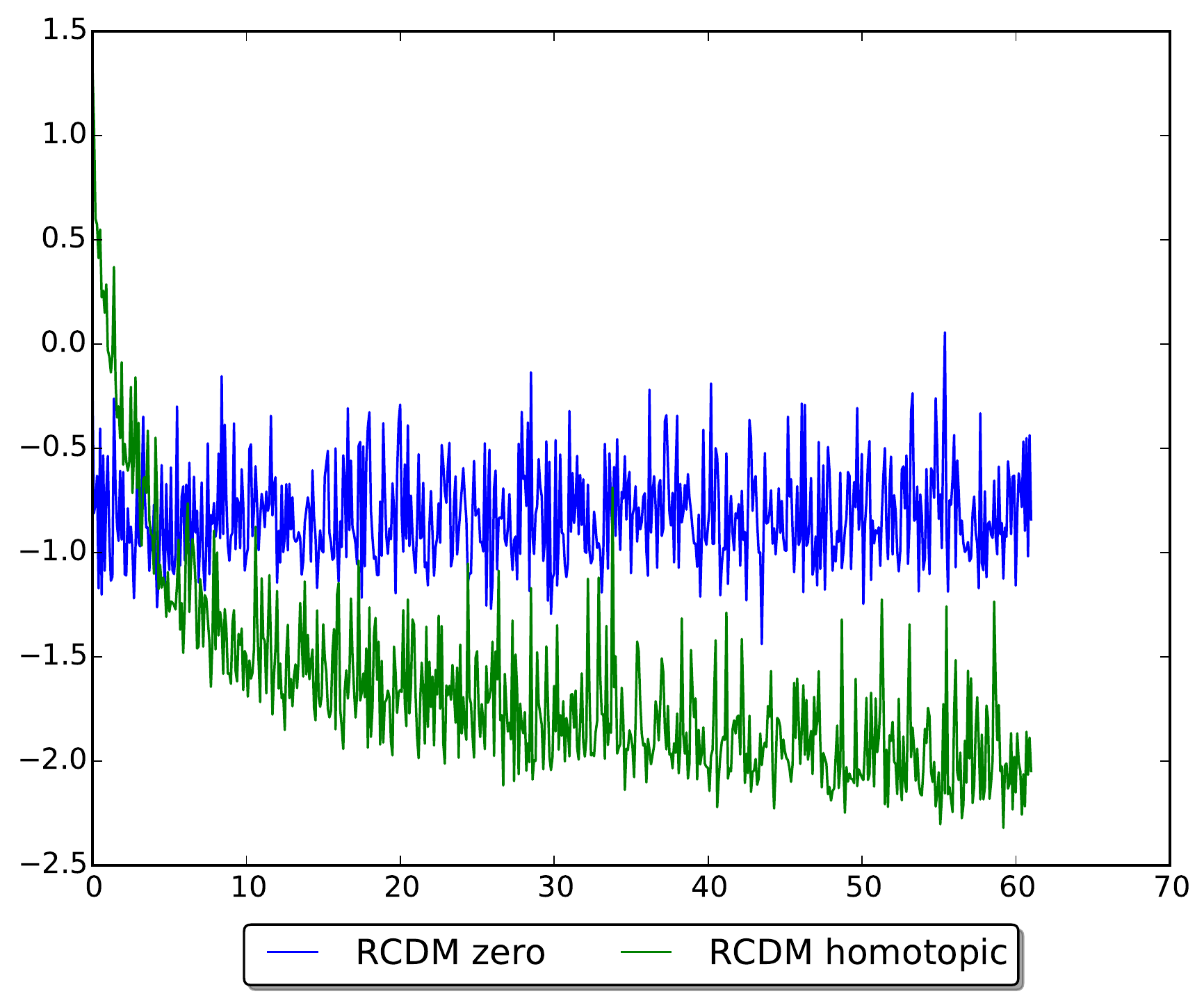} &
  \includegraphics[width=0.9\linewidth]{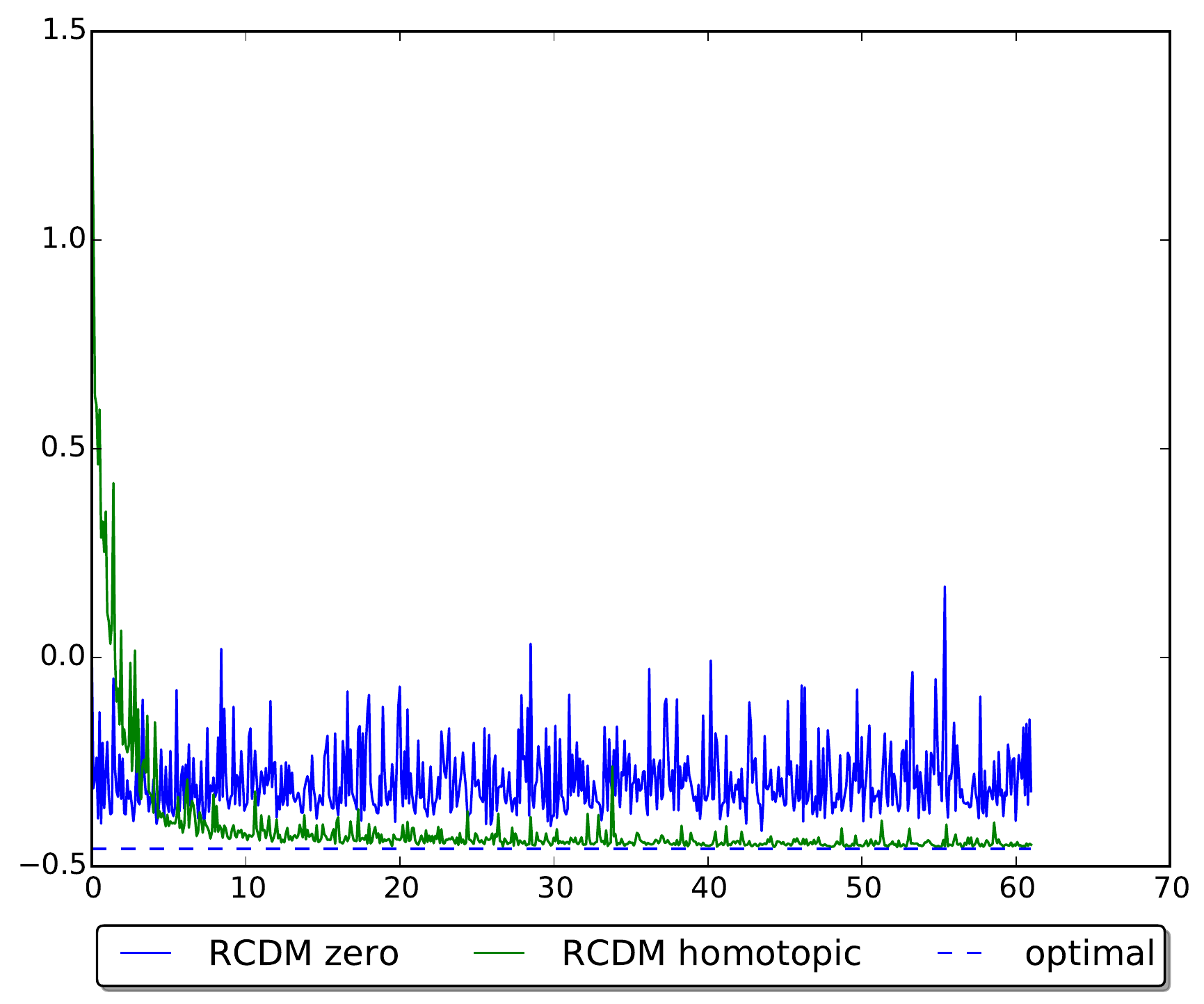}
  \\
  covtype &  \includegraphics[width=0.9\linewidth]{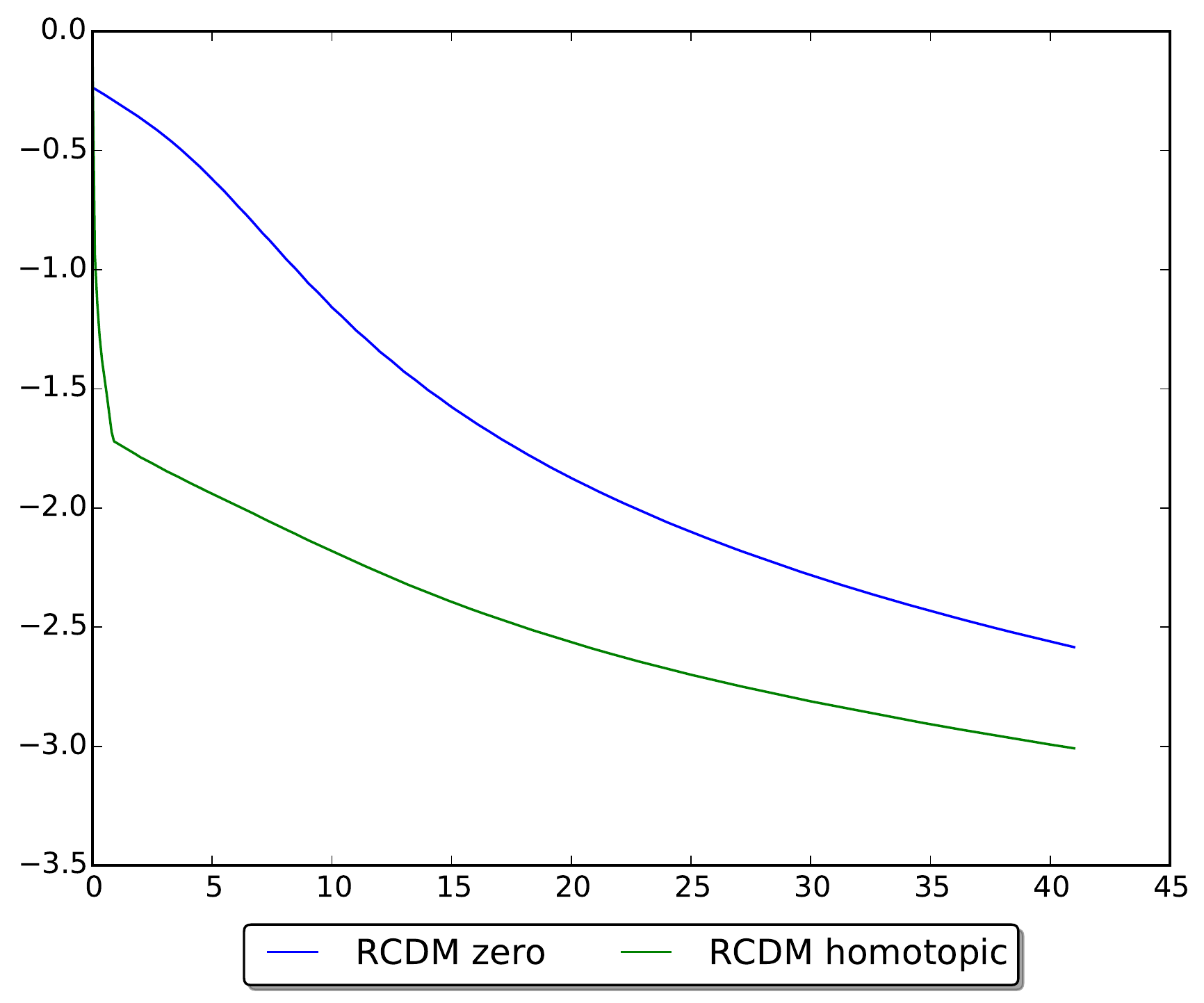}&
  \includegraphics[width=0.9\linewidth]{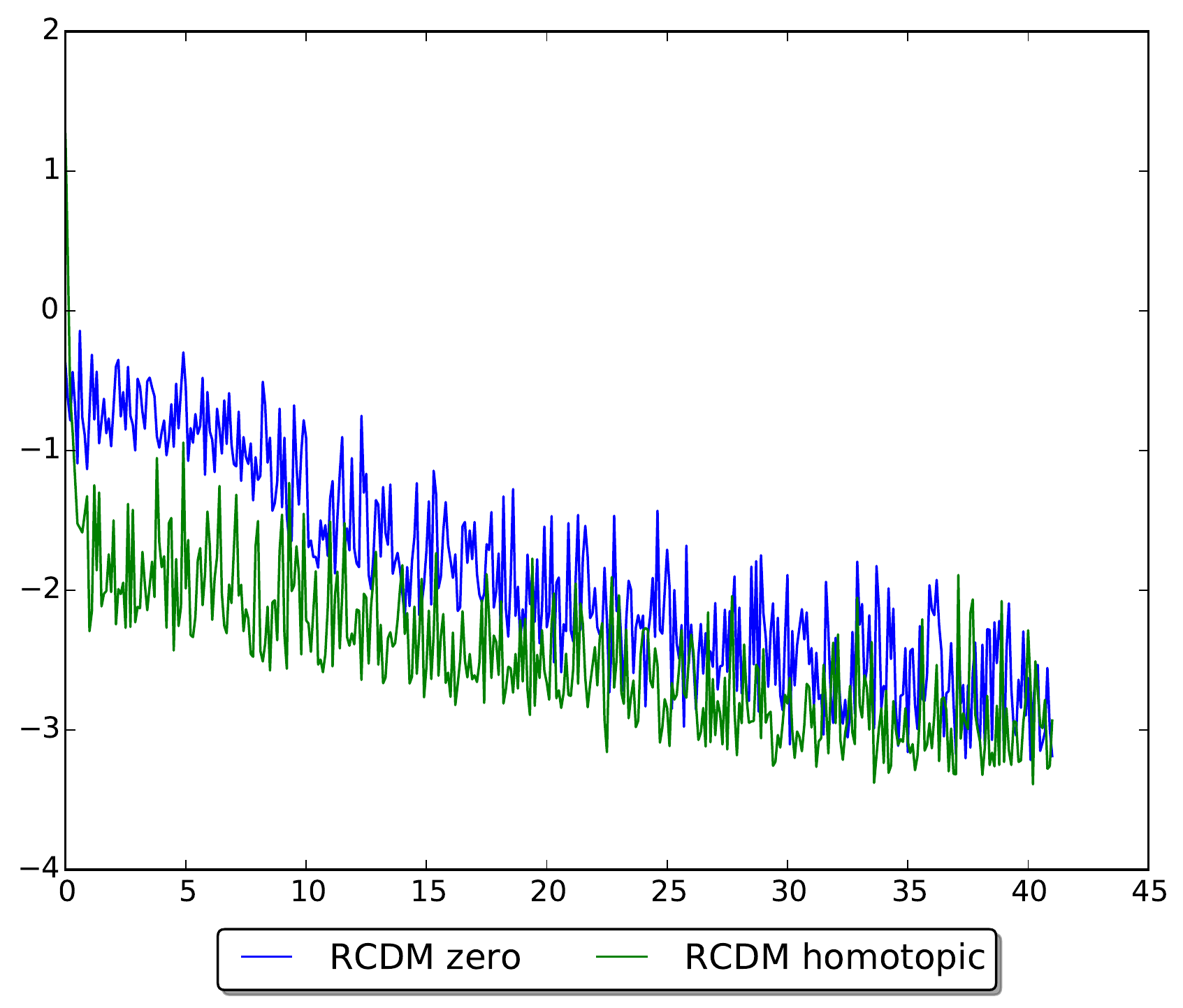} &
  \includegraphics[width=0.9\linewidth]{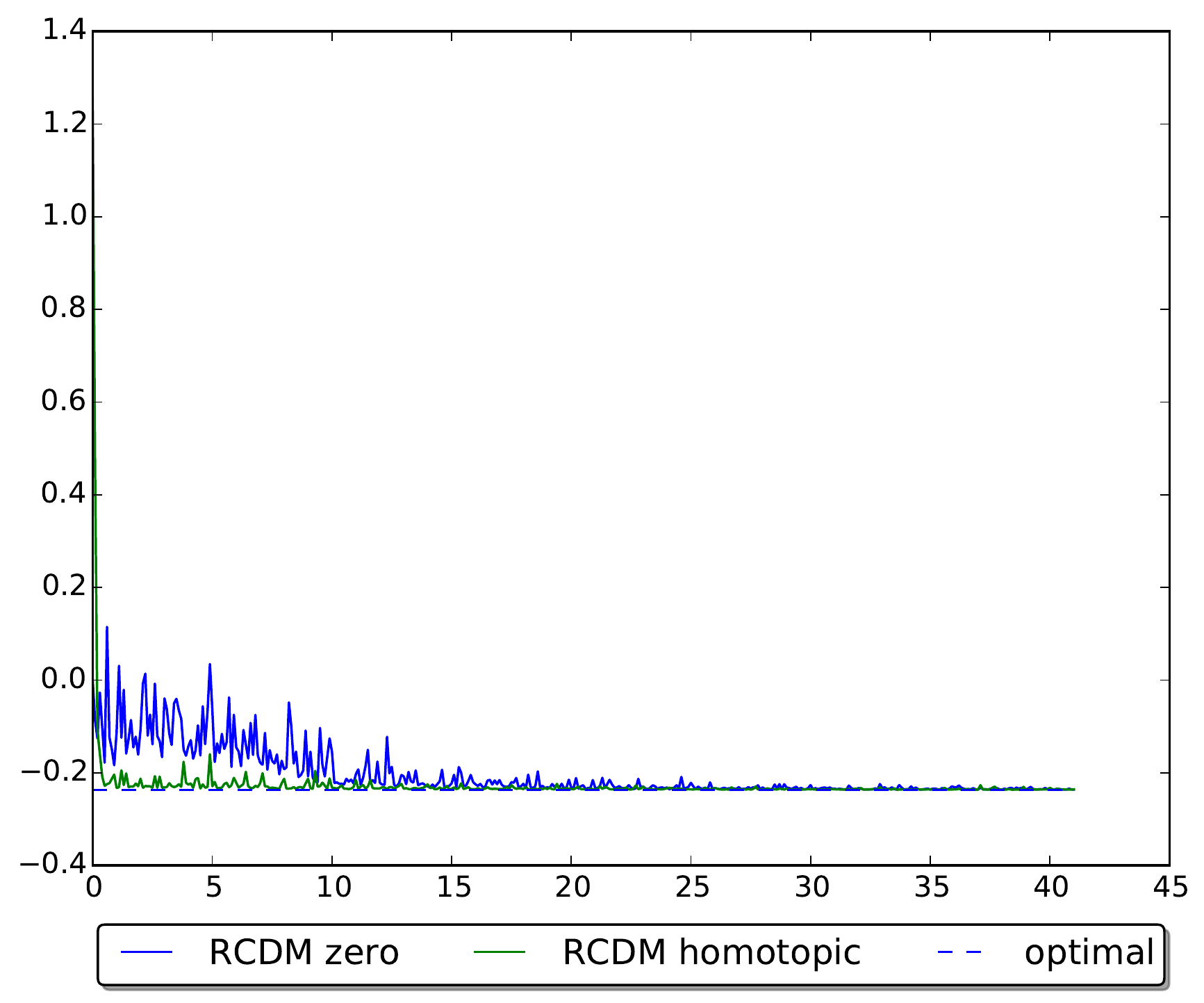} 
   \\ 
   gisette &  \includegraphics[width=0.9\linewidth]{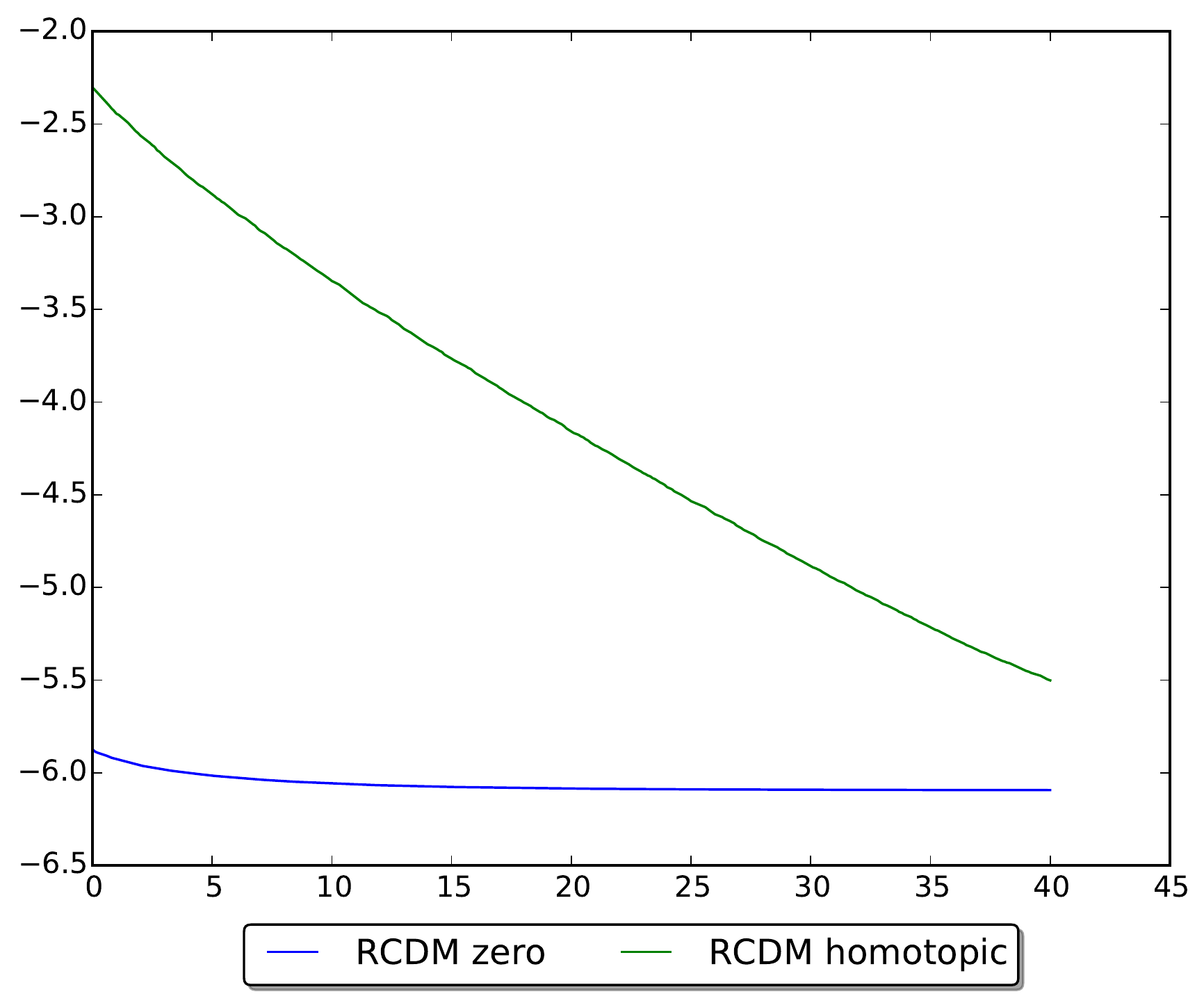}&
  \includegraphics[width=0.9\linewidth]{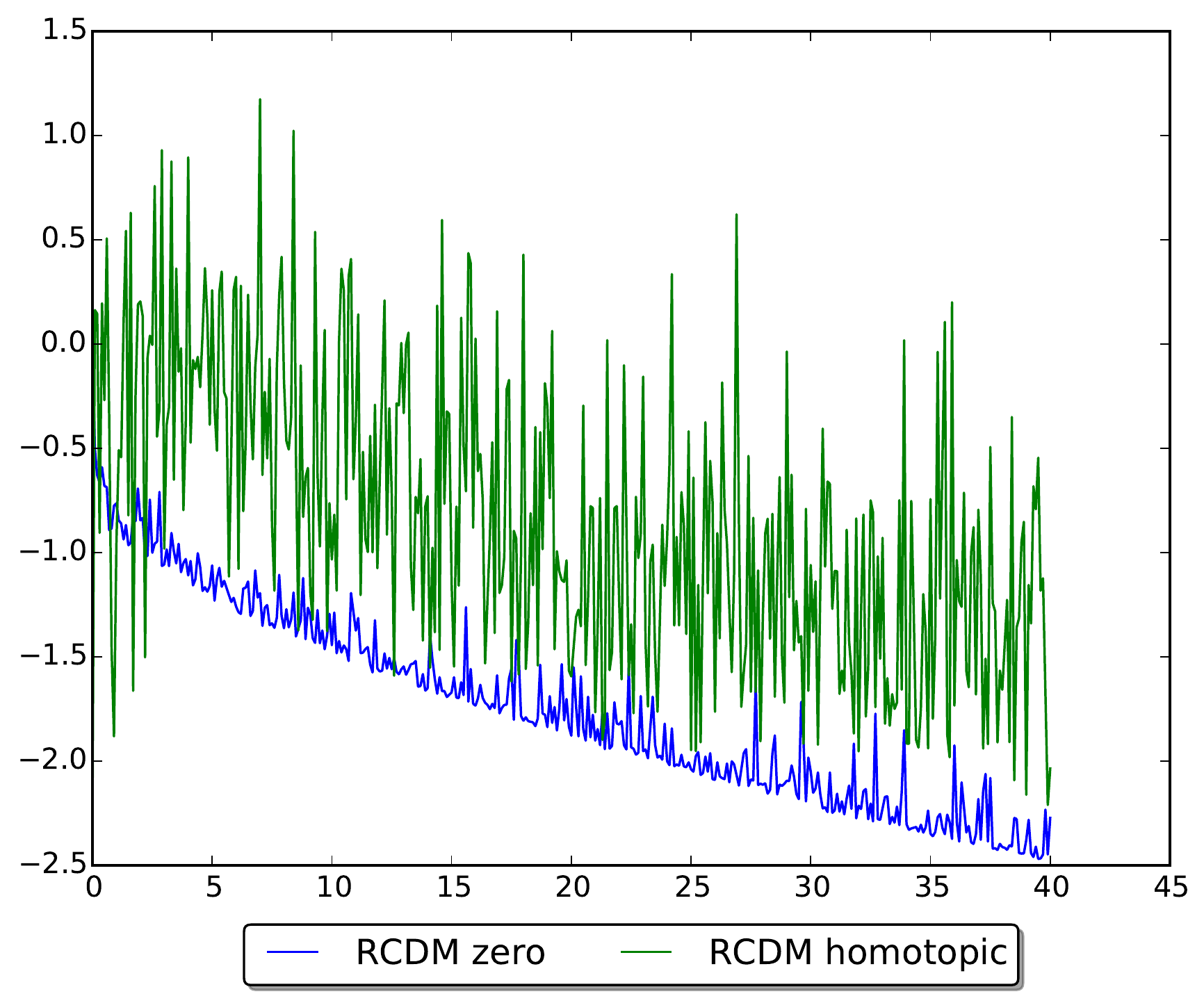} &
  \includegraphics[width=0.9\linewidth]{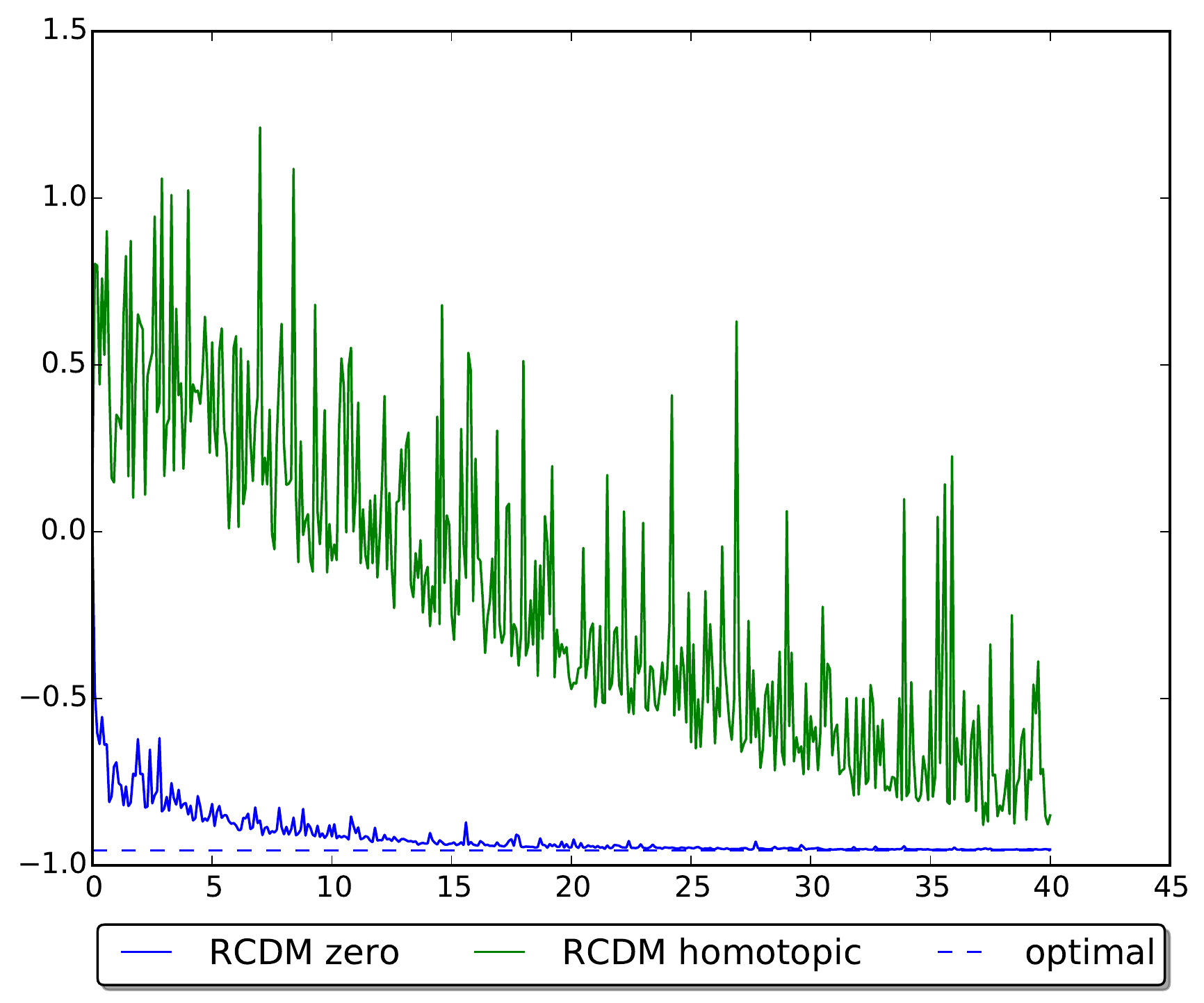} \\ 
  ijcnn1 &  \includegraphics[width=0.9\linewidth]{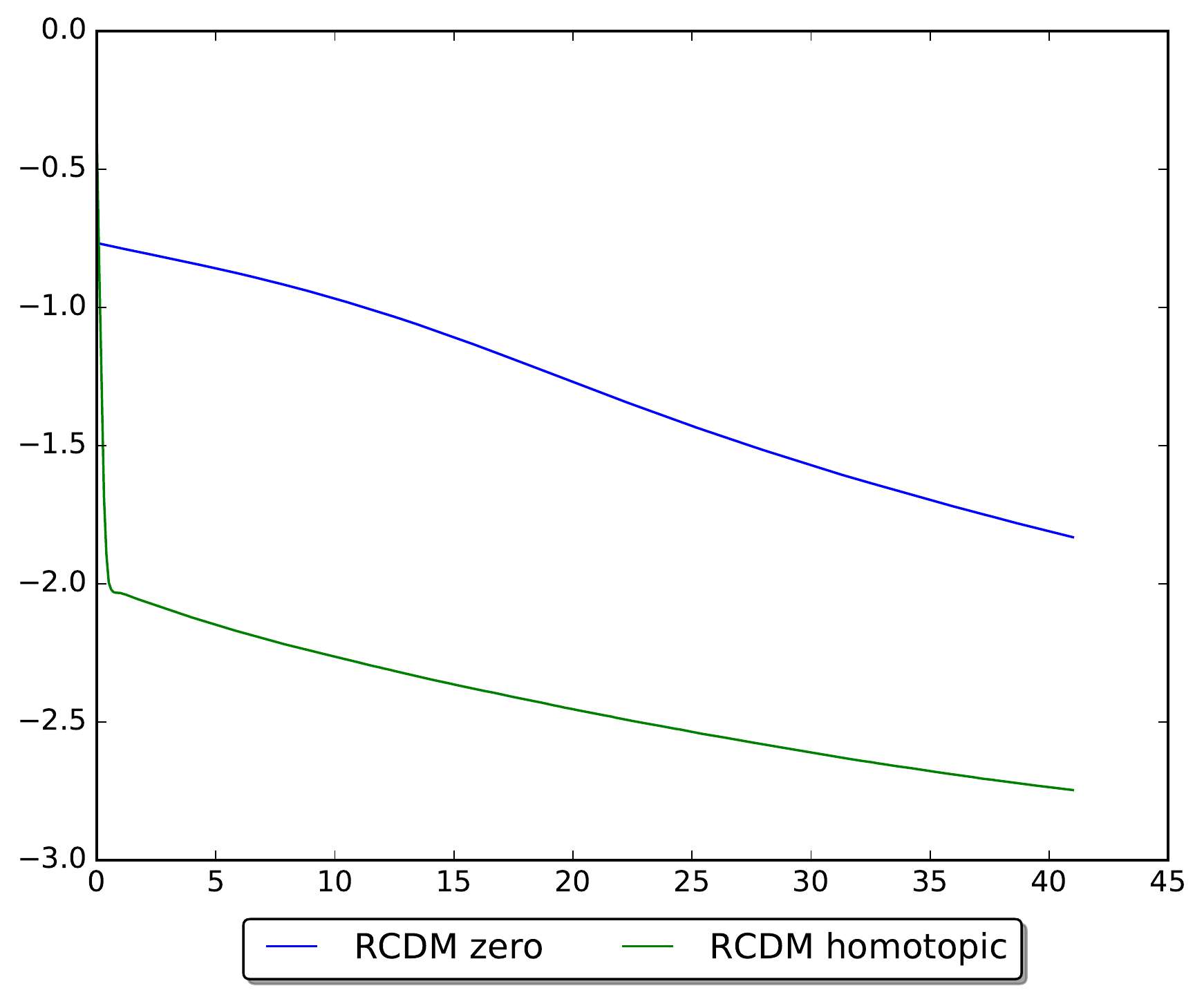}&
  \includegraphics[width=0.9\linewidth]{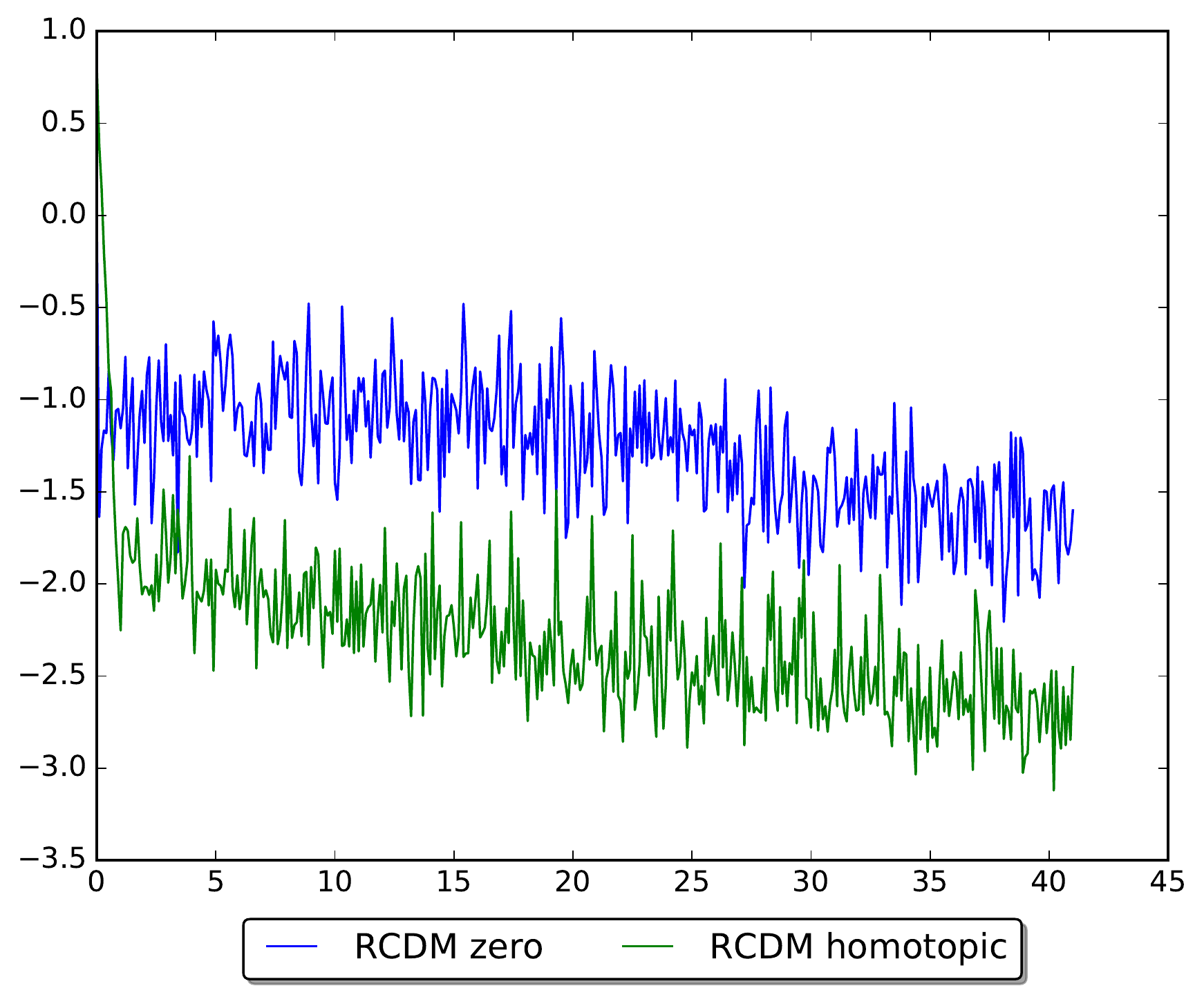} &
  \includegraphics[width=0.9\linewidth]{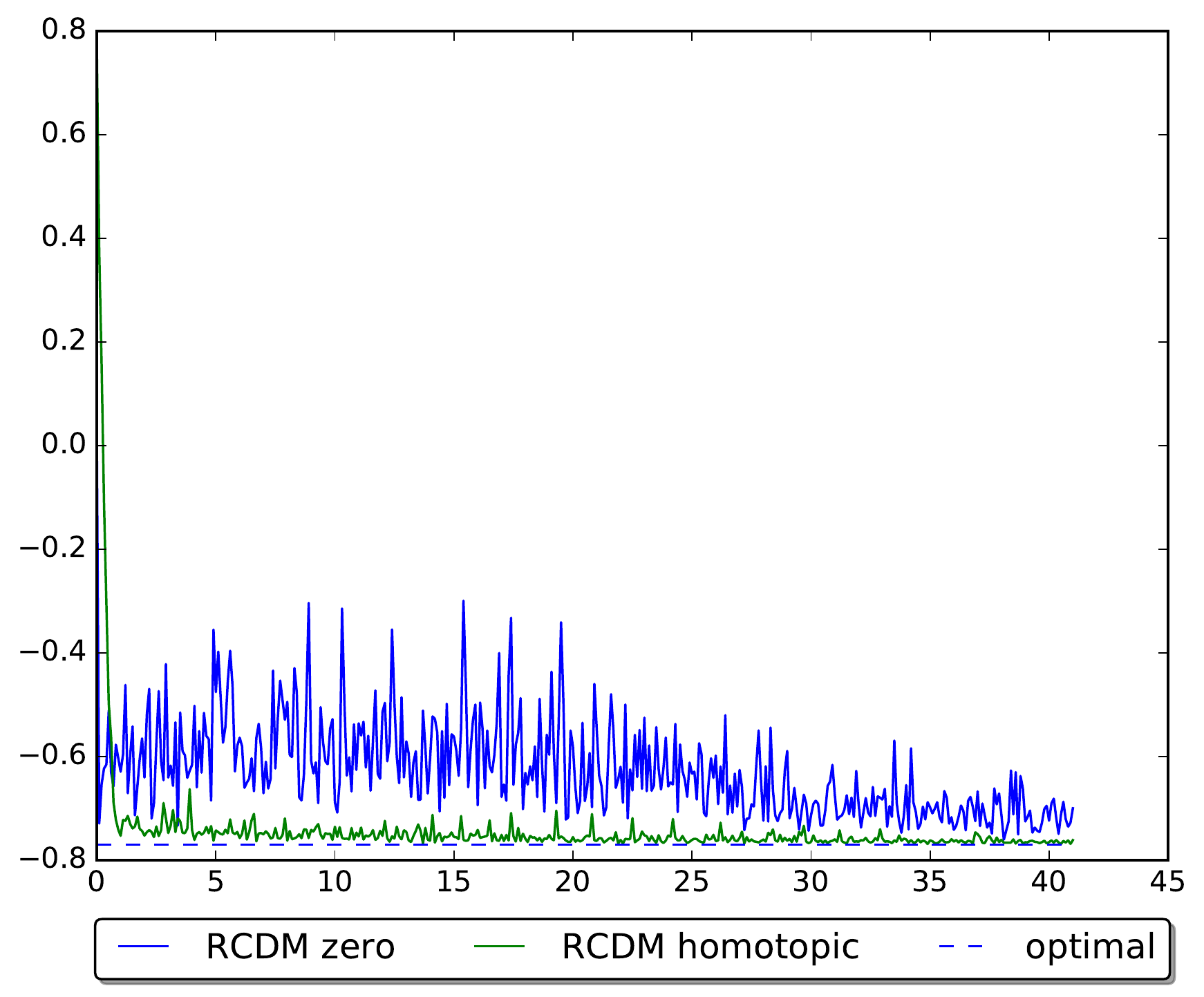} \\ 
  \\ 
  w8a &  \includegraphics[width=0.9\linewidth]{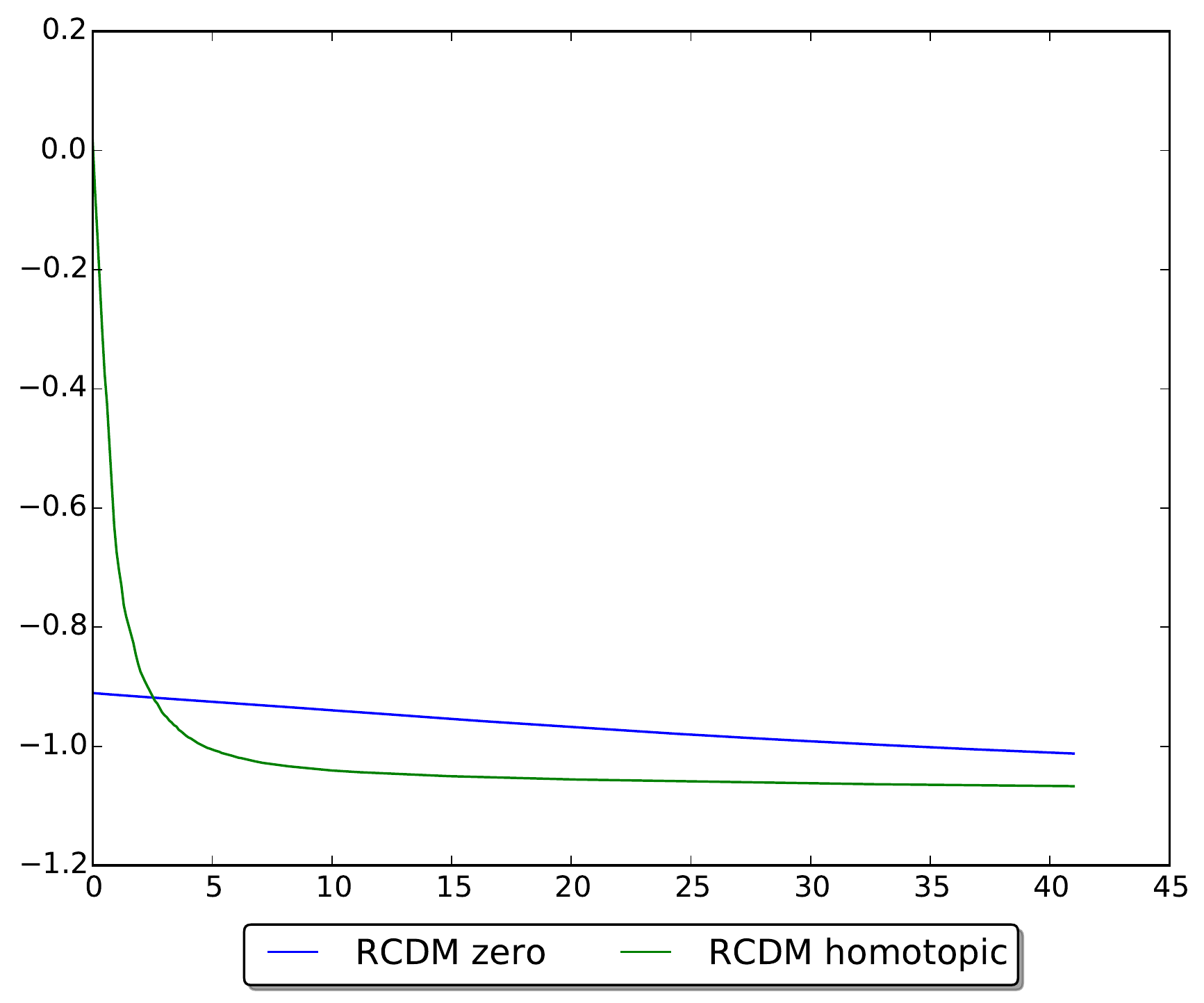}&
  \includegraphics[width=0.9\linewidth]{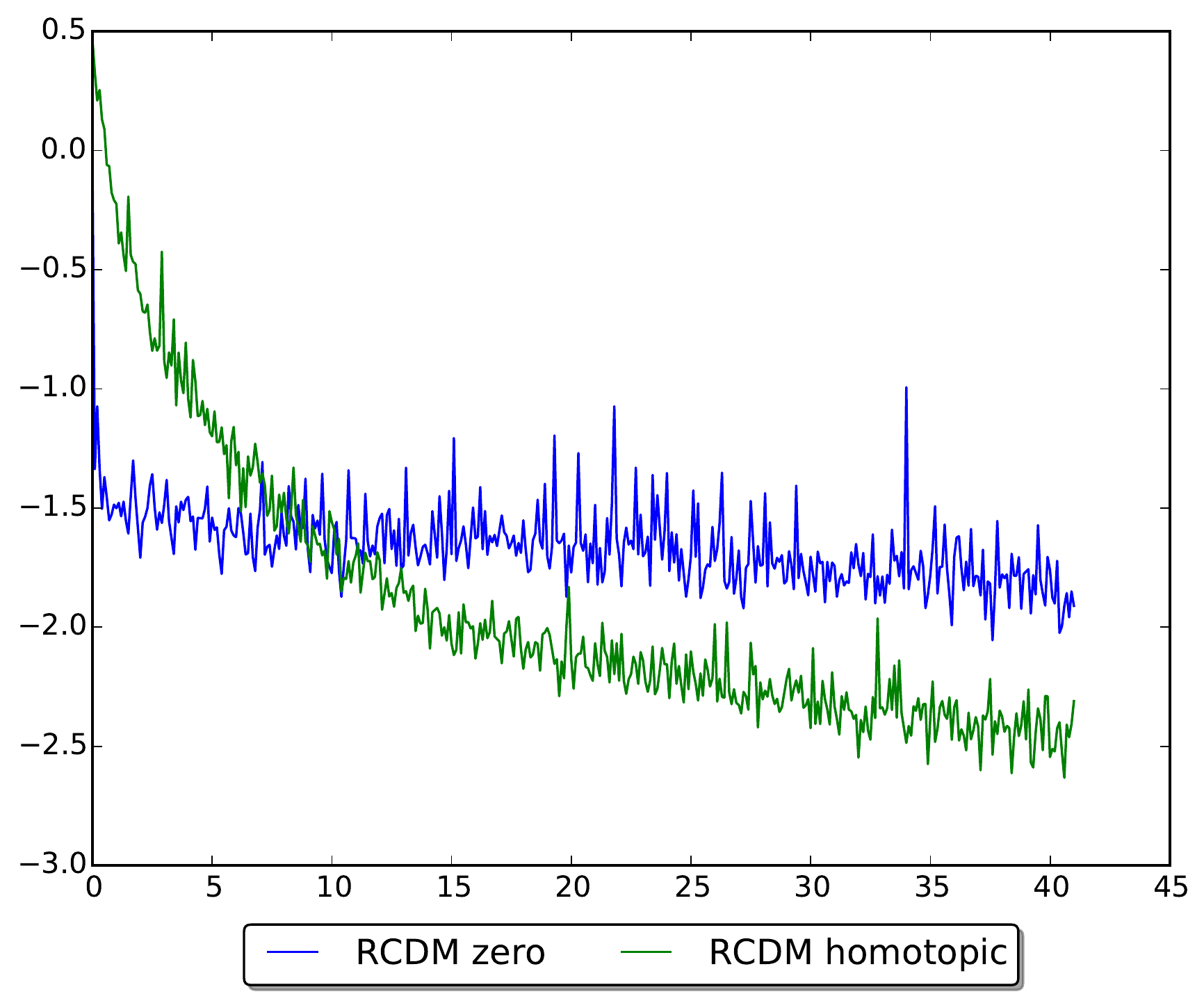} &
  \includegraphics[width=0.9\linewidth]{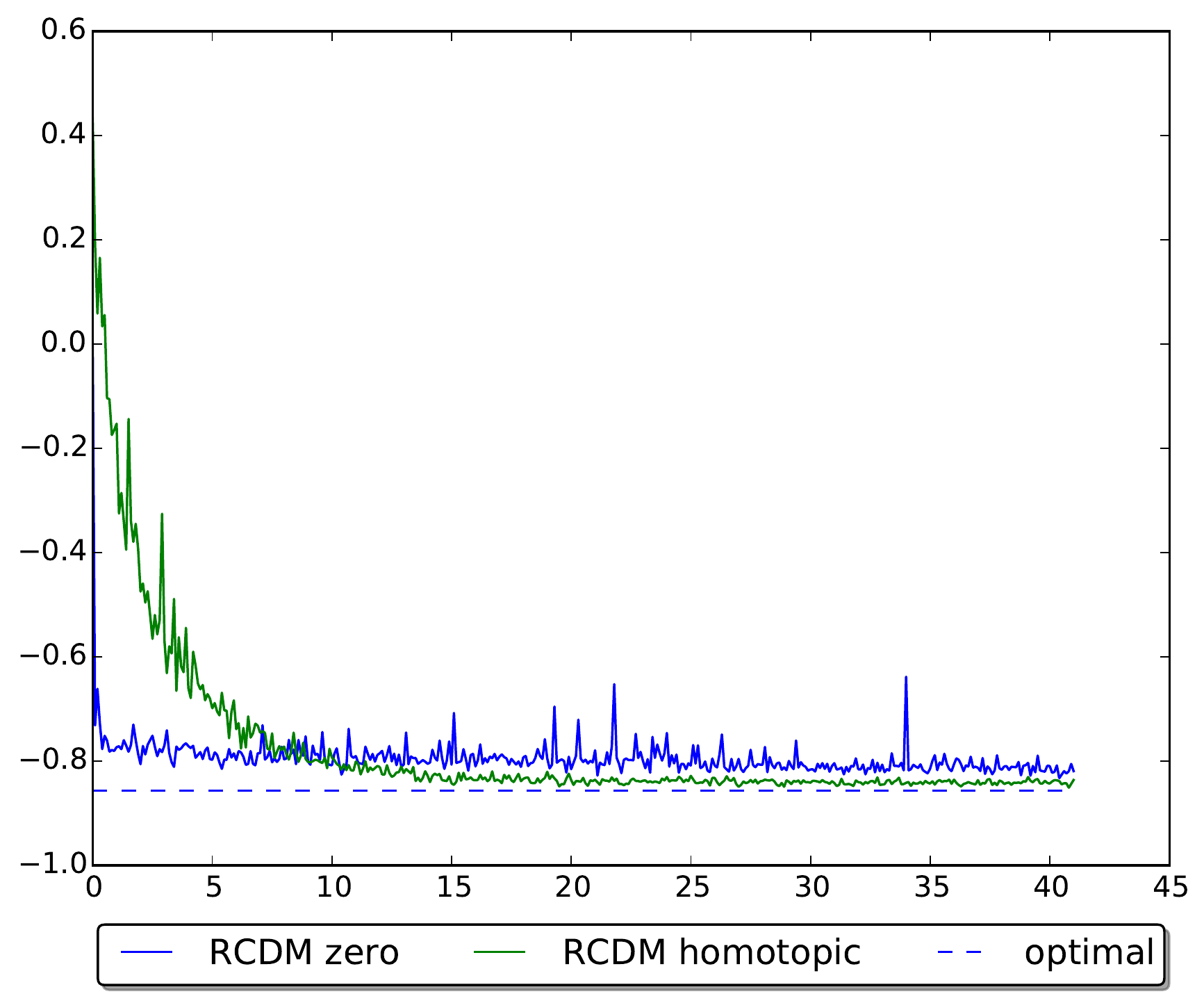}
  \\ 
  SUSY & \includegraphics[width=0.9\linewidth]{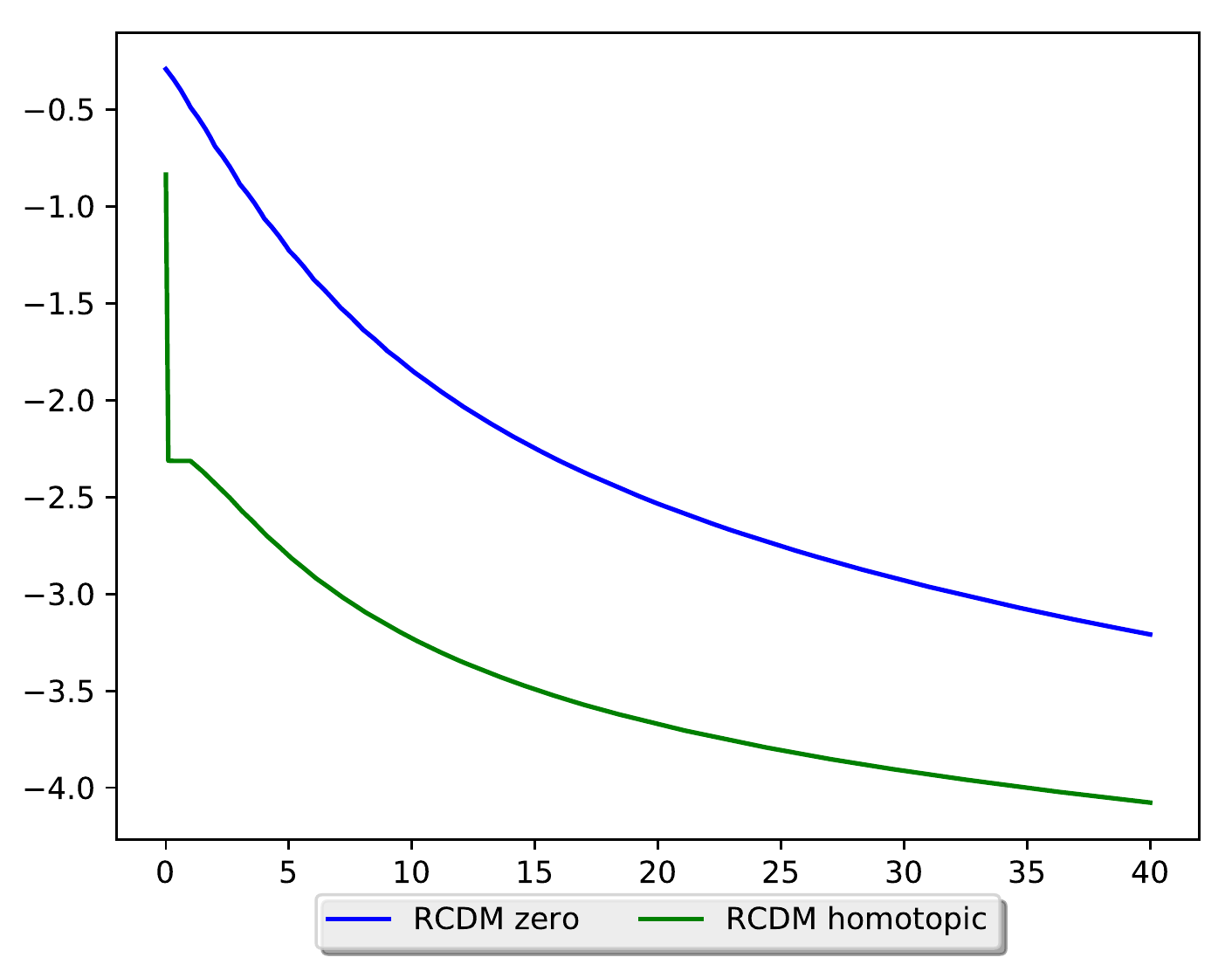}&
  \includegraphics[width=0.9\linewidth]{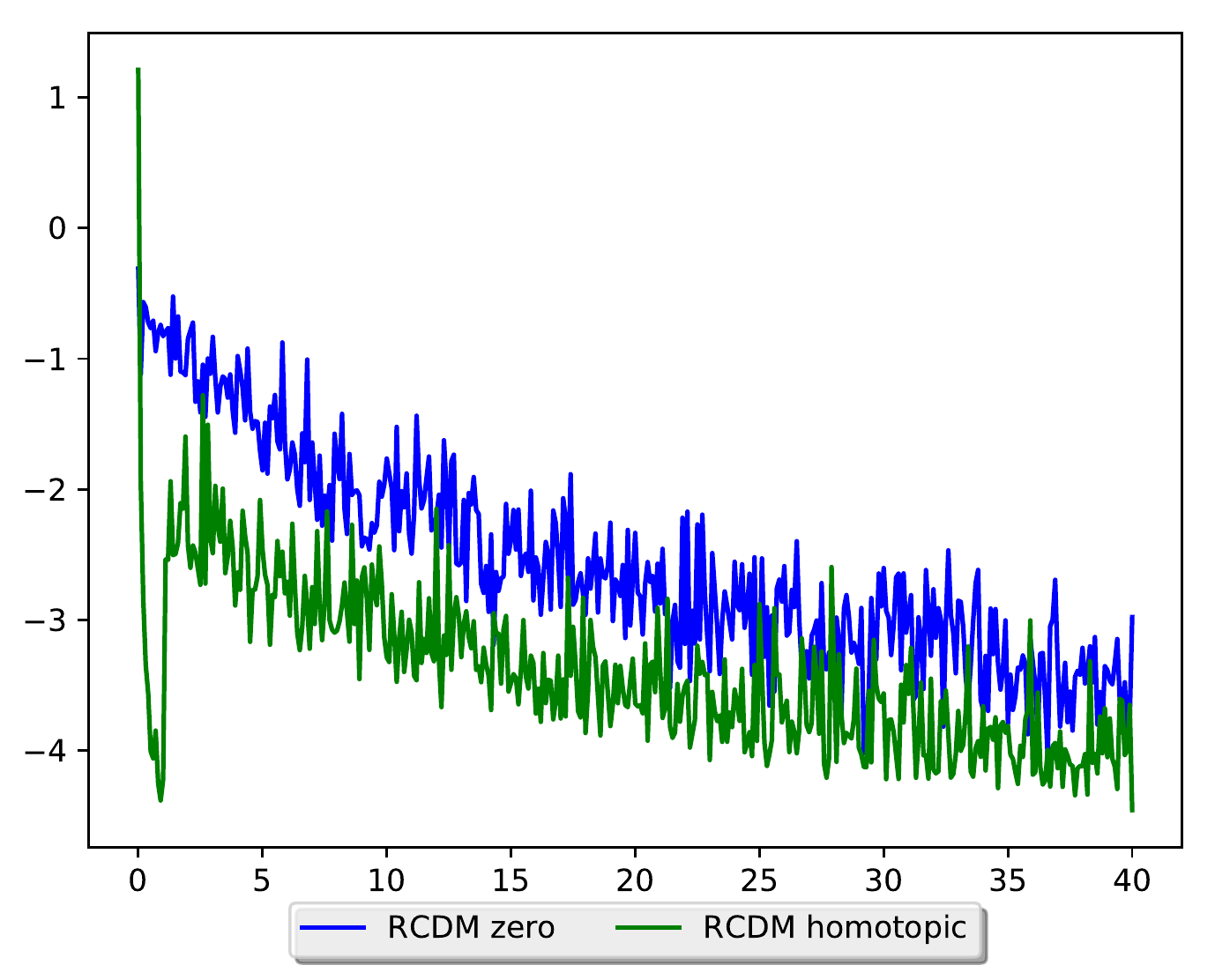} &
  \includegraphics[width=0.9\linewidth]{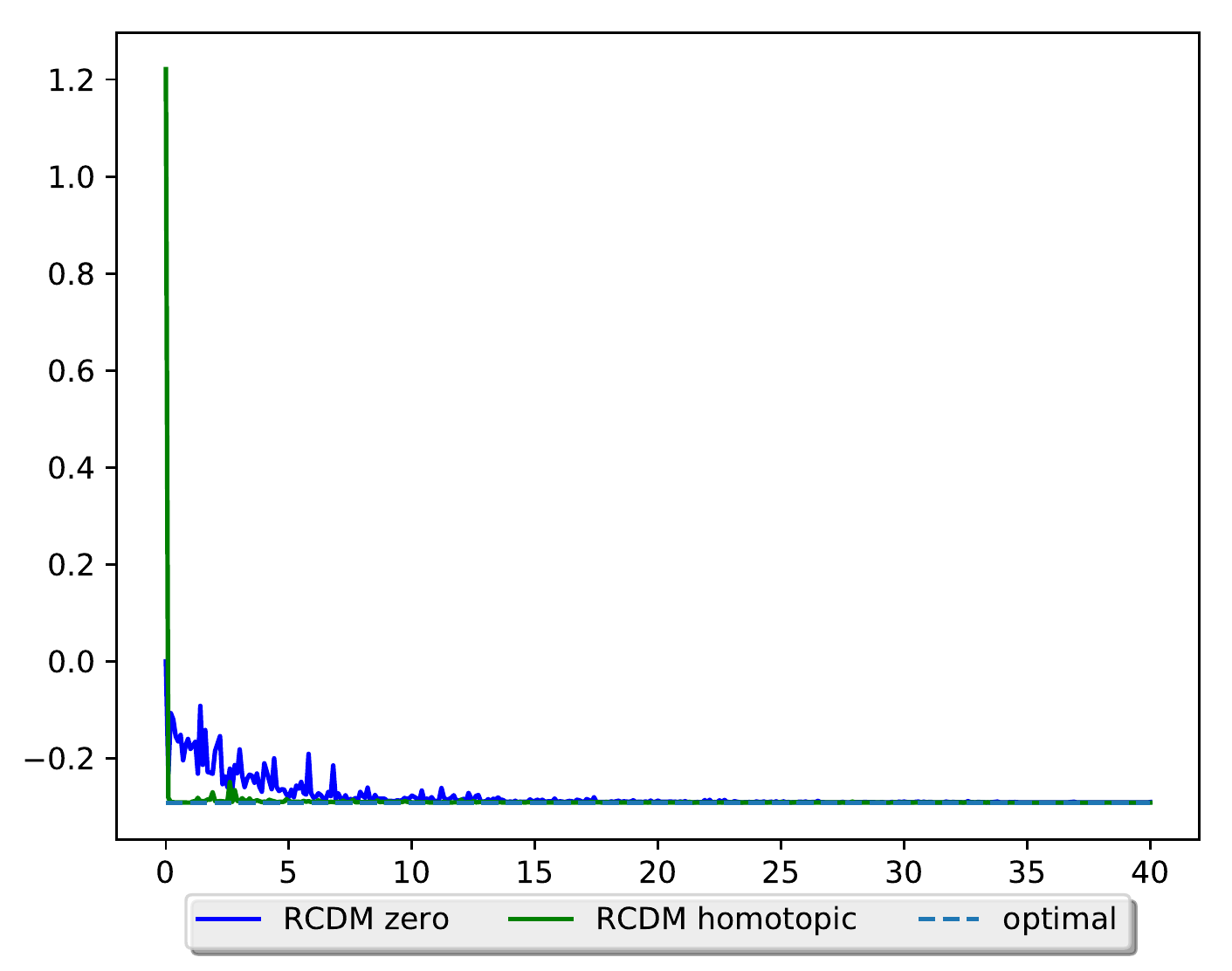}
\end{tabular}
\end{sc}
\end{footnotesize}
\caption{Homotopic initialization for dual SVM.
The vertical axis in plots (from left to right) shows logarithm of dual
suboptimality $\log_{10}( \quadratic(\dparam^{(t)})- \quadratic(\dparam^{*}))$,
primal suboptimality $\log_{10}(\quadratic(\param^{(t)})-
\risk(\param^*))$, and logarithm of the average of hinge loss on the
test set. The horizontal axis shows the number of epochs. The horizontal dashed lined in the test error plot shows test
error of the minimizer of the empirical risk. The
training set includes 80\% of the data.}
\label{fig:homotopic_rcdm_svm}
\end{figure}
\end{document}